\documentclass{article}
\usepackage[T1]{fontenc}
\usepackage[accepted]{icml2026}
\usepackage{times}
\usepackage{acronym}

\usepackage{algorithm}
\usepackage{algorithmic}
\usepackage{booktabs} 
\usepackage{CJK}
\usepackage{color}
\usepackage{float}
\usepackage{graphicx}
\usepackage{multirow}
\usepackage{svg}
\usepackage{subcaption}
\usepackage{tabularx}
\usepackage{tcolorbox}
\usepackage{natbib}
\usepackage{mdframed}
\usepackage{lipsum}
\usepackage{hyperref}
\usepackage{url}
\usepackage{multicol}
\usepackage{amsmath,amssymb}
\usepackage{amsthm}

\newtheorem{definition}{Definition}
\newtheorem{lemma}{Lemma}

\newtheorem{proposition}{Proposition}
\newtheorem{remark}{Remark}
\newtheorem{theorem}{Theorem}
\newtheorem{condition}{Condition}

\acrodef{llm}[LLM]{Large Language Model}

\let\oldfrac\frac% Store \frac
\renewcommand{\frac}[2]{%
	\mathchoice
	{\oldfrac{#1}{#2}}% display style
	{#1/#2}% text style
	{#1/#2}% script style
	{#1/#2}% script-script style
}

\icmltitlerunning{Feature Resemblance: Towards a Theoretical Understanding of Analogical Reasoning in Transformers}

\begin{document}
\twocolumn[
\icmltitle{Feature Resemblance: Towards a Theoretical Understanding of Analogical Reasoning in Transformers}
\icmlsetsymbol{equal}{*}

\begin{icmlauthorlist}
	\icmlauthor{Ruichen Xu}{ie}
	\icmlauthor{Wenjing Yan}{ie}
	\icmlauthor{Ying-Jun Angela Zhang}{ie}
\end{icmlauthorlist}

\icmlaffiliation{ie}{Department of Information Engineering, The Chinese University of Hong Kong, Hong Kong}
\icmlcorrespondingauthor{Ruichen Xu}{rcxu642@gmail.com}
\icmlcorrespondingauthor{Wenjing Yan}{wjyan@ie.cuhk.edu.hk}
\icmlcorrespondingauthor{Ying-Jun Angela Zhang}{yjzhang@ie.cuhk.edu.hk}
\icmlkeywords{Machine Learning, ICML}

\vskip 0.3in
]

\printAffiliationsAndNotice{} 

\begin{abstract}
    Understanding reasoning in large language models is complicated by evaluations that conflate multiple reasoning types. 
    We isolate analogical reasoning, where a model transfers an attribute between entities that share known properties, and study when such transfer can emerge from training.
    To make the problem analytically tractable, we study a minimal transformer-style abstraction that isolates how learned representations support analogical reasoning.
    Within this setting, we prove three key results. 
    First, joint training on similarity and attribution premises enables analogical reasoning through aligned representations. 
    Second, sequential training succeeds only when similarity structure is learned before specific attributes, revealing a curriculum asymmetry.
    Third, in our stylized setting, two-hop reasoning $(a \to b, b \to c \Rightarrow a \to c)$ can be viewed as analogical reasoning with identity bridges $(b=b)$, which appear explicitly in training data.
    Together, these results reveal a unified mechanism: entities with shared properties become aligned in representation space, enabling property transfer through feature resemblance. 
    Experiments with architectures up to 8B parameters show qualitative agreement with the theory and suggest that representational geometry plays an important role in analogical reasoning beyond the stylized model.
\end{abstract}

\section{Introduction}
\acp{llm} have demonstrated remarkable reasoning capabilities across diverse cognitive tasks \citep{wei2022chain,hubert2025olympiad,bubeck2023sparks}, yet the mechanisms underlying these abilities remain poorly understood.
A key obstacle to this understanding is that modern evaluation benchmarks typically require models to engage in multiple forms of reasoning simultaneously \cite{hendrycks2020measuring,srivastava2023beyond,jimenez2023swe}, making it difficult to isolate and analyze individual reasoning processes.
To illustrate, consider the following example from CommonsenseQA~\cite{talmor2018commonsenseqa}:
\emph{Question: What does a parent tell their child to do after they've played with a lot of toys? Answer: clean room.}
Answering this question involves both inductive reasoning (generalizing the pattern ``playing with toys $\to$
mess $\to$ need to clean'') and abductive reasoning (inferring the most plausible parental response). This entanglement is not unique to commonsense tasks; it pervades most existing benchmarks,  fundamentally limiting our ability to characterize the mechanisms underlying \acp{llm} reasoning.

To address this challenge, we focus on isolating a single, canonical form of inductive reasoning: \textbf{analogical reasoning}.
Analogical reasoning involves inferring that entities sharing certain properties are likely to share additional properties—a fundamental cognitive mechanism central to learning and generalization.
Formally, we adopt the following characterization~\cite{bartha2013analogy}\footnote{This formulation also covers proportional analogies of the form ``$A$ is to $B$ as $C$ is to $D$'' as a special case, by treating the compared objects as ordered pairs.
It is also consistent with Gentner's structure-mapping theory: the similarity premise captures a shared relational skeleton, while the conclusion projects a candidate inference from source to target~\citep{gentner1983structure}.
}:

\begin{definition}[Analogical argument]\label{def: aa}
An \textbf{analogical argument} is a form of inductive reasoning that concludes: if two entities $A_1$ and $A_2$ share certain properties, they are also likely to share a further property, $C$.
	\begin{itemize}
		\item (Similarity Premise) Entities $A_1$ and $A_2$ share some known properties.
		\item (Attribution Premise) $A_2$ has some further property $C$.
		\item (Conclusion) $A_1$ also has the property $C$.
	\end{itemize}
\end{definition}
By isolating this specific reasoning mode, we can ask a more focused version of a broader question:
\begin{center}
\textit{What mechanism enables analogical reasoning in transformer-style models?}
\end{center}

For modern transformers, a complete answer remains a broad open problem.
To make progress, we take a first step by studying a minimal transformer-style abstraction in which such a reasoning mechanism can be analyzed explicitly.
In the stylized setting we analyze, we identify the following mechanism:
\begin{tcolorbox}[title={Feature Resemblance}]
	Training can map entities with shared properties to similar internal features.
	Once such features are aligned, an attribute learned for one entity can transfer to another through the shared representation.
\end{tcolorbox}

To formalize this intuition, our main analysis studies the training dynamics of a minimal transformer-style token-mixing block across three \emph{out-of-context reasoning} scenarios:
\begin{itemize}
	\item \textbf{Joint training:} Simultaneous training on similarity and attribution premises produces aligned features that support analogical reasoning.
	
	\item \textbf{Sequential training:}   In this setting, analogical reasoning emerges only under a specific curriculum: the model must learn similarity structure, i.e., relations between $A_1$ and $A_2$, before learning specific attributes $C$. The reverse order can fit the training premises but fails to generalize.
    
	\item \textbf{Two-hop reasoning:} In our formulation, two-hop reasoning ($a \to b, b \to c \implies a \to c$) can be viewed as analogical reasoning where the similarity premise includes an identity relation ($b = b$). This equivalence reveals the necessity of explicit identity examples in training: without them, the model does not learn to bridge intermediate concepts and therefore fails to perform two-hop reasoning.
\end{itemize}

We further use deep linear networks to study how feature resemblance evolves across depth.
This multi-layer analysis shows that training can progressively align representations of related inputs, supporting the feature resemblance mechanism beyond the one-block setting.

Experiments on synthetic and natural-language datasets, using both controlled transformer models and pretrained language models up to 8B parameters, show qualitative agreement with our theoretical predictions.

\subsection{Related work}
\paragraph{Mechanistic interpretability for reasoning in LLMs.} 
Mechanistic interpretability seeks to understand how specific components within \acp{llm} implement particular functions through empirical analysis~\cite{cunningham2023sparse,bricken2023monosemanticity,park2023linear}. 
Recent work has identified specialized structures responsible for distinct capabilities: induction heads for in-context learning~\cite{olsson2022context}, knowledge circuits for factual storage and editing~\cite{yao2024knowledge,meng2022locating}, and extractive structures for compositional reasoning~\cite{feng2024extractive}. 
However, these interpretability studies typically analyze trained models \emph{post hoc}, leaving open the question of how such functional components emerge during training through gradient-based optimization.

\paragraph{Theoretical analyses of LLM reasoning.}
Theoretical studies of \acp{llm} reasoning capabilities can be broadly categorized into in-context and out-of-context reasoning. 
In-context reasoning involves deriving answers from information provided in the prompt. 
For example, several works prove that in-context learning in linear transformers is equivalent to implicit gradient descent~\cite{ahn2023transformers,mahankali2023one,zhang2024trained}, with extensions showing that chain-of-thought prompting enables multi-step optimization~\cite{huang2025transformers}.
Nichani et al.~\cite{nichani2024transformers} demonstrate that two-layer transformers can learn causal structures in random sequences.
In contrast, theoretical analyses of out-of-context reasoning—where models must generalize beyond prompt information—remain limited.
Existing work primarily focuses on knowledge mechanisms such as memorization capacity~\cite{nichani2024understanding} and next-token prediction~\cite{tian2023scan,zhu2024towards}.
Our work differs by investigating out-of-context \emph{reasoning} capabilities, specifically analogical inference, rather than knowledge retrieval or memorization.

We also use deep linear networks as a stylized multi-layer model in
Section~7. We contextualize this deep-linear analysis within the
broader literature on deep linear networks in Appendix~\ref{app:deep_rw}.

Our work is closely related to \citet{huang2025generalization}, who study out-of-context generalization in simplified linear and bilinear models under joint training. 
We differ in three aspects. 
First, we analyze a minimal transformer-style block with trainable attention and value projection, where attention controls the mixing of entity and relation tokens in a length-two prompt. 
Second, we also study sequential training and show a curriculum asymmetry: similarity-then-attribution enables analogical reasoning, whereas the reverse order can fit the training data but fail to generalize. 
Third, we give finite-time polynomial-iteration guarantees and connect two-hop reasoning to analogical reasoning through identity bridges. 
Thus, our results characterize how feature resemblance forms during training, beyond its asymptotic behavior under joint training.

We also acknowledge concurrent research \citet{lin2025identity} that empirically employs the identity bridge to improve performance in out-of-context two-hop reasoning. Distinctly, our research provides a theoretical grounding by linking two-hop reasoning to analogical reasoning and elucidating the effectiveness of the identity bridge from a theoretical perspective of training dynamics.

\paragraph{Notation.}
Scalars, vectors, and matrices are denoted by lowercase ($x$), bold lowercase ($\mathbf{x}$), and bold uppercase ($\mathbf{X}$), respectively. We define $[n] := \{1, \dots, n\}$. We use standard asymptotic notations $\mathcal{O}, \Omega, \Theta$, where tilde variants (e.g., $\tilde{\mathcal{O}}$) suppress logarithmic factors.

\section{Analogical Reasoning: Data Structure and Evaluation}\label{sec: data}

Following \cite{ghosal2024understanding}, we represent each knowledge triple as $(a, r, b)$, where $a$ is an entity, $r$ is a relation type, and $b$ is an attribute. 
This structure naturally encodes the premises and conclusion of analogical arguments (Definition~\ref{def: aa}):

\begin{itemize}
    \item \textbf{Similarity Premise:} Two entities share a common attribute, represented by triples $(a_1, r_1, b_1)$ and $(a_2, r_1, b_1)$.
    \item \textbf{Attribution Premise:} The source entity possesses an additional property, represented by $(a_2, r_2, b_2)$.
    \item \textbf{Conclusion:} The target entity also possesses this property, represented by $(a_1, r_2, b_2)$.
\end{itemize}

\noindent\textbf{Example.} Consider the following knowledge triples:

\textit{Premises (training):}\\
1. (pika, has-feature, feather) \quad 
2. (finch, has-feature, feather) \quad 
3. (pika, is-a, bird)

\textit{Conclusion (test):} (finch, is-a, \underline{bird})

The model must infer that because pika and finch share the feature ``feathers'' (similarity premise) and pika is a bird (attribution premise), finch is likely also a bird (conclusion).

\begin{remark}
   An analogical inference may be factually incorrect.
    Here we do not measure the factual accuracy but the success rate of analogy from the source entity to the target entity. 
\end{remark}

\paragraph{Training dataset (premises).} 
Let $\mathcal{I}(\cdot)$ denote an indexing function that maps an entity to its categorical label. Given $N$ distinct entity tuples, we define the training knowledge triples as follows:
\begin{itemize}
\item Similarity Premise ($\mathcal{T}_1, \mathcal{T}_2$): $\mathcal{T}_1 = \{(\mathbf{a}_{i},\mathbf{r}_{1}, \mathbf{b}_{i})\}_{i=1}^N$ and $\mathcal{T}_2 =  \{(\mathbf{a}'_{i},\mathbf{r}_{1},  \mathbf{b}_{i})\}_{i=1}^N$. 
These establish that source $\mathbf{a}'_i$ and target $\mathbf{a}_i$ share property $\mathbf{b}_i$ under relation $\mathbf{r}_1$.
\item Attribution Premise ($\mathcal{T}_3$): $\mathcal{T}_3=\{(\mathbf{a}'_{i},\mathbf{r}_{2}, \mathbf{c}_{i})\}_{i=1}^N$.
This specifies that an entity $\mathbf{a}'_i$ possesses property $\mathbf{c}_i$ under relation $\mathbf{r}_2$.
\end{itemize}
The corresponding training samples $\mathcal{S}_j$ for each knowledge set $\mathcal{T}_j$ are generated as:
\begin{align}\nonumber
    &\text{(The Similarity Premise 1)  }\mathcal{S}_1: \{([\mathbf{a}_i\;\;\mathbf{r}_1], \mathcal{I}(\mathbf{b}_i))\}_{i=1}^N,\\ \nonumber
    &\text{(The Similarity Premise 2)  }\mathcal{S}_2: \{([\mathbf{a}'_i\;\;\mathbf{r}_1], \mathcal{I}(\mathbf{b}_i))\}_{i=1}^N,\\\nonumber
    &\text{(The Attribution Premise)  } \mathcal{S}_3: \{([\mathbf{a}'_i\:\:\mathbf{r}_2], \mathcal{I}(\mathbf{c}_i))\}_{i=1}^N.
\end{align}

\paragraph{Test Dataset (Analogical Conclusion).}The test set $\mathcal{A}$ evaluates whether the model can project the property $\mathbf{c}_i$ from the source to the target. For each $i\in[N]$, the model is prompted to conclude:
$$\text{(The conclusions)  }\mathcal{A}: \{([\mathbf{a}_i\:\:\mathbf{r}_2], \mathcal{I}(\mathbf{c}_i))\}_{i=1}^N.$$

Our data construction is designed to isolate analogical reasoning. 
Each example is represented as a short entity--relation prompt with a single target attribute. 
This controlled setting allows us to study when training aligns the representations of analogous entities while abstracting away from discourse-level context, positional structure, and long-range dependencies.

Following \citet{tian2023scan}, we assume all token embeddings are drawn as a random
orthonormal system, i.e., $\left\|\mathbf{z}_{1}\right\|_2 = 1$ and $\langle \mathbf{z}_1,\mathbf{z}_2 \rangle = 0$ for all $\mathbf{z}_1,\mathbf{z}_2\in \{\mathbf{a}_{i},\mathbf{a}'_{i}\}_{i=1}^N\cup \{\mathbf{r}_1,\mathbf{r}_2\}$ and $\mathbf{z}_1\neq \mathbf{z}_2$.

\section{Setup}
In this section, we introduce the setup of the analyses.
\subsection{A minimal transformer-style block}
We analyze a minimal transformer-style block designed to isolate one mechanism: how trainable token mixing and feature learning can create transferable representations. 
The abstraction retains only the components needed for this analysis: an attention-based mixing operation, a value projection, and a linear feature layer.
\paragraph{Self-attention layer.}
Given input embeddings $\mathbf{X}\in\mathbb{R}^{d\times L}$ where $L$ is the sequence length, the self-attention mechanism computes attention scores as:
\begin{align}
	\boldsymbol{\alpha}(\mathbf{Z},\mathbf{X}) = \text{softmax}\left(\frac{\mathbf{X}^\top\mathbf{Z}\mathbf{X}[-1]}{\sqrt{d}}\right),
\end{align}
where $\mathbf{X}[-1]\in\mathbb{R}^d$ denotes the embedding of the last token.
Following prior theoretical analyses~\citep{tian2023scan}, we merge the query and key matrices into a single matrix $\mathbf{Z}=\mathbf{K}^\top\mathbf{Q}\in\mathbb{R}^{d\times d}$. This simplification retains the token-mixing role of attention in a tractable form.

The self-attention output is then:
\begin{align}
	\mathbf{o}_1(\mathbf{Z},\mathbf{V},\mathbf{X})= \mathbf{V}\mathbf{X}\boldsymbol{\alpha}(\mathbf{Z},\mathbf{X})\in\mathbb{R}^d,
\end{align}
where $\mathbf{V}\in\mathbb{R}^{d\times d}$ is the value matrix.

\paragraph{Linear feature layer.}
The feature layer takes the attention output $\mathbf{o}_1$ and produces the final representation $\mathbf{f}\in \mathbb{R}^d$ through a linear transformation. The $k$-th element of $\mathbf{f}$ is:
\begin{align}
	f_k(\mathbf{Z},\mathbf{V},\mathbf{W},\mathbf{X}) = \frac{\lambda}{m}\sum_{l=1}^m \langle\mathbf{w}_{k,l},\mathbf{o}_1(\mathbf{Z},\mathbf{V},\mathbf{X})\rangle,
\end{align}
where $\mathbf{w}_{k,l}\in\mathbb{R}^d$ are the feature-layer parameters and $m$ is the feature-layer width. 
This layer can be viewed as a linearized MLP block.
The scaling factor $\lambda$ emulates normalization operations in practical language models (e.g., RMSNorm, which normalizes the input's $\ell_2$ norm to $\sqrt{d}$).

\subsection{Loss functions}
\paragraph{Training loss.}
Given a training dataset $\mathcal{S} = \{(\mathbf{X}_i,y_i)\}_{i=1}^n$, we train the model using the cross-entropy loss:
\begin{align}
	\mathcal{L}_\mathcal{S}(\mathbf{Z},\mathbf{V},\mathbf{W}) = \frac{1}{n}\sum_{(\mathbf{X},y)\in\mathcal{S}} \mathcal{L}(\mathbf{Z}, \mathbf{V}, \mathbf{W},\mathbf{X},y),
\end{align}
where the per-sample loss is:
\begin{align}\nonumber
	&\mathcal{L}(\mathbf{Z}, \mathbf{V}, \mathbf{W},\mathbf{X},y) = -\text{log}(\text{softmax}(\mathbf{f}(\mathbf{Z},\mathbf{V},\mathbf{W},\mathbf{X}))_y).
\end{align}

\paragraph{Test error.}
We evaluate the generalization performance on an unseen dataset $\mathcal{S}'$ using the zero-one error rate:
\begin{equation}\nonumber 
\begin{aligned}
	&\mathcal{L}^{0-1}_{\mathcal{S}'}(\mathbf{Z}, \mathbf{V},\mathbf{W}) \\
	=& \frac{1}{|\mathcal{S}'|}\!\sum_{(\mathbf{X},y)\in\mathcal{S}'}\!\mathbb{P}[f_y(\mathbf{Z},\mathbf{V},\mathbf{W},\mathbf{X})\!<\!\max_{k\neq y}	f_k(\mathbf{Z},\mathbf{V},\mathbf{W},\mathbf{X})].
\end{aligned}
\end{equation}
The probability is taken over the random embeddings.
In this context, the metric represents the model's error rate, defined as the probability that the predicted score for the true class is lower than that of at least one alternative class on the test set.
\subsection{Initialization and Conditions}
\paragraph{Initialization.} All model parameters are initialized independently from a Gaussian distribution:
\begin{align}\nonumber
    Z_{i,j}^{(0)},\; V^{(0)}_{i,j},\; W^{(0)}_{l,k}\sim \mathcal{N}\left(0,\sigma_0^2\right), \quad \forall\, i,j,l\in[d],\; k\in[m],
\end{align}
where $\sigma_0$ controls the initialization scale.
\paragraph{Conditions.}
Our theoretical results require the following regularity conditions on the problem parameters:
\begin{condition}\label{condition:condition}
There exists a universal constant $C>0$ such that:

    \begin{enumerate}
        \item Large embedding dimension: $d\ge C\max\{\log(4md/\delta),\frac{m^3nN^2}{\lambda^6}\}$.
        \item Small initialization: $\sigma_0 \le C^{-1}\frac{\sqrt{nm}\log(d)}{(d\lambda)}$.
        \item Small learning rate: $\eta \le \left(C\frac{Nmd\sigma_0^2\log^2(d)}{\lambda^2}\right)^{-1}$.
    \end{enumerate}
\end{condition}
These conditions stabilize training and control attention dynamics: a sufficiently large embedding dimension $d$ keeps attention updates and softmax contributions well-scaled, while small initialization and learning rate ensure a favorable local landscape and stable gradient descent dynamics.

\section{Joint Training for Analogical Reasoning}\label{sec: joint}

\subsection{Training algorithm}
To simplify the analysis, we adopt a layer-wise training scheme following \citet{nichani2023provable}. Given training dataset $\mathcal{S}$, we first train the self-attention layer, then train the feature layer:

\paragraph{Stage 1: Attention training.} For $t \in [0,T_1]$, update the attention parameters by gradient descent:
\begin{align}\nonumber
	\mathbf{Z}^{(t+1)} =& \mathbf{Z}^{(t)} - \eta\nabla_{\mathbf{Z}^{(t)}} \mathcal{L}_{\mathcal{S}}(\mathbf{Z}^{(t)},\mathbf{V}^{(t)},\mathbf{W}^{(t)}),\\\nonumber
	\mathbf{V}^{(t+1)} =& \mathbf{V}^{(t)} - \eta\nabla_{\mathbf{V}^{(t)}} \mathcal{L}_{\mathcal{S}}(\mathbf{Z}^{(t)},\mathbf{V}^{(t)},\mathbf{W}^{(t)}),
\end{align}
while keeping $\mathbf{W}^{(t)}$ fixed.

\paragraph{Stage 2: Feature-layer training.}
For $t\in (T_1, T_1+T_2]$, update the feature-layer parameters by gradient descent:
\begin{align}\nonumber
\mathbf{W}^{(t+1)} = \mathbf{W}^{(t)} - \eta\nabla_{\mathbf{W}^{(t)}} \mathcal{L}_{\mathcal{S}}(\mathbf{Z}^{(t)},\mathbf{V}^{(t)},\mathbf{W}^{(t)}),
\end{align}
while keeping $\mathbf{Z}^{(t)}$ and $\mathbf{V}^{(t)}$ fixed.

\subsection{Main result}
We first analyze the case where the model is trained jointly on all premises. Let $\biguplus$ denote the multiset union (disjoint union). The training set is $\biguplus_{k=1}^{\kappa}(\mathcal{S}_1 \cup \mathcal{S}_2) \biguplus \mathcal{S}_3$, where $\kappa$ controls the relative frequency of similarity premise samples.
\begin{theorem}[Joint training on $\biguplus_{k=1}^{\kappa}(\mathcal{S}_1 \cup \mathcal{S}_2) \biguplus\mathcal{S}_3$]\label{thm:joint}
Suppose $\kappa = \Omega(\frac{m^{1/5}N^{1/5}\log^{2/5}(d)}{\lambda^{4/5}})$, $T_1 = \Theta(\oldfrac{mn\log(d)}{\kappa\lambda^2\eta\sqrt{d}\sigma_0})$ and $T_2 = \Theta(\oldfrac{\kappa mN^2}{\lambda^2d\sigma_0^2\eta })$.
Under Condition \ref{condition:condition}, with probability at least $1-\delta$, there exists an iteration $t\in(T_1, T_1+T_2]$ such  that
\begin{itemize}
    \item The training loss converges, i.e.,
    \begin{align}
        \mathcal{L}_{\biguplus_{k=1}^{\kappa}(\mathcal{S}_1 \cup \mathcal{S}_2) \biguplus\mathcal{S}_3}(\mathbf{Z}^{(t)},\mathbf{V}^{(t)},\mathbf{W}^{(t)}) \le \frac{0.01}{\kappa N}.
    \end{align}
    \item The trained model can perfectly analogize, i.e.,
    \begin{align}
        \mathcal{L}^{0-1}_{\mathcal{A}}(\mathbf{Z}^{(t)},\mathbf{V}^{(t)},\mathbf{W}^{(t)}) = 0.
    \end{align}
    
\end{itemize}
\end{theorem}
A high-level proof roadmap for the positive and negative results is provided in Appendix~\ref{app:proof-roadmap}.
Theorem~\ref{thm:joint} establishes that joint training on similarity and attribution premises is sufficient for analogical reasoning to emerge in this stylized setting.
This emergence is driven by a feature resemblance mechanism. During training, the model learns to align representations of entities with shared properties:
\begin{proposition}[Feature similarity]\label{prop: joint_fea}
    Under the same conditions as Theorem \ref{thm:joint}, for all $i\in[N]$, we have
    $$\frac{\langle\mathbf{V}^{(T_1)}\mathbf{a}_i, \mathbf{V}^{(T_1)}\mathbf{a}'_i\rangle}{\|\mathbf{V}^{(T_1)}\mathbf{a}_i\|_2\|\mathbf{V}^{(T_1)}\mathbf{a}'_i\|_2}= 1-o(1).$$
\end{proposition}
This high cosine similarity ensures that the value matrix $\mathbf{V}$ maps source $\mathbf{a}'_i$ and target $\mathbf{a}_i$ to aligned representations. 
Consequently, in Stage 2 (feature-layer training), properties learned for $\mathbf{a}'_i$ from $\mathcal{S}_3$ automatically transfer to $\mathbf{a}_i$ in the test set $\mathcal{A}$, enabling zero-shot analogical inference.

\begin{remark}[On the large-$\kappa$ regime]
Empirically, trained models exhibit analogical reasoning even with $\kappa = 1$ (equal frequency for all premise types). In the infinite-time limit ($T \rightarrow \infty$), this can be explained via the max-margin implicit bias of cross-entropy loss~\cite{lyu2019gradient,soudry2018implicit}. However, finite-time analysis with polynomial iteration complexity requires precise characterization of constants and training dynamics, which becomes technically challenging for small $\kappa$. We therefore present results for the large-$\kappa$ regime and defer the small-$\kappa$ case to future work.
\end{remark}

\section{Sequential Training for Analogical Reasoning}\label{sec:sequential}
Having established that joint training enables analogical reasoning, we now investigate whether the \emph{order} in which premises are learned affects the emergence of this capability. We study two sequential training curricula:
\begin{itemize}
    \item  \textbf{Similarity-then-Attribution (S$\to$A):} The model first trains on the similarity premises $\mathcal{S}_1 \cup \mathcal{S}_2$ (establishing that $\mathbf{a}_i$ and $\mathbf{a}'_i$ share property $\mathbf{b}_i$), then continues training on the attribution premise $\mathcal{S}_3$ (establishing that $\mathbf{a}'_i$ has property $\mathbf{c}_i$).

    \item \textbf{Attribution-then-Similarity (A$\to$S):} The model first trains on $\mathcal{S}_1 \cup \mathcal{S}_3$ (learning properties without establishing similarity), then continues training on $\mathcal{S}_2$ (learning the similarity between $\mathbf{a}_i$ and $\mathbf{a}'_i$).
\end{itemize}
As we will show, only the S$\to$A curriculum successfully supports analogical reasoning, revealing a critical curriculum effect: the model must learn relational structure before specific properties.

\subsection{Training algorithm}
The sequential training procedure follows the same layer-wise gradient descent scheme as Section~\ref{sec: joint}, but with two consecutive training phases on different datasets $\mathcal{B}_1$ and $\mathcal{B}_2$:
\paragraph{Phase 1: Training on $\mathcal{B}_1$.}
For $t \in [0, T_1]$, update the attention parameters by gradient descent:
\begin{align}\nonumber
	\mathbf{Z}^{(t+1)} =& \mathbf{Z}^{(t)} - \eta\nabla_{\mathbf{Z}^{(t)}} \mathcal{L}_{\mathcal{B}_1}(\mathbf{Z}^{(t)},\mathbf{V}^{(t)},\mathbf{W}^{(t)}),\\\nonumber
	\mathbf{V}^{(t+1)} =& \mathbf{V}^{(t)} - \eta\nabla_{\mathbf{V}^{(t)}} \mathcal{L}_{\mathcal{B}_1}(\mathbf{Z}^{(t)},\mathbf{V}^{(t)},\mathbf{W}^{(t)}).
\end{align}
For $t\in (T_1, T_1+T_2]$, update the feature-layer parameters by gradient descent:
\begin{align}\nonumber
\mathbf{W}^{(t+1)} =& \mathbf{W}^{(t)} - \eta\nabla_{\mathbf{W}^{(t)}} \mathcal{L}_{\mathcal{B}_1}(\mathbf{Z}^{(t)},\mathbf{V}^{(t)},\mathbf{W}^{(t)}),
\end{align}
while keeping $\mathbf{Z}^{(t)}$ and $\mathbf{V}^{(t)}$ fixed.
\paragraph{Phase 2: Training on $\mathcal{B}_2$.}
For $t\in (T_1+T_2, T_1+T_2+T_3]$, update the feature-layer parameters by gradient descent:
\begin{align}\nonumber
\mathbf{W}^{(t+1)} =& \mathbf{W}^{(t)} - \eta\nabla_{\mathbf{W}^{(t)}} \mathcal{L}_{\mathcal{B}_2}(\mathbf{Z}^{(t)},\mathbf{V}^{(t)},\mathbf{W}^{(t)}),
\end{align}
while keeping $\mathbf{Z}^{(t)}$ and $\mathbf{V}^{(t)}$ fixed at their values from the end of Phase 1.

\begin{remark}
In Phase 2, we train only the feature layer because the attention layer has already learned the relational structure in Phase 1. This design simplifies analysis while preserving the key property that Phase 1 features dominate subsequent training, ensuring feature similarities remain stable. Further discussions are provided in Appendix \ref{sec: ene}.
\end{remark}

\subsection{Main result}
We now present our central finding: although both sequential curricula expose the model to all premises, \emph{only} the Similarity-then-Attribution (S$\to$A) order supports analogical reasoning. 
The reverse order (A$\to$S) fails, revealing a critical curriculum dependence.

\begin{theorem}[S$\to$A curriculum enables analogical reasoning]\label{thm: sequential 123}
Suppose $T_1 = \Theta(\oldfrac{mn\log(d)}{\lambda^2\eta\sqrt{d}\sigma_0})$, $T_2 = \Theta(\oldfrac{mN^2}{\lambda^2d\sigma_0^2\eta })$ and $T_3 = \Theta(\oldfrac{mN^2}{\lambda^2d\sigma_0^2\eta })$.
Under Condition \ref{condition:condition}, with probability at least $1-\delta$,
there exists $t_1 \in(T_1,T_1+T_2]$ and $t_2\in(T_1+T_2,T_1+T_2+T_3]$ such that
    \begin{itemize}
        \item  The trained model converges on both phases, i.e.,
        $$\mathcal{L}_{\mathcal{S}_1\cup\mathcal{S}_2}(\mathbf{Z}^{(t_1)},\mathbf{V}^{(t_1)},\mathbf{W}^{(t_1)}) \le \frac{0.01}{N},$$ $$\mathcal{L}_{ \mathcal{S}_3}(\mathbf{Z}^{(t_2)},\mathbf{V}^{(t_2)},\mathbf{W}^{(t_2)}) \le \frac{0.01}{N}.$$
        \item  The trained model can perfectly analogize, i.e., $$\mathcal{L}^{0-1}_\mathcal{A}(\mathbf{Z}^{(t_2)},\mathbf{V}^{(t_2)},\mathbf{W}^{(t_2)})= 0.$$
    \end{itemize}
\end{theorem}
Theorem~\ref{thm: sequential 123} establishes that learning similarity structure \emph{before} specific properties is sufficient for analogical reasoning to emerge. This success is driven by the same feature resemblance mechanism identified in joint training:

\begin{proposition}[Feature similarity in S$\to$A curriculum]
    Under the same conditions as Theorem \ref{thm: sequential 123}, for all $i\in[N]$, we have
    $$\frac{\langle\mathbf{V}^{(T_1)}\mathbf{a}_i, \mathbf{V}^{(T_1)}\mathbf{a}'_i\rangle}{\|\mathbf{V}^{(T_1)}\mathbf{a}_i\|_2\|\mathbf{V}^{(T_1)}\mathbf{a}'_i\|_2}= 1-o(1).$$
\end{proposition}

We next investigate the alternative regime, where the model is first trained on the combined set of similarity and attribution premises ($\mathcal{S}_1 \cup \mathcal{S}_3$) and subsequently on the remaining similarity premises ($\mathcal{S}_2$).

\begin{theorem}[A$\to$S curriculum fails to enable analogical reasoning]
Suppose that $T_1 = \Theta(\oldfrac{mn\log(d)}{\lambda^2\eta\sqrt{d}\sigma_0})$, $T_2 = \Theta(\oldfrac{mN^2}{\lambda^2d\sigma_0^2\eta})$ and $T_3 = \Theta(\oldfrac{mN^2}{\lambda^2d\sigma_0^2\eta})$.
Under Condition \ref{condition:condition}, with probability at least $1-\delta$,
there exists $t_1 \in(T_1,T_1+T_2]$ and $t_2\in(T_1+T_2,T_1+T_2+T_3]$ such that
    \begin{itemize}
        \item  The trained model converges in both phases, i.e.,
        $$\mathcal{L}_{\mathcal{S}_1\cup\mathcal{S}_3}(\mathbf{Z}^{(t_1)},\mathbf{V}^{(t_1)},\mathbf{W}^{(t_1)}) \le \frac{0.01}{N},$$ $$\mathcal{L}_{ \mathcal{S}_2}(\mathbf{Z}^{(t_2)},\mathbf{V}^{(t_2)},\mathbf{W}^{(t_2)}) \le \frac{0.01}{N}.$$
        \item For all $t_3\in[0, T_1+T_2+T_3]$, the trained model fails to analogize i.e., $$\mathcal{L}^{0-1}_\mathcal{A}(\mathbf{Z}^{(t_3)},\mathbf{V}^{(t_3)},\mathbf{W}^{(t_3)})= 1- \frac{1}{N}.$$
    \end{itemize}
\end{theorem}

Theorem~\ref{thm: sequential 132} reveals a striking asymmetry: despite achieving low training loss on both datasets, the A$\to$S curriculum completely fails at analogical reasoning (achieving only random-level performance on the test set). The root cause is the \emph{absence of feature alignment}:

\begin{proposition}[No feature alignment in A$\to$S curriculum]\label{thm: sequential 132}
    Under the same conditions as Theorem \ref{thm: sequential 132}, for all $i\in[N]$, we have
    $$\frac{\langle\mathbf{V}^{(T_1)}\mathbf{a}_i, \mathbf{V}^{(T_1)}\mathbf{a}'_i\rangle}{\|\mathbf{V}^{(T_1)}\mathbf{a}_i\|_2\|\mathbf{V}^{(T_1)}\mathbf{a}'_i\|_2}= o(1).$$
\end{proposition}

Unlike the S$\to$A curriculum where $\mathbf{a}_i$ and $\mathbf{a}'_i$ become nearly aligned (cosine similarity $\ge 1-o(1)$), the A→S curriculum produces nearly orthogonal representations (cosine similarity $o(1)$). This occurs because Phase 1 trains the feature layer to predict properties of $\mathbf{a}_i$ without establishing similarity to $\mathbf{a}'_i$. 
This reveals a fundamental curriculum effect: relational structure should be learned before specific properties for analogical reasoning to emerge.

\section{Two-Hop Reasoning through Analogy}\label{sec: 2-hop}

Having characterized analogical reasoning in our controlled setting, we next show that a two-hop reasoning problem can be represented within the same formulation by adding an identity bridge.

\paragraph{The connection.}
Two-hop reasoning requires inferring $A \to C$ from the chain $A \to B$ and $B \to C$~\cite{feng2024extractive,lindsey2025biology}. At first glance, this appears distinct from analogical reasoning, which infers $A \to C$ from the premises $\{A \to B, A' \to B, A' \to C\}$. However, these become equivalent when we set the analogical source $A' = B$. Under this identification, the analogical premises become $\{A \to B, B \to B, B \to C\}$, revealing that two-hop reasoning is analogical reasoning with an \emph{identity bridge} $B \to B$.

\paragraph{The necessity of identity bridges.}
Crucially, the identity mapping $B \to B$ is not trivial—it must be explicitly represented in the training data. The model cannot automatically equate the output representation it generates when predicting $B$ (as the answer to $A \to ?$) with the input embedding of $B$ needed for the subsequent step ($B \to C$). Without explicit training examples of the form $B \to B$, the model fails to learn this bridge, and two-hop reasoning does not emerge. This demonstrates that, in our setting, two-hop reasoning relies on the same feature alignment mechanism as analogical reasoning.

We now formalize this connection and prove that, in our stylized setting, identity bridges are necessary for two-hop reasoning.

\subsection{Data structure}
\paragraph{Training dataset.} 
We construct training data with three components: first-hop relations $\mathcal{T}^h_1 = \{(\mathbf{a}_{i},\mathbf{r}_{1}, \mathbf{b}_{i})\}_{i=1}^N$, second-hop relations $\mathcal{T}^h_2 =  \{(\mathbf{b}_{i},\mathbf{r}_{2}, \mathbf{c}_{i})\}_{i=1}^N$, and identity bridges $\mathcal{T}^h_3=\{(\mathbf{b}_{i},\mathbf{r}_{3}, \mathbf{b}_{i})\}_{i=1}^N$. The corresponding training samples are:
$$\text{(First-Hop Relations)  } \mathcal{H}_1: \{([\mathbf{a}_i\:\:\mathbf{r}_1], \mathcal{I}(\mathbf{b}_i))\}_{i=1}^N,$$
\begin{align}\nonumber
    &\text{(Second-Hop Relations)  }\mathcal{H}_2: \{([\mathbf{b}_i\:\:\mathbf{r}_2], \mathcal{I}(\mathbf{c}_i))\}_{i=1}^N,\\\nonumber
    &\text{(Identity Bridges)  } \mathcal{IB}: \{([\mathbf{b}_i\:\:\mathbf{r}_3], \mathcal{I}(\mathbf{b}_i))\}_{i=1}^N.
\end{align}

\paragraph{Test dataset.} 
The test set $\mathcal{R}$ evaluates whether the model can compose the two hops:
\begin{align}\nonumber
    \text{(Two-Hop Reasoning Test)  } \mathcal{R}:\{([\mathbf{a}_i\:\:\mathbf{r}_2],\mathcal{I}(\mathbf{c}_i))\}_{i=1}^N.
\end{align}
As before, we assume all token embeddings are drawn as a random
orthonormal system: $\|\mathbf{z}_1\|_2 = 1$ and $\langle \mathbf{z}_1,\mathbf{z}_2 \rangle = 0$ for all distinct $\mathbf{z}_1,\mathbf{z}_2\in \{\mathbf{a}_{i},\mathbf{b}_{i}\}_{i=1}^N\cup \{\mathbf{r}_1,\mathbf{r}_2,\mathbf{r}_3\}$.

\subsection{Main results}
We now demonstrate that, under our stylized data construction and training setup, identity bridges are necessary for two-hop reasoning by comparing two training scenarios: with and without identity examples $\mathcal{IB}$.
\begin{theorem}[Identity bridges enable two-hop reasoning]\label{thm: 2h_id}
Suppose $\kappa = \Omega(\frac{m^{1/5}N^{1/5}\log^{2/5}(d)}{\lambda^{4/5}})$, $T_1 = \Theta(\oldfrac{mn\log(d)}{\lambda^2\eta\kappa\sqrt{d}\sigma_0})$ and $T_2 = \Theta(\oldfrac{\kappa mN^2}{\lambda^2d\sigma_0^2\eta })$.
Under Condition \ref{condition:condition}, with probability at least $1-\delta$, there exists an iteration $t \in (T_1,T_1+T_2]$ such that 
 
\begin{itemize}
            \item The trained model converges, i.e., $$\mathcal{L}_{\biguplus_{k=1}^{\kappa}(\mathcal{H}_1 \cup \mathcal{IB}) \biguplus\mathcal{H}_2}(\mathbf{Z}^{(t)},\mathbf{V}^{(t)},\mathbf{W}^{(t)}) \le \frac{0.01}{N\kappa}.$$
            \item The trained model can perfectly perform two-hop reasoning, i.e., $$\mathcal{L}^{0-1}_{\mathcal{R}} (\mathbf{Z}^{(t)},\mathbf{V}^{(t)},\mathbf{W}^{(t
            )}) = 0.$$
        \end{itemize}
\end{theorem}

The underlying mechanism is feature resemblance between the first-hop output $\mathbf{b}_i$ and its role as second-hop input:
\begin{proposition}[Feature resemblance with identity bridges]
    Under the same conditions as Theorem \ref{thm: 2h_id}, for all $i\in[N]$, we have
    $$\frac{\langle\mathbf{V}^{(T_1)}\mathbf{a}_i, \mathbf{V}^{(T_1)}\mathbf{b}_i\rangle}{\|\mathbf{V}^{(T_1)}\mathbf{a}_i\|_2\|\mathbf{V}^{(T_1)}\mathbf{b}_i\|_2}= 1-o(1).$$
\end{proposition}

Training on identity bridges $\mathcal{IB}$ aligns $\mathbf{a}_i$ and $\mathbf{b}_i$ in the learned representation space. This enables the model to compose $A \to B$ and $B \to C$ by transferring the learned mapping from $\mathbf{b}_i$ to $\mathbf{c}_i$ (trained via $\mathcal{H}_2$) onto the aligned representation of $\mathbf{a}_i$.

In contrast, without identity bridges, two-hop reasoning fails despite successful training:
\begin{theorem}[Trained without the identity bridge]\label{thm: 2h_woid}
Suppose that $T_1 = \Theta(\oldfrac{mn\log(d)}{\lambda^2\eta\sqrt{d}\sigma_0})$ and $T_2 = \Theta(\oldfrac{mN^2}{\lambda^2d\sigma_0^2\eta})$. 
Under Condition \ref{condition:condition}, with probability at least $1-\delta$, there exists an iteration $t\in(T_1,T_1+T_2]$ such that
        \begin{itemize}
            \item The trained model converges, i.e., $$\mathcal{L}_{\mathcal{H}_1\cup\mathcal{H}_2}(\mathbf{Z}^{(t)},\mathbf{V}^{(t)},\mathbf{W}^{(t)}) \le \frac{0.01}{N}.$$
            \item For any $t'\in[0,T_1+T_2]$, the trained model cannot perform two-hop reasoning, i.e., $$\mathcal{L}^{0-1}_{\mathcal{R}}(\mathbf{Z}^{(t')},\mathbf{V}^{(t')},\mathbf{W}^{(t')}) = 1-\frac{1}{N}.$$
        \end{itemize}
\end{theorem}
\begin{proposition}[No feature alignment without identity bridges]\label{prop: 2h_woid}
    Under the same conditions as Theorem \ref{thm: 2h_woid}, for all $i\in[N]$, we have
    $$\frac{\langle\mathbf{V}^{(T_1)}\mathbf{a}_i, \mathbf{V}^{(T_1)}\mathbf{b}_i\rangle}{\|\mathbf{V}^{(T_1)}\mathbf{a}_i\|_2\|\mathbf{V}^{(T_1)}\mathbf{b}_i\|_2}= o(1).$$
\end{proposition}
Without identity bridges, $\mathbf{a}_i$ and $\mathbf{b}_i$ remain nearly orthogonal in the learned representation. The model successfully learns $A \to B$ and $B \to C$ as independent mappings, but cannot compose them because the output representation from $A \to B$ is not aligned with the input representation needed for $B \to C$. This proves that, in our stylized setting, identity bridges are \emph{necessary} for two-hop reasoning to emerge.

\section{Feature Resemblance in Deep Linear Networks}\label{sec: multi_layer}
Thus far, our analysis has focused on a minimal transformer-style block. 
We next use deep linear networks as a stylized multi-layer model to study how feature resemblance evolves across depth.

\subsection{Setup}
\paragraph{Deep linear neural networks.}
We study $L$-layer linear neural networks, where each layer is parameterized by $\mathbf{W}_i\in\mathbb{R}^{d\times d}$ for $i\in[L]$. Given input $\mathbf{x}\in\mathbb{R}^d$, the network output is:
\begin{align}
    \mathbf{f}(\mathbf{W}_L,\cdots,\mathbf{W}_1,\mathbf{x}) = \mathbf{W}_L\cdots\mathbf{W}_1\mathbf{x}.
\end{align}
\paragraph{Training data.}
The training dataset $\mathcal{S}_o = \{(\mathbf{x}_i,y_i)\}_{i=1}^n$ consists of orthogonal inputs: $\langle\mathbf{x}_i,\mathbf{x}_j\rangle = 0$ for all $i\neq j$.
Assume that \(\mathbf{x}_1,\ldots,\mathbf{x}_n\) are drawn as a uniformly random orthonormal system in \(\mathbb R^d\), independently of the labels.
At most $\mathcal{O}(1)$ samples share the same label.

\paragraph{Layer-wise training.}
Following our earlier approach, we train one layer at a time. During iterations $t\in(\sum_{k=1}^i T_k,\sum_{k=1}^{i+1} T_k]$, only the $i$-th layer is updated via gradient descent:
\begin{equation}
\begin{aligned}\nonumber
    \mathbf{W}^{(t+1)}_i =& \mathbf{W}^{(t)}_i \\
    &- \frac{\eta}{n}\sum_{j\in[n]}\nabla_{\mathbf{W}^{(t)}_i} \mathcal{L}_{\mathcal{S}_o}(\mathbf{W}^{(t)}_L,\cdots,\mathbf{W}^{(t)}_1,\mathbf{x}_j,y_j).
\end{aligned}
\end{equation}
\subsection{Main result}
Our main result shows that multi-layer networks exhibit \emph{progressive} feature resemblance: representations of inputs with the same label become increasingly aligned as they propagate through deeper layers.

\begin{theorem}\label{thm: multi_layer}
Suppose the neural network is initialized with $\mathbf{W}_1^{(0)} = \cdots = \mathbf{W}_L^{(0)} = \mathbf{I}$ and $T_i=\Theta(\frac{n}{(\eta L)})$ for $i\in[1,L-1]$.
    Under the condition that $d^{0.9} > C'L^2\log(2nd/\delta)$ with a large constant $C'$, with probability at least $1-\delta$, for $k\in[1,L-2]$ and for any two inputs $(\mathbf{x}_i,y_i),(\mathbf{x}_j,y_j)\in \mathcal{S}_o$ satisfying $y_i=y_j$: 
    \begin{equation}
    \begin{aligned}
        &\frac{\langle \prod_{l=1}^{k+1}\mathbf{W}^{(\sum_{l=1}^{k+1} T_l)}_{l}\mathbf{x}_i,\prod_{l=1}^{k+1}\mathbf{W}^{(\sum_{l=1}^{k+1} T_l)}_{l}\mathbf{x}_j \rangle}{\| \prod_{l=1}^{k+1}\mathbf{W}^{(\sum_{l=1}^{k+1} T_l)}_{l}\mathbf{x}_i\|_2\| \prod_{l=1}^{k+1}\mathbf{W}^{(\sum_{l=1}^{k+1} T_l)}_{l}\mathbf{x}_j\|_2} \ge \\&\frac{\langle \prod_{l=1}^{k}\mathbf{W}^{(\sum_{l=1}^{k} T_l)}_{l}\mathbf{x}_i,\prod_{l=1}^{k}\mathbf{W}^{(\sum_{l=1}^{k} T_l)}_{l}\mathbf{x}_j \rangle}{\|\! \prod_{l=1}^{k}\!\mathbf{W}^{(\sum_{l=1}^{k} T_l)}_{l}\mathbf{x}_i\|_2\|\! \prod_{l=1}^{k}\!\mathbf{W}^{(\sum_{l=1}^{k} T_l)}_{l}\mathbf{x}_j\|_2} \!-\!\tilde{\mathcal{O}}(d^{-0.1}), 
    \end{aligned}
    \end{equation}
    and 
    \begin{align}
        \frac{\langle \prod_{l=1}^{L-1}\mathbf{W}^{(\sum_{l=1}^{L-1} T_l)}_{l}\mathbf{x}_i,\prod_{l=1}^{L-1}\mathbf{W}^{(\sum_{l=1}^{L-1} T_l)}_{l}\mathbf{x}_j \rangle}{\| \prod_{l=1}^{L-1}\mathbf{W}^{(\sum_{l=1}^{L-1} T_l)}_{l}\mathbf{x}_i\|_2\| \prod_{l=1}^{L-1}\mathbf{W}^{(\sum_{l=1}^{L-1} T_l)}_{l}\mathbf{x}_j\|_2} =\Omega(1).
    \end{align}
\end{theorem}
Starting from orthogonal inputs, alignment is approximately non-decreasing with depth, suggesting that layer-wise feature resemblance emerges in a multi-layer linear network.
Appendix~\ref{app: verf} further reports supporting evidence from GPT-2 trained on orthogonal data.

\begin{table*}[t]
	\caption{Training loss, feature similarity, and success rate: mean $\pm$ standard deviation of analogical reasoning and two-hop reasoning tasks over one-layer and GPT-2 transformers. GPT-2 is trained end-to-end with AdamW. We set $\kappa = 3$ and show the impact of $\kappa$ in Appendix \ref{app: kappa}.}
	\label{table: 3-token_1}
	\centering
    \setlength{\tabcolsep}{4pt}

	\begin{tabular}{lcccc}
		\toprule
		\multicolumn{1}{c}{\textbf{Training tasks for Analogical Reasoning}} &\multicolumn{1}{c}{\textbf{Architecture}} & \multicolumn{1}{c}{\textbf{Training loss}} & \multicolumn{1}{c}{\textbf{Feature similarity}} & \multicolumn{1}{c}{\textbf{Success rate on $\mathcal{A}$}  } \\
		\cmidrule(r){1-5} \cmidrule(lr){2-2} \cmidrule(lr){3-3} \cmidrule(l){4-4}
		Joint training on $\biguplus_{k=1}^{\kappa}(\mathcal{S}_1 \cup \mathcal{S}_2) \biguplus\mathcal{S}_3$ & GPT-2  &  0.0077$\pm$1e-4  &  0.995$\pm$6e-4  & 100\% $\pm$ 0\%\\
        Joint training on $\biguplus_{k=1}^{\kappa}(\mathcal{S}_1 \cup \mathcal{S}_2) \biguplus\mathcal{S}_3$        & One-layer  &  0.001$\pm$5e-5  &  0.9202$\pm$7e-4  & 100\% $\pm$ 0\%\\
		Late training on the attribution premise $\mathcal{S}_3$ & GPT-2 & 0.0845 $\pm$ 3e-4 & 0.7983 $\pm$ 0.027  & 100\% $\pm$ 0\%\\
		Late training on the attribution premise $\mathcal{S}_3$ & One-layer & 0.0024 $\pm$ 0 & 0.9689 $\pm$ 4e-4 & 100\% $\pm$ 0\% \\
	   Late training on the similarity premise $\mathcal{S}_2$ & GPT-2 & 0.0085 $\pm$ 1e-4 & 0.4326 $\pm$ 0.012 & 0.7\% $\pm$ 0.471\%\\
	   Late training on the similarity premise $\mathcal{S}_2$ & One-layer & 0.0061 $\pm$ 5e-5 & 0.0014 $\pm$ 3.7e-3 & 0\%$\pm$0\% \\
       \bottomrule
	\end{tabular}
    \hfill
    \begin{tabular}{lcccc}
		\toprule
		\multicolumn{1}{c}{\textbf{Training tasks for Two-Hop Reasoning}} &\multicolumn{1}{c}{\textbf{Architecture}} & \multicolumn{1}{c}{\textbf{Training loss}} & \multicolumn{1}{c}{\textbf{Feature similarity}} & \multicolumn{1}{c}{\textbf{Success rate on $\mathcal{R}$}  } \\
		\cmidrule(r){1-5} \cmidrule(lr){2-2} \cmidrule(lr){3-3} \cmidrule(l){4-4}
       Training with the identity bridge & GPT-2 & 0.0953 $\pm$ 4e-3 & 0.9642 $\pm$ 7e-4 & 99.7\% $\pm$ 0.5\%\\
       Training without the identity bridge & GPT-2 & 0.1024 $\pm$ 4e-4 & 0.016 $\pm$ 6e-3  & 0\% $\pm$ 0\%\\
       Training with the identity bridge & One-layer & 0.075 $\pm$ 4e-4 & 0.9132 $\pm$ 5e-4 & 100\% $\pm$ 0\% \\
       Training without the identity bridge & One-layer & 0.0676 $\pm$ 2e-4&0.006 $\pm$ 3e-3& 1.7\% $\pm$ 0.9\%\\
		\bottomrule
	\end{tabular}
\end{table*}
\section{Experiments}\label{sec:exp}
We conduct experiments with two complementary goals. 
First, we verify that the theoretical predictions hold within the stylized setting, using minimal transformer-style blocks and controlled synthetic datasets. 
These experiments confirm that feature alignment and property transfer emerge as predicted by our analysis. 

Second, we test whether similar qualitative patterns appear in larger pretrained language models and natural-language datasets, including models up to 8B parameters. 
While these models are not minimal abstractions, the experiments illustrate that the relationship between representational geometry and analogical reasoning observed in the theory can also be observed in practical settings.

We conduct two sets of experiments:

\textbf{Controlled setting:} One-layer transformers and GPT-2 trained on synthetic 3-token datasets to directly test our theoretical predictions.

\textbf{Natural-language setting:} Pretrained models (Llama-3-1B, Qwen-2.5-1.5B, Qwen-3-8B) evaluated on natural language benchmarks to test whether feature resemblance extends beyond the minimal theoretical setup.

We quantify feature similarity using cosine similarity between entity representations:
\begin{align}
    \text{Feature Similarity}(\mathbf{a}_i, \mathbf{a}_i') = \frac{\langle\mathbf{g}(\mathbf{a}_i), \mathbf{g}(\mathbf{a}_i')\rangle}{\|\mathbf{g}(\mathbf{a}_i)\|_2\|\mathbf{g}(\mathbf{a}_i')\|_2},
\end{align}
where $\mathbf{g}(\mathbf{a}_i) = \mathbf{V}\mathbf{a}_i$ for one-layer transformers (matching our theoretical results) and $\mathbf{g}(\mathbf{a}_i) = \mathbf{h}_L(\mathbf{a}_i)$ for multi-layer transformers (the final hidden representation before the linear prediction head at the last token).

\subsection{Controlled synthetic data experiments}

We train one-layer transformers and GPT-2 on synthetic datasets matching the structure from Section~\ref{sec: data}. For one-layer transformers, we precisely replicate the theoretical setting; training details are in Appendix~\ref{app: exp}.

Table~\ref{table: 3-token_1} shows that both architectures exhibit consistent trends across all training regimes (joint and sequential): feature alignment scores and task success rates closely follow the theoretical predictions, \emph{showing strong qualitative agreement under standard end-to-end training.}

\subsection{Controlled natural-language reasoning experiments}

We construct a factual knowledge dataset using Gemini 3 Pro (examples in Table \ref{tab:extended_data}, Appendix \ref{app: exp_nlp}). The dataset consists of three components:

\textbf{Similarity premises} express shared features between two entities using natural language templates, e.g., ``Chair is used for resting'' and ``Sofa is used for resting.''

\textbf{Attribution premises} explicitly define one entity's category, e.g., ``Sofa is classified as Furniture.''

\textbf{Reasoning task:} Given these premises, the model must predict the other entity's category (e.g., that Chair is also Furniture).

We fine-tune pre-trained Llama-3-1B and Qwen-2.5-1.5B models under three training regimes: joint training, late training on attribution, and late training on similarity. 
Table~\ref{table: real} reports both the feature similarity between the two entities and the success rate on the reasoning task.

Across both models, joint training and late training on attribution substantially outperform late training on similarity, consistent with our theoretical prediction. 
The feature-similarity measurements follow the same qualitative pattern: joint training and late-attribution training produce more aligned representations than late-similarity training. 
Notably, all learned feature similarities remain relatively high, exceeding 0.6 across all conditions. 
This stems from representation degeneration \cite{gao2019representation}, where features compress into a narrow cone during training. Despite this compression limiting geometric separation, performance gaps remain substantial, with late similarity training underperforming by 20\%-52\%.

\begin{table}[h]
    \caption{Feature similarity and success rate over Llama-3-1B and Qwen-2.5-1.5B.}
    \label{table: real}
    \centering
    \footnotesize
    \setlength{\tabcolsep}{4pt}
    \begin{tabular}{lccc}
        \toprule
        \textbf{Models} & \textbf{Tasks} & \textbf{Fea. Sim.} & \textbf{Success rate} \\
        \midrule
        Llama &Joint training & 0.8174 & 73\% \\
        &Late on attribution & 0.7990 & 86\% \\
        &Late on similarity & 0.6431 & 53\% \\
        %&Without training & 0.3379 & 15\% \\
        \midrule
        Qwen &Joint training & 0.8488 & 85\% \\
        &Late on attribution & 0.7971 & 72\% \\
        &Late on similarity & 0.7455 & 33\% \\
        %&Without training & 0.0175 & 13\% \\
        \bottomrule
    \end{tabular}
\end{table}

We also evaluate Qwen-3-8B under the same three regimes. Raw full-representation cosine similarities are less separated in this larger model, likely because global hidden states contain task-irrelevant shared components. We therefore apply a sparse top-$K$ activation-space intervention, retaining the top 20\% MLP activations at relational tokens. As shown in Table~\ref{tab:feature-sim-success}, this focused measure recovers the same ordering as the behavioral results: joint training and late-attribution training yield higher feature similarity and substantially higher success rates than late-similarity training.

\begin{table}[t]
\centering
\caption{Feature similarity and success rate over Qwen-3-8B.}
\begin{tabular}{lcc}
\toprule
\textbf{Tasks} & \textbf{Fea. Sim.} & \textbf{Success rate} \\
\midrule
Joint training & 0.3785 & 84\% \\
Late on Attribution & 0.3757 & 75\% \\
Late on Similarity & 0.2704 & 34\% \\
\bottomrule
\end{tabular}

\label{tab:feature-sim-success}
\end{table}

\section{Conclusion}
We studied feature resemblance as a mechanism for analogical reasoning in stylized transformer-style models. Our results show that joint training aligns analogous entities, sequential training succeeds only when similarity is learned before attribution, and two-hop reasoning can be supported by explicit identity bridges. Together, these results identify a geometric mechanism: shared properties align representations, enabling attributes to transfer through shared features. Experiments on synthetic and natural-language settings show qualitative agreement with the theory, suggesting that representational alignment helps explain analogical transfer beyond the stylized model.

\section*{Impact Statement}
This work advances the theoretical understanding of analogical reasoning in language models by identifying a geometric feature-alignment mechanism. The results may inform data curricula for robust reasoning, but similar mechanisms could also be used to steer model behavior through crafted training data. Such risks should be considered in high-stakes applications.

\bibliographystyle{icml2026}
\bibliography{ref}

\newpage
\appendix
\onecolumn

\section{Additional related work}\label{app:deep_rw}
\paragraph{Related work on deep linear networks.}
Deep linear networks have been widely used as analytically tractable models
for studying optimization and representation dynamics in deep learning.
\citet{saxe2014exact} derive exact solutions for the nonlinear
learning dynamics of deep linear networks under gradient flow, showing how
different input-output modes are learned over time. A related line of work
studies the optimization landscape of deep linear networks:
\citet{kawaguchi2016deep} shows that deep linear networks have no
poor local minima under squared loss, while \citet{laurent2018deep} extend such landscape guarantees to arbitrary
convex differentiable losses. Other works investigate how depth and
parameterization affect optimization and implicit bias. \citet{hardt2017identity} emphasize the role of identity
parameterization in deep linear residual networks, \citet{arora2018optimization} show that depth-induced
overparameterization can accelerate optimization in linear networks, and
\citet{arora2019implicit} study implicit regularization in
deep matrix factorization.

Our use of deep linear networks is complementary to this literature.
Rather than aiming to characterize the full optimization landscape or exact
gradient-flow dynamics, we use a layer-wise gradient-descent setting to
isolate a specific phenomenon relevant to our paper: representations of
inputs sharing the same label progressively acquire a common feature
component and therefore become increasingly aligned across layers. Thus,
the analysis in Section~7 should be viewed as a stylized multi-layer
illustration of feature resemblance, rather than as a general theory of
deep linear networks.
\section{Discussion of layer-wise training}\label{sec: ene}
In this section, we present the experimental results of \textbf{end-to-end} training of one-layer transformers with \textbf{ReLU} activations.
The results show that end-to-end training shares the same trend as layer-wise training. We set $\kappa = 1$.
\begin{table*}[h]
	\caption{Training loss, feature similarity, and success rate: mean $\pm$ standard deviation of analogical reasoning tasks over one-layer transformers. These results are under \textbf{end-to-end training}.}
	\label{table: 3-token2}
	\centering
    \setlength{\tabcolsep}{4pt}

	\begin{tabular}{lcccc}
		\toprule
		\multicolumn{1}{c}{\textbf{Training tasks for Analogical Reasoning}} &\multicolumn{1}{c}{\textbf{Architecture}} & \multicolumn{1}{c}{\textbf{Training loss}} & \multicolumn{1}{c}{\textbf{Feature similarity}} & \multicolumn{1}{c}{\textbf{Success rate on $\mathcal{A}$}  } \\
		\cmidrule(r){1-5} \cmidrule(lr){2-2} \cmidrule(lr){3-3} \cmidrule(l){4-4}
		Joint training on $\biguplus_{k=1}^{\kappa}(\mathcal{S}_1 \cup \mathcal{S}_2) \biguplus\mathcal{S}_3$         & One-layer  & 0.0133 $\pm$2e-4  &  0.9413$\pm$ 3e-4  & 100\% $\pm$ 0\%\\
		Late training on the attribution premise $\mathcal{S}_3$ & One-layer &  0.0198$\pm$1e-4  & 0.7116$\pm$2.2e-3  & 100\% $\pm$ 0\% \\
	   Late training on the similarity premise $\mathcal{S}_2$ & One-layer &  0.0138$\pm$1e-4 &  0.4367$\pm$ 1.2e-3 & 0.67\%$\pm$0.47\% \\
       \bottomrule
	\end{tabular}
\end{table*}

Table \ref{table: 3-token2} shows that one-layer transformers with end-to-end training and ReLU activations exhibit the same success rate and feature similarity trends as layer-wise training in Table \ref{table: 3-token_1}, confirming that Phase 1 features dominate Phase 2 learning. 
This dominance explains the reduced feature similarity in late training on attribution premises. Our theoretical analysis therefore simplifies Phase 2 by training only the MLP layer, preserving essential properties while enabling tractable analysis.

\section{Impact of \texorpdfstring{$\kappa$}{kappa}}\label{app: kappa}
In this section, we evaluate the impact of $\kappa$ in joint training on dataset $\biguplus_{k=1}^{\kappa}(\mathcal{S}_1 \cup \mathcal{S}_2) \biguplus\mathcal{S}_3$.
\begin{table*}[h]
	\caption{Training loss, feature similarity, and success rate: mean $\pm$ standard deviation of analogical reasoning tasks over one-layer and GPT-2 transformers. These results are under \textbf{end-to-end training}.}
	\label{table: kappa}
	\centering

	\begin{tabular}{lcccc}
		\toprule
		\multicolumn{1}{c}{\textbf{Values of $\kappa$}} &\multicolumn{1}{c}{\textbf{Architecture}} & \multicolumn{1}{c}{\textbf{Training loss}} & \multicolumn{1}{c}{\textbf{Feature similarity}} & \multicolumn{1}{c}{\textbf{Success rate on $\mathcal{A}$}  } \\
		\cmidrule(r){1-5} \cmidrule(lr){2-2} \cmidrule(lr){3-3} \cmidrule(l){4-4}
		$\kappa=1$        &GPT-2  & 0.0294$\pm$ 3e-4 & 0.9756 $\pm$ 5e-4 & 100\% $\pm$ 0\%\\
		$\kappa=1$       & One-layer  & 0.0079 $\pm$ 5e-5 &  0.6706$\pm$ 9e-4  & 100\% $\pm$ 0\%\\
        $\kappa = 3$        & GPT-2  &  0.0077$\pm$1e-4  &  0.995$\pm$6e-4  & 100\% $\pm$ 0\%\\
        $\kappa = 3$        & One-layer  &  0.001$\pm$5e-5  &  0.9202$\pm$7e-4  & 100\% $\pm$ 0\%\\
        $\kappa = 5$        & GPT-2 &  0.0039$\pm $5e-4  & 0.9979$\pm$3e-4   & 100\% $\pm$ 0\%\\
        $\kappa = 5$        & One-layer &  5e-4$\pm$1e-4  &  0.9405$\pm$2e-4  & 100\% $\pm$ 0\%\\
       \bottomrule
	\end{tabular}
\end{table*}

Table \ref{table: kappa} shows that the feature similarity increases with $\kappa$. 
This is consistent with our analysis in Proposition \ref{prop: joint_fea} that when $\kappa$ is large, the cosine similarity tends to 1.

\section{Verification of layer-wise feature resemblance}\label{app: verf}
We verify layer-wise feature resemblance in two settings:
\begin{itemize}
\item Linear neural networks under the setting defined in Section \ref{sec: multi_layer}.
\item GPT-2 models trained end-to-end on orthogonal data using AdamW.
\end{itemize}
For each setting, we compute the average cosine similarity of penultimate layer representations at the final token position for inputs sharing the same label. Figure 1 shows that both architectures exhibit layer-wise feature resemblance, confirming that this phenomenon emerges in modern multi-layer networks.

\begin{figure}[t]
    \centering
    \label{fig:multi_layer}
    \includegraphics[width=0.4\linewidth]{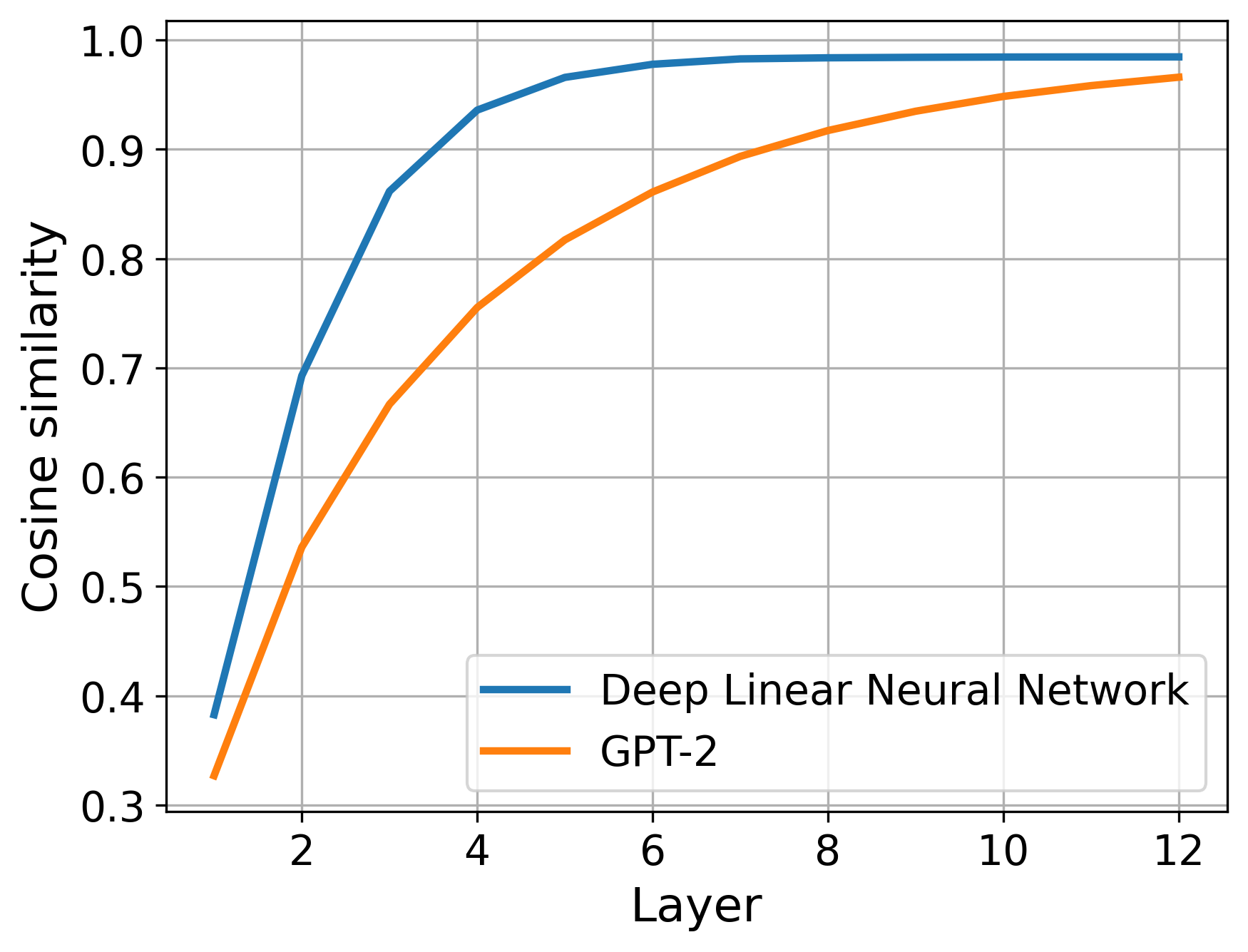}
    \caption{Feature cosine similarity of data with the same labels of deep linear neural networks and GPT-2 trained on orthogonal data.}
\end{figure}
\section{Proof Roadmap and Gradient Computations}
\subsection{Proof roadmap}
\label{app:proof-roadmap}

We first illustrate the proof strategy using Theorem~1, which is the representative positive result for analogical reasoning. The proofs of Theorems~2 and~4 follow the same alignment-and-transfer template, with different data curricula producing the required alignment. Theorems~3 and~5 follow the complementary no-alignment argument.

\paragraph{Theorem~1: joint training.}
The proof has three steps. First, during the attention/value training stage, the samples $[\mathbf{a}_i,\mathbf{r}_1]$ and $[\mathbf{a}'_i,\mathbf{r}_1]$ share the same target label $\mathcal{I}(\mathbf{b}_i)$. Early in training, the model is far from convergence, so the correct-class residual remains of constant order. Consequently, the value representations of $\mathbf{a}_i$ and $\mathbf{a}'_i$ receive gradient updates with a common component in the direction of the initial classifier weights associated with label $\mathcal{I}(\mathbf{b}_i)$. Concentration bounds show that cross-sample interactions and initialization noise are lower-order terms. Hence $\mathbf{V}^{(T_1)}\mathbf{a}_i$ and $\mathbf{V}^{(T_1)}\mathbf{a}'_i$ become aligned, yielding the feature-resemblance statement in Proposition~1.

Second, the subsequent layer is then trained to converge to a near-optimal solution. This is theoretically guaranteed by the homogeneity of linear functions combined with the convexity of the cross-entropy loss, ensuring global convergence with polynomial iterations.

Third, this margin transfers from $[\mathbf{a}'_i,\mathbf{r}_2]$ to the unseen query $[\mathbf{a}_i,\mathbf{r}_2]$. Because $\mathbf{V}^{(T_1)}\mathbf{a}_i$ and $\mathbf{V}^{(T_1)}\mathbf{a}'_i$ are aligned, replacing $\mathbf{a}'_i$ by $\mathbf{a}_i$ introduces only a smaller-order perturbation to the logits. The positive margin therefore remains, and the model predicts $\mathcal{I}(\mathbf{c}_i)$ on the test query.

\paragraph{Extensions to Theorems~2 and~4.}
Theorem~2 follows the same argument under the similarity-then-attribution curriculum. The first phase aligns $\mathbf{a}_i$ and $\mathbf{a}'_i$, and the second phase learns the attribution margin. Theorem~4 is analogous: the identity-bridge examples align $\mathbf{a}_i$ with the intermediate token $\mathbf{b}_i$, allowing the margin learned from $[\mathbf{b}_i,\mathbf{r}_2]\mapsto \mathcal{I}(\mathbf{c}_i)$ to transfer to the test query $[\mathbf{a}_i,\mathbf{r}_2]$.

\paragraph{Negative results: Theorems~3 and~5.}
The negative results use the complementary argument. In the attribution-then-similarity curriculum of Theorem~3, the model can fit the observed training queries but does not receive an early common gradient signal aligning $\mathbf{a}_i$ with $\mathbf{a}'_i$. Similarly, without identity bridges in Theorem~5, there is no systematic signal aligning $\mathbf{a}_i$ with $\mathbf{b}_i$. By symmetry and concentration, the relevant cosine similarities remain $o(1)$. Thus the positive training margin cannot be transferred to the unseen test queries, leading to failure on analogical or two-hop reasoning.

\subsection{Gradient computation}
In this subsection, we derive the closed-form expression of the model gradient.
We define 
\begin{align}
	\textbf{logit}(\mathbf{Z},\mathbf{V},\mathbf{W},\mathbf{X}) =& \text{softmax}(\mathbf{f}(\mathbf{Z},\mathbf{V},\mathbf{W},\mathbf{X})),\\
    \mathbf{x}_a(\mathbf{Z},\mathbf{X}) =& \mathbf{X}\boldsymbol{\alpha}(\mathbf{Z},\mathbf{X}),
\end{align}
and 
\begin{align}
    \mathbf{W}_2 = \begin{bmatrix}
        \oldfrac{1}{m}\mathbf{1}^\top & \mathbf{0}^\top & \cdots &\mathbf{0}^\top\\
        \mathbf{0}^\top & \oldfrac{1}{m}\mathbf{1}^\top & \cdots &\mathbf{0}^\top\\
        \vdots & \vdots &&\vdots\\
        \mathbf{0}^\top & \mathbf{0}^\top & \cdots &\oldfrac{1}{m}\mathbf{1}^\top,
    \end{bmatrix}
\end{align}
Then, we can rewrite the output as
\begin{align}
    \mathbf{f}(\mathbf{Z},\mathbf{V},\mathbf{W},\mathbf{X}) = \lambda \mathbf{W}_2\mathbf{W}\mathbf{VX}\boldsymbol{\alpha}(\mathbf{Z},\mathbf{X}).
\end{align}
\begin{lemma}
	For each $(\mathbf{X},y)\in\mathcal{S}$, the gradients of the transformer are given as 
	\begin{align}
		\nabla_{\mathbf{w}_{i,j}} \mathcal{L}(\mathbf{Z},\mathbf{V},\mathbf{W},\mathbf{X},y) = -\frac{\lambda}{m}(\mathbb{I}(i=y)-\textnormal{logit}_i(\mathbf{Z},\mathbf{V},\mathbf{W},\mathbf{X}))\mathbf{V}\mathbf{x}_a(\mathbf{Z},\mathbf{X}),
	\end{align}
	\begin{align}
	\nabla_{\mathbf{v}_{l}} \mathcal{L}(\mathbf{Z},\mathbf{V},\mathbf{W},\mathbf{X},y) = \lambda(\textbf{logit}(\mathbf{Z},\mathbf{V},\mathbf{W},\mathbf{X})-\mathbf{e}_y)^\top\mathbf{W}_2 \mathbf{W}
\begin{bmatrix} 
	\mathbf{0}\\
	\vdots\\
	\mathbf{x}_a^\top(\mathbf{Z},\mathbf{X})\\
	\vdots\\
	\mathbf{0}
\end{bmatrix},
	\end{align}
and
\begin{equation}
\begin{aligned}
	&(\nabla _{\mathbf{Z}}\mathcal{L}(\mathbf{Z},\mathbf{V},\mathbf{W},\mathbf{X},y)\mathbf{X}[-1])^\top\\
    =&\frac{\lambda}{\sqrt{d}}(\textbf{logit}(\mathbf{Z},\mathbf{V},\mathbf{W},\mathbf{X})-\mathbf{e}_y)^\top\mathbf{W}_2\mathbf{W}\mathbf{VX}\left(\textnormal{diag}(\boldsymbol{\alpha}(\mathbf{Z},\mathbf{X})) - \boldsymbol{\alpha}(\mathbf{Z},\mathbf{X})\boldsymbol{\alpha(\mathbf{Z},\mathbf{X})}^\top\right)\mathbf{X}^\top.
\end{aligned}
\end{equation}
\end{lemma}
\begin{proof}
To unclutter the notation, we hereafter suppress the arguments in the following proof.

Using chain rule, we have
\begin{align}
	&\nabla_{\mathbf{w}_{i,j}} \mathcal{L} = \nabla_{\mathbf{f}}
	\mathcal{L}\nabla_{\mathbf{w}_{i,j}}\mathbf{f}, \\
	&\nabla_{\mathbf{V}} \mathcal{L} = \nabla_{\mathbf{f}}
	\mathcal{L}\nabla_{\mathbf{o}_1}\mathbf{f}\nabla_{\mathbf{V}}\mathbf{o}_1,\\
	&\nabla_{\mathbf{Z}} \mathcal{L} = \nabla_{\mathbf{f}}
	\mathcal{L} \nabla_{\mathbf{o}_1}\mathbf{f}\nabla_{\mathbf{x}_a}\mathbf{o}_1\nabla_{\boldsymbol{\alpha}} \mathbf{x}_a \nabla_{\mathbf{Z}}\boldsymbol{\alpha}.
\end{align}

Each component can be computed as 
\begin{align}
	&\nabla_{\mathbf{f}}\mathcal{L} = (\textbf{logit}-\mathbf{e}_y)^\top,\\
	&\nabla_{\mathbf{w}_{i,j}}f_{i} = \frac{\lambda}{m} (\mathbf{V}\mathbf{x}_a)^\top,\\
	& \nabla_{\mathbf{o}_1}\mathbf{f} = \lambda\mathbf{W}_2\mathbf{W},\\
	&\nabla_{\mathbf{v}_l}\mathbf{o}_{1} = \begin{bmatrix} 
		\mathbf{0}\\
		\vdots\\
		\mathbf{x}_a^\top\\
		\vdots\\
		\mathbf{0}
			\end{bmatrix},\\
	& \nabla_{\mathbf{\mathbf{ZX}[-1]}}\mathbf{o}_1 = \frac{1}{\sqrt{d}}\mathbf{VX}\left(\textnormal{diag}(\boldsymbol{\alpha}(\mathbf{Z},\mathbf{X})) - \boldsymbol{\alpha}(\mathbf{Z},\mathbf{X})\boldsymbol{\alpha(\mathbf{Z},\mathbf{X})}^\top\right)\mathbf{X}^\top.
\end{align} 
Furthermore, we have
\begin{align}
	\nabla_{\mathbf{Z}}\mathcal{L}(\mathbf{Z},\mathbf{V},\mathbf{W},\mathbf{X},y) = \nabla _{\mathbf{Z}\mathbf{X}[-1]}\mathcal{L}(\mathbf{Z},\mathbf{V},\mathbf{W},\mathbf{X},y) \mathbf{X}[-1]^\top.
\end{align}
As a result, we have
\begin{align}
    \nabla_{\mathbf{Z}}\mathcal{L}(\mathbf{Z},\mathbf{V},\mathbf{W},\mathbf{X},y)\mathbf{X}[-1] = \nabla _{\mathbf{Z}\mathbf{X}[-1]}\mathcal{L}(\mathbf{Z},\mathbf{V},\mathbf{W},\mathbf{X},y).
\end{align}
Combining them finishes the proof.
\end{proof}
\section{Notations}
For simplicity and clarity, we define
\begin{align}
	&\textbf{logit}^{(t)}(\mathbf{X})=\textbf{logit}(\mathbf{Z}^{(t)},\mathbf{V}^{(t)},\mathbf{W}^{(t)},\mathbf{X}),\\
	&\boldsymbol{\tau}^{(t)} = \mathbf{W}_2\mathbf{W}^{(t)}\mathbf{V}^{(t)},\\
	&\Xi^{(t)}(\mathbf{X}) = \mathbf{X}\boldsymbol{\alpha}(\mathbf{Z}^{(t)},\mathbf{X}),\\
    &\mathbf{f}^{(t)}(\mathbf{X}) = \mathbf{f}(\mathbf{Z}^{(t)},\mathbf{V}^{(t)},\mathbf{W}^{(t)},\mathbf{X}).
\end{align}

\section{Useful lemmas and preliminaries}

\begin{lemma}\label{lemma: decomposition}
	For all $i\in[N]$, $l,h\in[m], q\in[2]$, the model parameters $\mathbf{w}$ and $\mathbf{V}$ is a linear combination of initialization vectors, i.e., there exists coefficients $\gamma^{(t)}_{i,k,l},\rho^{(t)}_{i,k,l},  \zeta_{q,k,l}^{(t)}, \beta^{(t)}_{1,k,l,q,h},\beta^{(t)}_{2,k,l,j},\beta^{(t)}_{3,k,l,j},\beta^{(t)}_{4,k,l,q}\in\mathbb{R}$ such that
	\begin{align}
		\mathbf{V}^{(t)}\mathbf{a}_i =& \sum_{k=1}^d\sum_{l=1}^{m}\gamma^{(t)}_{i,k,l}\mathbf{w}^{(0)}_{k,l} + \mathbf{V}^{(0)}\mathbf{a}_i,\\
		\mathbf{V}^{(t)}\mathbf{a}'_i =& \sum_{k=1}^d\sum_{l=1}^{m}\rho^{(t)}_{i,k,l}\mathbf{w}^{(0)}_{k,l} + \mathbf{V}^{(0)}\mathbf{a}'_i,\\
		\mathbf{V}^{(t)}\mathbf{r}_q =& \sum_{k=1}^d\sum_{l=1}^{m}
		\zeta_{q,k,l}^{(t)}\mathbf{w}^{(0)}_{k,l} +  \mathbf{V}^{(0)}\mathbf{r}_q,\\
		\mathbf{w}_{k,l}^{(t)} =& \sum_{q=1}^d\sum_{h=1}^{m}\beta^{(t)}_{1,k,l,q,h}\mathbf{w}^{(0)}_{q,h} + \sum_{j=1}^{N}\beta^{(t)}_{2,k,l,j}\mathbf{V}^{(0)}\mathbf{a}_j + \sum_{j=1}^{N}\beta^{(t)}_{3,k,l,j}\mathbf{V}^{(0)}\mathbf{a}'_j + \sum_{q\in[2]}\beta^{(t)}_{4,k,l,q}\mathbf{V}^{(0)}\mathbf{r}_q.
	\end{align}
\end{lemma}
This lemma holds because the gradient updates remain within the linear span of the initialized random vectors. 
While the softmax function introduces non-linearity, it acts solely on the scalar coefficients, effectively preserving the vector subspace.

\begin{lemma}\label{lemma: inner}
	For all $i,j\in[d]$, $l,q\in[m]$ and $(j,q)\neq (i,l)$, with probability at least $1-\delta$, we have
	\begin{align}
		&\frac{\sigma_0^2d}{2} \le \|\mathbf{w}^{(0)}_{i,l}\|^2_2 \le \frac{3\sigma_0^2d}{2},\\
		&|\langle \mathbf{w}^{(0)}_{i,l},\mathbf{w}^{(0)}_{j,q} \rangle| \le  2\sigma^2_0 \sqrt{d\log(4d^2m^2/\delta)}.
	\end{align}
\end{lemma}
\begin{proof}
	Suppose that $d = \Omega(\log(4md/\delta))$.
    By Bernstein's inequality, for any $i\in[d]$, $l\in[m]$, with probability at least $1-\delta/(2md)$, we have
	\begin{align}
		|\|\mathbf{w}^{(0)}_{i,l}\|_2^2 - \sigma_0^2d|= \mathcal{O}(\sigma_0^2\sqrt{d\log(4md/\delta)}).
	\end{align}
	With $d = \Omega(\log(4md/\delta))$, we have $\frac{\sigma_0^2d}{2} \le \|\mathbf{w}_i\|^2_2 \le \frac{3\sigma_0^2d}{2}$.
	Moreover, by Bernstein's inequality, for any $i,j\in[d]$, $l,q\in[m]$ and $(j,q)\neq (i,l)$, with probability at least $1- \delta/(4m^2d^2)$, we have
	\begin{align}
		|\langle \mathbf{w}^{(0)}_{i,l},\mathbf{w}^{(0)}_{j,q} \rangle| \le  2\sigma^2_0 \sqrt{d\log(2m^2d^2/\delta)}.
	\end{align}
Using union bound finishes the proof.
\end{proof}

\begin{lemma}
    Suppose that $d = \Omega(\log(4N/\delta))$. 
    In analogical reasoning tasks, for any $\mathbf{a},\mathbf{b}\in\{\mathbf{V}^{(0)}\mathbf{a}_i\}_{i=1}^N\cup\{\mathbf{V}^{(0)}\mathbf{a}'_i\}_{i=1}^N\cup\{\mathbf{V}^{(0)}\mathbf{r}_1,\mathbf{V}^{(0)}\mathbf{r}_2\}$ and $\mathbf{a}\neq\mathbf{b}$, with probability $1-\delta$, we have
    \begin{align}
        \frac{\sigma_0^2d}{2} \le \|\mathbf{a}\|^2_2 \le \frac{3\sigma_0^2d}{2},
    \end{align}
    and
    \begin{align}
        |\langle \mathbf{a},\mathbf{b}\rangle| \le  2\sigma^2_0 \sqrt{d\log(2(2N+2)^2/\delta)}.
    \end{align}
\end{lemma}

    \begin{proof}
        By Bernstein's inequality, for all $i\in[N]$ and $j\in[2]$, with probability at least $1-\delta/(2N+2)$, we have
	\begin{align}
		&|\|\mathbf{V}^{(0)}\mathbf{a}_{i}\|_2^2 - \sigma_0^2d|= \mathcal{O}(\sigma_0^2\sqrt{d\log(4N/\delta)})\\
        &|\|\mathbf{V}^{(0)}\mathbf{a}'_{i}\|_2^2 - \sigma_0^2d|= \mathcal{O}(\sigma_0^2\sqrt{d\log(4N/\delta)})\\
        &|\|\mathbf{V}^{(0)}\mathbf{r}_{j}\|_2^2 - \sigma_0^2d|= \mathcal{O}(\sigma_0^2\sqrt{d\log(4N/\delta)}).
	\end{align}
    With $d = \Omega(\log(4N/\delta))$, we have $\frac{\sigma_0^2d}{2} \le \|\mathbf{V}^{(0)}\mathbf{a}_i\|^2_2 \le \frac{3\sigma_0^2d}{2}$, $\frac{\sigma_0^2d}{2} \le \|\mathbf{V}^{(0)}\mathbf{a}'_i\|^2_2 \le \frac{3\sigma_0^2d}{2}$, and $\frac{\sigma_0^2d}{2} \le \|\mathbf{V}^{(0)}\mathbf{r}_j\|^2_2 \le \frac{3\sigma_0^2d}{2}$ for all $i\in[N]$ and $j\in[2]$.
    
    Moreover, by Bernstein's inequality, for any $\mathbf{a},\mathbf{b}\in\{\mathbf{V}^{(0)}\mathbf{a}_i\}_{i=1}^N\cup\{\mathbf{V}^{(0)}\mathbf{a}'_i\}_{i=1}^N\cup\{\mathbf{V}^{(0)}\mathbf{r}_1,\mathbf{V}^{(0)}\mathbf{r}_2\}$ and $\mathbf{a}\neq\mathbf{b}$, with probability at least $1- \delta/(2(2N+2)^2))$, we have
	\begin{align}
		|\langle \mathbf{a},\mathbf{b}\rangle| \le  2\sigma^2_0 \sqrt{d\log(2(2N+2)^2/\delta)}.
	\end{align}
Using union bound finishes the proof.
    \end{proof}

\begin{lemma}
    Suppose that $d = \Omega(\log(4N/\delta))$. 
    In two-hop reasoning tasks, for any $\mathbf{a},\mathbf{b}\in\{\mathbf{V}^{(0)}\mathbf{a}_i\}_{i=1}^N\cup\{\mathbf{V}^{(0)}\mathbf{b}_i\}_{i=1}^N\cup\{\mathbf{V}^{(0)}\mathbf{r}_1,\mathbf{V}^{(0)}\mathbf{r}_2,\mathbf{V}^{(0)}\mathbf{r}_3\}$ and $\mathbf{a}\neq\mathbf{b}$, with probability $1-\delta$, we have
    \begin{align}
        \frac{\sigma_0^2d}{2} \le \|\mathbf{a}\|^2_2 \le \frac{3\sigma_0^2d}{2},
    \end{align}
    and
    \begin{align}
        |\langle \mathbf{a},\mathbf{b}\rangle| \le  2\sigma^2_0 \sqrt{d\log(2(2N+3)^2/\delta)}.
    \end{align}
\end{lemma}

    \begin{proof}
        By Bernstein's inequality, for all $i\in[N]$ and $j\in[3]$, with probability at least $1-\delta/(2N+3)$, we have
	\begin{align}
		&|\|\mathbf{V}^{(0)}\mathbf{a}_{i}\|_2^2 - \sigma_0^2d|= \mathcal{O}(\sigma_0^2\sqrt{d\log(4N/\delta)})\\
        &|\|\mathbf{V}^{(0)}\mathbf{b}_{i}\|_2^2 - \sigma_0^2d|= \mathcal{O}(\sigma_0^2\sqrt{d\log(4N/\delta)})\\
        &|\|\mathbf{V}^{(0)}\mathbf{r}_{j}\|_2^2 - \sigma_0^2d|= \mathcal{O}(\sigma_0^2\sqrt{d\log(4N/\delta)}).
	\end{align}
    With $d = \Omega(\log(4N/\delta))$, we have $\frac{\sigma_0^2d}{2} \le \|\mathbf{V}^{(0)}\mathbf{a}_i\|^2_2 \le \frac{3\sigma_0^2d}{2}$, $\frac{\sigma_0^2d}{2} \le \|\mathbf{V}^{(0)}\mathbf{b}_i\|^2_2 \le \frac{3\sigma_0^2d}{2}$, and $\frac{\sigma_0^2d}{2} \le \|\mathbf{V}^{(0)}\mathbf{r}_j\|^2_2 \le \frac{3\sigma_0^2d}{2}$ for all $i\in[N]$ and $j\in[3]$.
    
    Moreover, by Bernstein's inequality, for any $\mathbf{a},\mathbf{b}\in\{\mathbf{V}^{(0)}\mathbf{a}_i\}_{i=1}^N\cup\{\mathbf{V}^{(0)}\mathbf{b}_i\}_{i=1}^N\cup\{\mathbf{V}^{(0)}\mathbf{r}_1,\mathbf{V}^{(0)}\mathbf{r}_2,\mathbf{V}^{(0)}\mathbf{r}_3\}$ and $\mathbf{a}\neq\mathbf{b}$, with probability at least $1- \delta/(2(2N+3)^2))$, we have
	\begin{align}
		|\langle \mathbf{a},\mathbf{b}\rangle| \le  2\sigma^2_0 \sqrt{d\log(2(2N+3)^2/\delta)}.
	\end{align}
Using union bound finishes the proof.
    \end{proof}
    
\section{Proofs of joint training on \texorpdfstring{$\biguplus_{k=1}^\kappa(\mathcal{S}_1\cup\mathcal{S}_2)\biguplus\mathcal{S}_3$}{S\_1 U S\_2 U S\_3}}\label{app:proof_thm1}
We perform a two-stage proof for the joint training. 
Note that the total number of data samples is $n = N(2\kappa+1)$.
We assume that Lemmas 3 and 4 hold with probability $1-2\delta'$ by union bound.
\subsection{Properties during Stage-1 training (joint training on \texorpdfstring{$\biguplus_{k=1}^\kappa(\mathcal{S}_1\cup\mathcal{S}_2)\biguplus\mathcal{S}_3$}{S\_1 U S\_2 U S\_3})}
 
By the choice of \(T_1\) and the small-initialization and learning-rate
conditions in Condition~1, the cumulative parameter movement during
Stage~1 is small enough to keep all logits uniformly bounded, i.e., for all $(\mathbf{X},y)\in\biguplus_{k=1}^\kappa(\mathcal{S}_1\cup\mathcal{S}_2)\biguplus\mathcal{S}_3$,
\[
\max_{t\le T_1}\|\mathbf{f}^{(t)}(\mathbf{X})\|_\infty=\tilde{O}(1).
\]
Consequently, there exists an absolute constant \(c_0>0\) such that
\[
1-\operatorname{logit}^{(t)}_y(\mathbf{X})\ge c_0
\]
for all \(t\le T_1\).
Equivalently, we have that $1-\text{logit}_y(\mathbf{X}) = \Theta(1)$ holds for all $(\mathbf{X},y)\in\biguplus_{k=1}^\kappa(\mathcal{S}_1\cup\mathcal{S}_2)\biguplus\mathcal{S}_3$.

Then, in Stage 1, the following properties hold.
\begin{lemma}\label{lemma: j1}
	For any $t\in[0, T_1]$ with $T_1 = \Theta(\oldfrac{mn\log(d)}{\lambda^2\eta\kappa\sigma_0\sqrt{d}})$, we have
	\begin{enumerate}
		\item The attention scores are balanced, i.e.,
		\begin{align}
			\exp\left(-\mathcal{O}\left(\frac{mn\log^2(d)}{\lambda^2d^{3/2}\sigma_0}\right)\right)\le \frac{\alpha_1(\mathbf{Z}^{(t)},\mathbf{X})}{\alpha_2(\mathbf{Z}^{(t)},\mathbf{X})} \le \exp\left(\mathcal{O}\left(\frac{mn\log^2(d)}{\lambda^2d^{3/2}\sigma_0}\right)\right).
		\end{align}
		\item For all $i\in[N]$, $l_1, l_2,l_3 \in [m]$, the coefficients satisfy
		\begin{align}
			\gamma^{(T_1)}_{i,\mathcal{I}(\mathbf{b}_i),l_1} = \Theta\left(\frac{\log(d)}{\lambda\sigma_0\sqrt{d}}\right), 	\rho^{(T_1)}_{i,\mathcal{I}(\mathbf{c}_i),l_2} = \Theta\left(\frac{\log(d)}{\lambda\kappa\sigma_0\sqrt{d}}\right),
			\rho^{(T_1)}_{i,\mathcal{I}(\mathbf{b}_i),l_3} = \Theta\left(\frac{\log(d)}{\lambda\sigma_0\sqrt{d}}\right).
		\end{align}
		\item For all $i\in[N], j_1\in[d]\backslash\{\mathcal{I}(\mathbf{b}_i),\mathcal{I}(\mathbf{c}_i)\}_{i=1}^N,j_2\in[d]\backslash\{\mathbf{b}_i\}_{i=1}^N$ and $l\in [m]$, the coefficients satisfy
		\begin{align}
			-\mathcal{O}\left(\frac{\log(d)}{\lambda \sigma_0d}\right)\le \gamma^{(T_1)}_{i,j_1,l} \le 0, -\mathcal{O}\left(\frac{\log(d)}{\lambda \sigma_0d}\right)\le \rho^{(T_1)}_{i,j_2,l} \le 0.
		\end{align}
		\item For all $k\in[N]$, $l_1\in\mathcal{B}_{1,i}\cap\mathcal{B}_{2,i}$, and $l_2\in\mathcal{B}_{3,i}$ the coefficients satisfy
		\begin{align}
			\zeta_{1,\mathcal{I}(\mathbf{b}_k),l_1}^{(T_1)} = \Theta\left(\frac{\log(d)}{\lambda\sigma_0\sqrt{d}}\right), \zeta_{2,\mathcal{I}(\mathbf{c}_k),l_2}^{(T_1)} = \Theta\left(\frac{\log(d)}{\lambda\kappa\sigma_0\sqrt{d}}\right).
		\end{align}
		\item For all $k_1\in[d]\backslash\{\mathcal{I}(\mathbf{b}_i)\}_{i=1}^N$, $k_2 \in [d]\backslash\{\mathcal{I}(\mathbf{c}_i)\}_{i=1}^N$ and $l\in[m]$, the coefficients satisfy
		\begin{align}
			-\mathcal{O}\left(\frac{\log(d)}{\lambda\sigma_0d}\right)\le\zeta_{1,k_1,l}^{(T_1)} \le 0, -\mathcal{O}\left(\frac{\log(d)}{\lambda\sigma_0d}\right)\le\zeta_{2,k_2,l}^{(T_1)} \le 0.
		\end{align}
		\item For all $k\in[d]$, $l\in[m]$, the feature-layer weights remain unchanged, i.e.,for $l\in[m],k\in[d]$,
		\begin{align}
			\mathbf{w}^{(T_1)}_{k,l} = \mathbf{w}^{(0)}_{k,l}.
		\end{align}
		
	\end{enumerate}
\end{lemma}
\begin{proof}
	We first prove the first statement.
The update of matrix $\mathbf{Z}$ satisfies
\begin{equation}
	\begin{aligned}
		|\tilde{Z}^{(t+1)}_{\mathcal{I}(\mathbf{a}_i),\mathcal{I}(\mathbf{r}_1)} -\tilde{Z}^{(t)}_{\mathcal{I}(\mathbf{a}_i),\mathcal{I}(\mathbf{r}_1)}|
		=& \frac{\eta\lambda\kappa}{n\sqrt{d}}\alpha_1^{(t)}([\mathbf{a}_i,\mathbf{r}_1])\left|(\mathbf{e}_{\mathcal{I}(\mathbf{b}_i)}-\textbf{logit}^{(t)}([\mathbf{a}_i\;\;\mathbf{r}_1]))^\top\boldsymbol{\tau}^{(t)}\left(\mathbf{a}_i - \Xi^{(t)}([\mathbf{a}_i\;\; \mathbf{r}_1])\right)\right|\\
		=&\mathcal{O}\left(\frac{\eta\kappa\log(d)}{n\sqrt{d}}\right),
	\end{aligned}
\end{equation}
and 
\begin{equation}
	\begin{aligned}
		|\tilde{Z}^{(t+1)}_{\mathcal{I}(\mathbf{r}_1),\mathcal{I}(\mathbf{r}_1)} -\tilde{Z}^{(t)}_{\mathcal{I}(\mathbf{r}_1),\mathcal{I}(\mathbf{r}_1)}|
		=& \sum_{i=1}^N\frac{\eta\lambda\kappa}{n\sqrt{d}}\alpha_2^{(t)}([\mathbf{a}_i,\mathbf{r}_1])\left|(\mathbf{e}_{\mathcal{I}(\mathbf{b})}-\textbf{logit}^{(t)}([\mathbf{a}_i\;\;\mathbf{r}_1]))^\top\boldsymbol{\tau}^{(t)}\left(\mathbf{r}_1-\Xi^{(t)}([\mathbf{a}_i\;\;\mathbf{r}_1])\right)\right|\\
		=&\mathcal{O}\left(\frac{\eta\kappa\log(d)}{\sqrt{d}}\right).
	\end{aligned}
\end{equation}
Therefore, we have
\begin{align}
	|\tilde{Z}^{(T_1)}_{\mathcal{I}(\mathbf{a}_i),\mathcal{I}(\mathbf{r}_1)} -\tilde{Z}^{(0)}_{\mathcal{I}(\mathbf{a}_i),\mathcal{I}(\mathbf{r}_1)}| = \mathcal{O}\left(\frac{\eta\kappa\log(d)}{n\sqrt{d}}T_1\right) = \mathcal{O}\left(\frac{\log^2(d)}{\lambda^2d\sigma_0}\right),
\end{align}
and
\begin{align}
	|\tilde{Z}^{(T_1)}_{\mathcal{I}(\mathbf{r}_1),\mathcal{I}(\mathbf{r}_1)} -\tilde{Z}^{(0)}_{\mathcal{I}(\mathbf{r}_1),\mathcal{I}(\mathbf{r}_1)}| = \mathcal{O}\left(\frac{\eta\kappa\log(d)}{\sqrt{d}}T_1\right) = \mathcal{O}\left(\frac{mn\log^2(d)}{\lambda^2d\sigma_0}\right).
\end{align}
By initialization, we have
\begin{align}
	|\tilde{Z}^{(0)}_{\mathcal{I}(\mathbf{a}_i),\mathcal{I}(\mathbf{r}_1)} -\tilde{Z}^{(0)}_{\mathcal{I}(\mathbf{r}_1),\mathcal{I}(\mathbf{r}_1)}| = \mathcal{O}(\sigma_0).
\end{align}
We have
\begin{equation}
\begin{aligned}
	|\tilde{Z}^{(T_1)}_{\mathcal{I}(\mathbf{a}_i),\mathcal{I}(\mathbf{r}_1)} -\tilde{Z}^{(T_1)}_{\mathcal{I}(\mathbf{r}_1),\mathcal{I}(\mathbf{r}_1)}| \le& |\tilde{Z}^{(T_1)}_{\mathcal{I}(\mathbf{a}_i),\mathcal{I}(\mathbf{r}_1)} -\tilde{Z}^{(0)}_{\mathcal{I}(\mathbf{a}_i),\mathcal{I}(\mathbf{r}_1)}| + |\tilde{Z}^{(0)}_{\mathcal{I}(\mathbf{a}_i),\mathcal{I}(\mathbf{r}_1)} -\tilde{Z}^{(0)}_{\mathcal{I}(\mathbf{r}_1),\mathcal{I}(\mathbf{r}_1)}| + 	|\tilde{Z}^{(T_1)}_{\mathcal{I}(\mathbf{r}_1),\mathcal{I}(\mathbf{r}_1)} -\tilde{Z}^{(0)}_{\mathcal{I}(\mathbf{r}_1),\mathcal{I}(\mathbf{r}_1)}|\\
	=& \mathcal{O}\left(\frac{mn\log^2(d)}{\lambda^2d\sigma_0}\right).
\end{aligned}
\end{equation}
Therefore, we have
\begin{align}
	\frac{\alpha_1(\Xi^{(t)}([\mathbf{a}_i\;\;\mathbf{r}_1]))}{\alpha_2(\Xi^{(t)}([\mathbf{a}_i\;\;\mathbf{r}_1]))} \le \exp\left(\mathcal{O}\left(\frac{mn\log^2(d)}{\lambda^2d^{3/2}\sigma_0}\right)\right).
\end{align}
Similarly, we have
\begin{align}
	\frac{\alpha_1(\Xi^{(t)}([\mathbf{a}_i\;\;\mathbf{r}_1]))}{\alpha_2(\Xi^{(t)}([\mathbf{a}_i\;\;\mathbf{r}_1]))} \ge \exp\left(-\mathcal{O}\left(\frac{mn\log^2(d)}{\lambda^2d^{3/2}\sigma_0}\right)\right).
\end{align}
Similarly, we can prove that
\begin{align}
    &|\tilde{Z}^{(T_1)}_{\mathcal{I}(\mathbf{a}'_i),\mathcal{I}(\mathbf{r}_1)} -\tilde{Z}^{(T_1)}_{\mathcal{I}(\mathbf{r}_1),\mathcal{I}(\mathbf{r}_1)}| = \mathcal{O}\left(\frac{mn\log^2(d)}{\lambda^2d\sigma_0}\right),\\
    &|\tilde{Z}^{(T_1)}_{\mathcal{I}(\mathbf{a}'_i),\mathcal{I}(\mathbf{r}_2)} -\tilde{Z}^{(T_1)}_{\mathcal{I}(\mathbf{r}_2),\mathcal{I}(\mathbf{r}_2)}| = \mathcal{O}\left(\frac{mn\log^2(d)}{\lambda^2d\sigma_0}\right),
\end{align}
so that 
\begin{align}
    \frac{\alpha_1(\Xi^{(t)}([\mathbf{a}'_i\;\;\mathbf{r}_1]))}{\alpha_2(\Xi^{(t)}([\mathbf{a}'_i\;\;\mathbf{r}_1]))} \le \exp\left(\mathcal{O}\left(\frac{mn\log^2(d)}{\lambda^2d^{3/2}\sigma_0}\right)\right), \frac{\alpha_1(\Xi^{(t)}([\mathbf{a}'_i\;\;\mathbf{r}_1]))}{\alpha_2(\Xi^{(t)}([\mathbf{a}'_i\;\;\mathbf{r}_1]))} \ge \exp\left(-\mathcal{O}\left(\frac{mn\log^2(d)}{\lambda^2d^{3/2}\sigma_0}\right)\right),
\end{align}
\begin{align}
    \frac{\alpha_1(\Xi^{(t)}([\mathbf{a}'_i\;\;\mathbf{r}_2]))}{\alpha_2(\Xi^{(t)}([\mathbf{a}'_i\;\;\mathbf{r}_2]))} \le \exp\left(\mathcal{O}\left(\frac{mn\log^2(d)}{\lambda^2d^{3/2}\sigma_0}\right)\right), \frac{\alpha_1(\Xi^{(t)}([\mathbf{a}'_i\;\;\mathbf{r}_2]))}{\alpha_2(\Xi^{(t)}([\mathbf{a}'_i\;\;\mathbf{r}_2]))} \ge \exp\left(-\mathcal{O}\left(\frac{mn\log^2(d)}{\lambda^2d^{3/2}\sigma_0}\right)\right).
\end{align}

	Next, we prove the second, third and fourth statements.
	Based on the gradient form,	the updates of $\mathbf{V}^{(t)}\mathbf{a}_i,\mathbf{V}^{(t)}\mathbf{a}'_i, \mathbf{V}^{(t)}\mathbf{r}_1$ for all $t\in [0,T_1]$ satisfy
\begin{equation}
	\begin{aligned}
		\mathbf{V}^{(t+1)}\mathbf{a}_i =& \mathbf{V}^{(t)}\mathbf{a}_i + \frac{\lambda\eta\kappa}{nm}\sum_{k=1}^d\sum_{l=1}^{m}(\mathbb{I}(k=\mathcal{I}(\mathbf{b}_i))-\text{logit}^{(t)}_{k}(\Xi^{(t)}([\mathbf{a}_i\;\;\mathbf{r}_1])))\mathbf{w}^{(0)}_{k,l}\\
		=& \mathbf{V}^{(t)}\mathbf{a}_i + \frac{\lambda\eta\kappa}{mn}\sum_{l=1}^m(1-\text{logit}^{(t)}_{\mathcal{I}(\mathbf{b}_i)}(\Xi^{(t)}([\mathbf{a}_i\;\;\mathbf{r}_1]))) \mathbf{w}^{(0)}_{\mathcal{I}(\mathbf{b}_i),l}\\
		&-\frac{\lambda\eta\kappa}{mn}\sum_{l=1}^m\sum_{k\neq\mathcal{I}(\mathbf{b}_i)}\text{logit}_{k}^{(t)}(\Xi^{(t)}([\mathbf{a}_i\;\;\mathbf{r}_1]))\mathbf{w}_{k,l}^{(0)},
	\end{aligned}
\end{equation}
\begin{equation}
	\begin{aligned}
		\mathbf{V}^{(t+1)}\mathbf{a}'_i =& \mathbf{V}^{(t)}\mathbf{a}'_i + \frac{\lambda\eta}{nm}\sum_{k=1}^d\sum_{l=1}^{m}(\mathbb{I}(k=\mathcal{I}(\mathbf{c}_i))-\text{logit}^{(t)}_{k}(\Xi^{(t)}([\mathbf{a}'_i\;\;\mathbf{r}_2])))\mathbf{w}^{(0)}_{k,l}\\&+\frac{\lambda\eta\kappa}{nm}\sum_{k=1}^d\sum_{l=1}^{m}(\mathbb{I}(k=\mathcal{I}(\mathbf{b}_i))-\text{logit}^{(t)}_{k}(\Xi^{(t)}([\mathbf{a}'_i\;\;\mathbf{r}_1])))\mathbf{w}^{(0)}_{k,l},
	\end{aligned}
\end{equation}
\begin{equation}
	\begin{aligned}
		\mathbf{V}^{(t+1)}\mathbf{r}_1
		= \mathbf{V}^{(t)}\mathbf{r}_1 &+ \frac{\lambda\eta\kappa}{nm}\sum_{i=1}^N\sum_{k=1}^d\sum_{l=1}^{m}(\mathbb{I}(k=\mathcal{I}(\mathbf{b}_i))-\text{logit}^{(t)}_{k}(\Xi^{(t)}([\mathbf{a}_i\;\;\mathbf{r}_1])))\mathbf{w}^{(0)}_{k,l}\\
        &+ \frac{\lambda\eta\kappa}{nm}\sum_{i=1}^N\sum_{k=1}^d\sum_{l=1}^{m}(\mathbb{I}(k=\mathcal{I}(\mathbf{b}_i))-\text{logit}^{(t)}_{k}(\Xi^{(t)}([\mathbf{a}'_i\;\;\mathbf{r}_1])))\mathbf{w}^{(0)}_{k,l},
	\end{aligned}
\end{equation}
and
\begin{equation}
	\begin{aligned}
		&\mathbf{V}^{(t+1)}\mathbf{r}_2
		= \mathbf{V}^{(t)}\mathbf{r}_2 + \frac{\lambda\eta}{nm}\sum_{i=1}^N\sum_{k=1}^d\sum_{l=1}^{m}(\mathbb{I}(k=\mathcal{I}(\mathbf{c}_i))-\text{logit}^{(t)}_{k}(\Xi^{(t)}([\mathbf{a}'_i\;\;\mathbf{r}_2])))\mathbf{w}^{(0)}_{k,l}.
	\end{aligned}
\end{equation}
Then, for all $t\in[0,T_1-1]$, we can simply conclude that
\begin{itemize}
	\item for all $i\in[N]$ and $l\in [m]$,
    \begin{equation}
	\begin{aligned}
		\gamma^{(t+1)}_{i,\mathcal{I}(\mathbf{b}_i),l} =& \gamma^{(t)}_{i,\mathcal{I}(\mathbf{b}_i),l} +  \frac{\lambda\eta \kappa}{mn}(1-\text{logit}^{(t)}_{\mathcal{I}(\mathbf{b}_i)}(\Xi^{(t)}([\mathbf{a}_i\;\;\mathbf{r}_1])))\\
        =& \gamma^{(t)}_{i,\mathcal{I}(\mathbf{b}_i),l} + \Theta(\frac{\lambda\eta\kappa}{mn}),\\
 	\end{aligned}
    \end{equation}
    \begin{equation}
	\begin{aligned}   
    &\rho^{(t+1)}_{i,\mathcal{I}(\mathbf{b}_i),l} = \rho^{(t)}_{i,\mathcal{I}(\mathbf{b}_i),l} +  \frac{\lambda\eta\kappa}{mn}(1-\text{logit}^{(t)}_{\mathcal{I}(\mathbf{b}_i)}(\Xi^{(t)}([\mathbf{a}'_i\;\;\mathbf{r}_1])))   =\rho^{(t)}_{i,\mathcal{I}(\mathbf{b}_i),l} + \Theta(\frac{\lambda\eta\kappa}{mn}).
	\end{aligned}
    \end{equation} 
    \begin{equation}
	\begin{aligned}   
    &\rho^{(t+1)}_{i,\mathcal{I}(\mathbf{c}_i),l} = \rho^{(t)}_{i,\mathcal{I}(\mathbf{c}_i),l} +  \frac{\lambda\eta\kappa}{mn}(1-\text{logit}^{(t)}_{\mathcal{I}(\mathbf{c}_i)}(\Xi^{(t)}([\mathbf{a}'_i\;\;\mathbf{r}_2])))   =\rho^{(t)}_{i,\mathcal{I}(\mathbf{c}_i),l} + \Theta(\frac{\lambda\eta\kappa}{mn}).
	\end{aligned}
    \end{equation}
	\item for all $i,j\in[N]$ and $j\neq i$,
	\begin{align}
		&\gamma^{(t+1)}_{i,\mathcal{I}(\mathbf{b}_j),l} = \gamma^{(t)}_{i,,\mathcal{I}(\mathbf{b}_j),l} -\frac{\lambda\eta\kappa}{mn}\text{logit}_{\mathcal{I}(\mathbf{b}_i)}^{(t)}(\Xi^{(t)}([\mathbf{a}_j\;\;\mathbf{r}_1])),\\
        &\rho^{(t+1)}_{i,\mathcal{I}(\mathbf{b}_j),l} = \rho^{(t)}_{i,\mathcal{I}(\mathbf{b}_j),l} -\frac{\lambda\eta\kappa}{mn}\text{logit}_{\mathcal{I}(\mathbf{b}_i)}^{(t)}(\Xi^{(t)}([\mathbf{a}'_j\;\;\mathbf{r}_2])).
	\end{align} 
	\item for all $i\in[N]$ and $l\in[m]$,
	\begin{align}
		\zeta_{1,\mathcal{I}(\mathbf{b}_i),l}^{(t+1)} = \zeta_{1,\mathcal{I}(\mathbf{b}_i),l}^{(t)} +& \frac{\lambda\eta\kappa}{mn}(1-\text{logit}^{(t)}_{\mathcal{I}(\mathbf{b}_i)}(\Xi^{(t)}([\mathbf{a}_i\;\;\mathbf{r}_1]))) - \frac{\lambda\eta\kappa}{mn}\sum_{k\neq \mathcal{I}(\mathbf{b}_i)} \text{logit}^{(t)}_{\mathcal{I}(\mathbf{b}_i)}(\Xi^{(t)}([\mathbf{a}_k\;\;\mathbf{r}_1]))\\
        +& \frac{\lambda\eta\kappa}{mn}(1-\text{logit}^{(t)}_{\mathcal{I}(\mathbf{b}_i)}(\Xi^{(t)}([\mathbf{a}'_i\;\;\mathbf{r}_1]))) - \frac{\lambda\eta\kappa}{mn}\sum_{k\neq \mathcal{I}(\mathbf{b}_i)} \text{logit}^{(t)}_{\mathcal{I}(\mathbf{b}_i)}(\Xi^{(t)}([\mathbf{a}'_k\;\;\mathbf{r}_1])).
	\end{align}
	\begin{align}
		&\zeta_{1,\mathcal{I}(\mathbf{c}_i),l}^{(t+1)} = \zeta_{1,\mathcal{I}(\mathbf{c}_i),l}^{(t)} + \frac{\lambda\eta\kappa}{mn}(1-\text{logit}^{(t)}_{\mathcal{I}(\mathbf{c}_i)}(\Xi^{(t)}([\mathbf{a}'_i\;\;\mathbf{r}_2]))) - \frac{\lambda\eta\kappa}{mn}\sum_{k\neq \mathcal{I}(\mathbf{c}_i)} \text{logit}^{(t)}_{\mathcal{I}(\mathbf{c}_i)}(\Xi^{(t)}([\mathbf{a}'_k\;\;\mathbf{r}_2])).
	\end{align}
\end{itemize}
As $\text{logit}_{k}(\Xi^{(t)}([\mathbf{a}_i\;\;\mathbf{r}_1])) <\frac{2}{\sqrt{d}}$ for all $k\in[d]\backslash\{\mathcal{I}(\mathbf{b}_i)\}$ due to the fact that $\text{logit}_{k}(\Xi^{(t)}([\mathbf{a}_i\;\;\mathbf{r}_1]))\le \text{logit}_{\mathcal{I}(\mathbf{b}_i)}(\Xi^{(t)}([\mathbf{a}_i\;\;\mathbf{r}_1]))/2$, we have
\begin{itemize}
	\item for all $i\in[N]$ and $l\in[m]$,
	\begin{align}
		&\gamma^{(T_1)}_{i,\mathcal{I}(\mathbf{b}_i),l} =  \Theta\left(\frac{\lambda\eta\kappa}{mn}T_1\right) = \Theta\left(\frac{\log(d)}{\lambda\sqrt{d}\sigma_0}\right),\\
        &\rho^{(T_1)}_{i,\mathcal{I}(\mathbf{b}_i),l} =  \Theta\left(\frac{\lambda\eta}{mn}T_1\right) = \Theta\left(\frac{\log(d)}{\kappa\lambda\sqrt{d}\sigma_0}\right),\\
        &\rho^{(T_1)}_{i,\mathcal{I}(\mathbf{c}_i),l} =  \Theta\left(\frac{\lambda\eta}{mn}T_1\right) = \Theta\left(\frac{\log(d)}{\kappa\lambda\sqrt{d}\sigma_0}\right),
	\end{align}
	\item for all $i\in[N]$ and $j_1\notin [d]\backslash\{\mathcal{I}(\mathbf{b}_i)\}_{i=1}^N, j_2 \notin [d]\backslash\{\mathcal{I}(\mathbf{c}_i)\}_{i=1}^N$,
	\begin{align}
		&-\mathcal{O}\left(\frac{\lambda\eta}{mn\sqrt{d}}T_1\right) = -\mathcal{O}\left(\frac{\log(d)}{\lambda d\sigma_0}\right) \le\gamma^{(T_1)}_{i,j_1,l} \le \gamma^{(0)}_{i,j_1,l} = 0,\\
        &-\mathcal{O}\left(\frac{\lambda\eta}{mn\sqrt{d}}T_1\right) = -\mathcal{O}\left(\frac{\log(d)}{\lambda d\sigma_0}\right) \le\rho^{(T_1)}_{i,j_2,l} \le \rho^{(0)}_{i,j_2,l} = 0,
	\end{align} 
	\item for all $i\in[N]$ and $l\in[m]$,
	\begin{align}
		&\zeta_{1,\mathcal{I}(\mathbf{b}_i),l}^{(T_1)} = \Theta\left(\frac{\eta\lambda\kappa}{mn}T_1\right) = \Theta\left(\frac{\log(d)}{\lambda\sqrt{d}\sigma_0}\right),\\
        &\zeta_{2,\mathcal{I}(\mathbf{c}_i),l}^{(T_1)} = \Theta\left(\frac{\eta\lambda}{mn}T_1\right) = \Theta\left(\frac{\log(d)}{\lambda\sqrt{d}\sigma_0\kappa}\right).
	\end{align}
    \item for all $j_1\notin [d]\backslash\{\mathcal{I}(\mathbf{b}_i)\}_{i=1}^N$ and $j_2 \notin [d]\backslash\{\mathcal{I}(\mathbf{c}_i)\}_{i=1}^N$, 
    \begin{align}
        &-\mathcal{O}\left(\frac{\lambda\eta}{mn\sqrt{d}}T_1\right) = -\mathcal{O}\left(\frac{\log(d)}{\lambda d\sigma_0}\right) \le\zeta_{1,j_1,l}^{(T_1)} \le \zeta^{(0)}_{1,j_1,l} = 0,\\
        &-\mathcal{O}\left(\frac{\lambda\eta}{mn\sqrt{d}}T_1\right) = -\mathcal{O}\left(\frac{\log(d)}{\lambda d\sigma_0}\right) \le\zeta_{2,j_2,l}^{(T_1)}\le \zeta^{(0)}_{2,j_2,l} = 0.
    \end{align}
\end{itemize}

The sixth statement holds as the feature layer is not updated.

This completes the proof.
\end{proof}

\subsection{Convergence of Stage-2 training (joint training on \texorpdfstring{$\biguplus_{k=1}^\kappa(\mathcal{S}_1\cup\mathcal{S}_2)\biguplus\mathcal{S}_3$}{S\_1 U S\_2 U S\_3})}

\begin{lemma}
	There exists an iteration $t \in(T_1, T_1+T_2]$, where $T_2 = \Theta(\oldfrac{mN^2}{\lambda^2d\sigma_0^2\eta })$ such that
	\begin{align}
		\mathcal{L}_{\biguplus_{k=1}^\kappa(\mathcal{S}_1 \cup \mathcal{S}_2) \biguplus \mathcal{S}_3}(\mathbf{Z}^{(t)},\mathbf{V}^{(t)},\mathbf{W}^{(t)}) \le \frac{0.01}{n}.
	\end{align}
\end{lemma}
\begin{proof}
	Let 
	\begin{equation}
		\mathbf{w}^*_{k,l} = \mathbf{w}^{(0)}_{k,l} + 
	\begin{cases}
 \oldfrac{C}{\lambda}\cdot\log(1/\epsilon)(\oldfrac{\mathbf{V}^{(0)}\mathbf{a}_i}{\|\mathbf{V}^{(0)}\mathbf{a}_i\|_2^2} + \oldfrac{\mathbf{V}^{(0)}\mathbf{a}'_i}{\|\mathbf{V}^{(0)}\mathbf{a}'_i\|_2^2}+ \oldfrac{\mathbf{V}^{(0)}\mathbf{r}_1}{\|\mathbf{V}^{(0)}\mathbf{r}_1\|_2^2}) & \text{if } k=\mathcal{I}(\mathbf{b}_i),\\
   \oldfrac{C}{\lambda}\cdot\log(1/\epsilon)(\oldfrac{\mathbf{V}^{(0)} \mathbf{a}'_i}{\|\mathbf{V}^{(0)} \mathbf{a}'_i\|_2^2}+\oldfrac{\mathbf{V}^{(0)} \mathbf{r}_2}{\|\mathbf{V}^{(0)} \mathbf{r}_2\|_2^2}) & \text{if } k=\mathcal{I}(\mathbf{c}_i),\\
  0& k\notin \{\mathcal{I}(\mathbf{b}_i), \mathcal{I}(\mathbf{c}_i)\}_{i=1}^N.
	\end{cases}
	\end{equation}
	 for all $k\in[d], i\in[N]$ and $l\in[m]$.
     Although \(\mathbf{W}^*\) is written in terms of the initialization vectors, this is sufficient for the Stage-2 argument since \(\mathbf{W}^*\) separates the frozen Stage-1 features with large margins. 

	 For simplicity, we denote $\mathcal{S} = \biguplus_{k=1}^\kappa(\mathcal{S}_1\cup\mathcal{S}_2)\biguplus\mathcal{S}_3$ and $n = |\mathcal{S}|$.
	Then, we have
	\begin{equation}\label{equ: increment1}
		\begin{aligned}
			&\left\|\mathbf{W}^{(t)} - \mathbf{W}^*\right\|_F^2 - \left\|\mathbf{W}^{(t+1)} - \mathbf{W}^*\right\|_F^2\\
			=&  2\eta\langle\nabla \mathcal{L}_{\mathcal{S}}(\mathbf{Z}^{(t)},\mathbf{V}^{(t)},\mathbf{W}^{(t)}), \mathbf{W}^{(t)} - \mathbf{W}^*\rangle - \eta^2 \left\|\nabla \mathcal{L}_{\mathcal{S}}(\mathbf{Z}^{(t)},\mathbf{V}^{(t)},\mathbf{W}^{(t)})\right\|_F^2\\
			=& \frac{2\eta}{n} \sum_{(\mathbf{X},y)\in\mathcal{S}}\sum_{k=1}^d \frac{\partial \mathcal{L}}{\partial f_k} \langle\nabla f_k(\mathbf{Z}^{(t)},\mathbf{V}^{(t)},\mathbf{W}^{(t)},\mathbf{X}),\mathbf{W}^{(t)}\rangle -
			\frac{2\eta}{n} \sum_{(\mathbf{X},y)\in\mathcal{S}}\sum_{k=1}^d \frac{\partial \mathcal{L}}{\partial f_k} \underbrace{\langle \nabla f_k(\mathbf{Z}^{(t)},\mathbf{V}^{(t)},\mathbf{W}^{(t)},\mathbf{X}), \mathbf{W}^*\rangle}_{A}\\
			&- \underbrace{\eta^2 \left\|\nabla \mathcal{L}_\mathcal{S}(\mathbf{Z}^{(t)},\mathbf{V}^{(t)},\mathbf{W}^{(t)})\right\|_F^2}_{B}.
		\end{aligned}
	\end{equation}
	Next, we bound the terms $A$ and $B$.
	For the term $A$, we have
	\begin{align}\label{equ: A1}
		A 
		\begin{cases}\ge \frac{C\log(1/\epsilon)}{\lambda} - \mathcal{O}(\oldfrac{\sqrt{N}\log(d)\log(1/\epsilon)}{\lambda^2\sqrt{md}\sqrt{d\sigma_0^2}})& \text{if } k = y, \\
			\le \mathcal{O}(\oldfrac{\sqrt{N}\log(d)\log(1/\epsilon)}{\lambda^2\sqrt{md}\sqrt{d\sigma_0^2}}) & \text{if } k \neq y.
		\end{cases}
	\end{align}
	For the term $B$, we have
	\begin{equation}\label{equ: B1}
	\begin{aligned}
		B \le& \eta^2 \left[ \frac{1}{n} \sum_{(\mathbf{X},y)\in\mathcal{S}}\sum_{k\in[d]}\left|\frac{\partial \mathcal{L}(\mathbf{Z}^{(t)},\mathbf{V}^{(t)},\mathbf{W}^{(t)},\mathbf{X},y)}{\partial f_k}\right| \left\| \nabla f_k(\mathbf{Z}^{(t)},\mathbf{V}^{(t)},\mathbf{W}^{(t)},\mathbf{X})\right\|_F\right]^2\\
		\le& 2\eta^2\max_{(\mathbf{X},y)\in\mathcal{S}}\|\mathbf{V}^{(t)}\Xi^{(t)}(\mathbf{X})\|_2^2\cdot\mathcal{L}_\mathcal{S}(\mathbf{Z},\mathbf{V},\mathbf{W}).
	\end{aligned}
	\end{equation}
	Next, we need to bound $\max_{(\mathbf{X},y)\in\mathcal{S}}\|\mathbf{V}^{(t)}\Xi^{(t)}(\mathbf{X})\|_2^2$. 
	By Lemma \ref{lemma: j1}, We have
	\begin{align}
		\max_{(\mathbf{X},y)\in\mathcal{S}}\|\mathbf{V}^{(t)}\Xi^{(t)}(\mathbf{X})\|_2^2 = \max_{(\mathbf{X},y)\in\mathcal{S}}\|\mathbf{V}^{(T_1)}\Xi^{(T_1)}(\mathbf{X})\|_2^2 = \mathcal{O}(\oldfrac{Nm\log^2(d)}{\lambda^2}).
	\end{align}
	
	Combining (\ref{equ: A1}), (\ref{equ: B1}), and (\ref{equ: increment1}), we have
	\begin{equation}\label{equ: descent1}
		\begin{aligned}
			&\left\|\mathbf{W}^{(t)} - \mathbf{W}^*\right\|_F^2 - \left\|\mathbf{W}^{(t+1)} - \mathbf{W}^*\right\|_F^2\\
			=&  \frac{2\eta}{n} \sum_{(\mathbf{X},y)\in\mathcal{S}}\left(\sum_{k=1}^d \frac{\partial \mathcal{L}}{\partial f_k} \langle\nabla f_k(\mathbf{Z}^{(t)},\mathbf{V}^{(t)},\mathbf{W}^{(t)},\mathbf{X}),\mathbf{W}^{(t)}\rangle - \sum_{k=1}^d \frac{\partial \mathcal{L}}{\partial f_k} \langle \nabla f_k(\mathbf{Z}^{(t)},\mathbf{V}^{(t)},\mathbf{W}^{(t)},\mathbf{X}), \mathbf{W}^*\rangle\right)\\
			&- \eta^2 \left\|\nabla \mathcal{L}_\mathcal{S}(\mathbf{Z}^{(t)},\mathbf{V}^{(t)},\mathbf{W}^{(t)})\right\|_F^2\\
			\overset{(a)}{\ge}& \frac{2\eta}{n}\sum_{(\mathbf{X},y)\in\mathcal{S}}\left(\mathcal{L}(\mathbf{Z}^{(t)},\mathbf{V}^{(t)},\mathbf{W}^{(t)},\mathbf{X}_i,y_i)-\log(1/\epsilon)\right) - \eta\mathcal{L}_\mathcal{S}(\mathbf{Z}^{(t)},\mathbf{V}^{(t)},\mathbf{W}^{(t)})\\
			=&\eta \mathcal{L}_\mathcal{S}(\mathbf{Z}^{(t)},\mathbf{V}^{(t)},\mathbf{W}^{(t)}) - 2\eta \epsilon,
		\end{aligned}
	\end{equation}
	where $(a)$ is by the homogeneity of the feature layer and convexity of the cross-entropy function.
	
	Rearranging (\ref{equ: descent1}), we have
	\begin{align}
		\frac{1}{T_2+1}\sum_{t'=0}^{T_2}\mathcal{L}_\mathcal{S}(\mathbf{Z}^{(T_1+t')},\mathbf{V}^{(T_1+t')},\mathbf{W}^{(T_1+t')}) \le \frac{\left\|\mathbf{W}^{(0)} - \mathbf{W}^*\right\|_F^2}{\eta (T_2+1)} + 2\epsilon.
	\end{align}
	Here, we have
	\begin{align}
		\left\|\mathbf{W}^{(0)} - \mathbf{W}^*\right\|_F^2 = \mathcal{O}\left(\frac{mN\log^2(1/\epsilon)}{\lambda^2d\sigma_0^2}\right).
	\end{align}
	As a result, we have
	\begin{align}
		\frac{1}{T_2+1}\sum_{t'=0}^{T_2}\mathcal{L}_\mathcal{S}(\mathbf{Z}^{(T_1+t')},\mathbf{V}^{(T_1+t')},\mathbf{W}^{(T_1+t')}) \le 3\epsilon.
	\end{align}
    Letting $\epsilon = 0.003/n$ finishes the proof.
\end{proof}

\subsection{Feature resemblance (joint training on \texorpdfstring{$\biguplus_{k=1}^\kappa(\mathcal{S}_1\cup\mathcal{S}_2)\biguplus\mathcal{S}_3$}{S\_1 U S\_2 U S\_3})}
Proposition 1 follows from the following Proposition 6 under the large-\(\kappa\) condition.
\begin{proposition}[Feature resemblance]
After the Stage-1 training, for all $i\in[N]$, we have
    $$\frac{\langle\mathbf{V}^{(T_1)}\mathbf{a}_i, \mathbf{V}^{(T_1)}\mathbf{a}'_i\rangle}{\|\mathbf{V}^{(T_1)}\mathbf{a}_i\|_2\|\mathbf{V}^{(T_1)}\mathbf{a}'_i\|_2}\ge 1 - 2\cdot\frac{\mathcal{O}(\frac{nm\log^2(d)}{(d\lambda^2)}+\frac{m\log^2(d)}{(\kappa^2\lambda^2)} +d\sigma_0^2)}{\mathcal{O}(\frac{nm\log^2(d)}{(d\lambda^2)}+d\sigma_0^2 +\frac{m\log^2(d)}{(\kappa^2\lambda^2)})+\Theta(\frac{m\log^2(d)}{\lambda^2})}.$$
\end{proposition}
\begin{proof}
    Based on Lemma \ref{lemma: j1}, both $\mathbf{V}^{(T_1)}\mathbf{a}$ and $\mathbf{V}^{(T_1)}\mathbf{a}'$ have large coefficients on the components $\mathbf{w}^{(0)}_{\mathcal{I}(\mathbf{b}_i),l}$ for all $l\in[m]$.
    Furthermore, other components together have a squared norm $\mathcal{O}(\frac{nm\log^2(d)}{(d\lambda^2)} + \frac{m\log^2(d)}{(\kappa^2\lambda^2)}+d\sigma_0^2)$.
    Therefore, under Condition \ref{condition:condition} and large-$\kappa$ regime, by the double-angle formula, we have 
    \begin{equation}
    \begin{aligned}
        \frac{\langle\mathbf{V}^{(T_1)}\mathbf{a}_i, \mathbf{V}^{(T_1)}\mathbf{a}'_i\rangle}{\|\mathbf{V}^{(T_1)}\mathbf{a}_i\|_2\|\mathbf{V}^{(T_1)}\mathbf{a}'_i\|_2} =& 1 - 2\cdot\frac{\mathcal{O}(\frac{nm\log^2(d)}{(d\lambda^2)}+\frac{m\log^2(d)}{(\kappa^2\lambda^2)}+d\sigma_0^2)}{\mathcal{O}(\frac{nm\log^2(d)}{(d\lambda^2)} + \frac{m\log^2(d)}{(\kappa^2\lambda^2)}+d\sigma_0^2)+\Theta(\frac{m\log^2(d)}{\lambda^2})}\\
        =& 1-o(1). 
    \end{aligned}
    \end{equation}
    This finishes the proof.
\end{proof}

\subsection{Proof of success rate bound (joint training on \texorpdfstring{$\biguplus_{k=1}^\kappa(\mathcal{S}_1\cup\mathcal{S}_2)\biguplus\mathcal{S}_3$}{S\_1 U S\_2 U S\_3})}

\begin{proof}
After training, there exists an iteration that $t\in(T_1, T_1+T_2]$ such that
\begin{align}
   \mathcal{L}_{\mathcal{S}}(\mathbf{Z}^{(t)}, \mathbf{V}^{(t)},\mathbf{W}^{(t)}) \le 0.01/n,
\end{align}
and 
\begin{align}
    \max_k\frac{1}{m}\sum_{l=1}^m\langle\mathbf{w}^{(t)}_{k,l},\mathbf{V}^{(t)}\mathbf{a}_i\rangle\le \mathcal{O}(\frac{\eta\lambda}{mn}T_2\|\mathbf{V}^{(t)}\mathbf{a}_i\|_2^2) = \mathcal{O}\left(mN\frac{\log^2(d)}{\lambda^3d\sigma_0^2}\right).
\end{align}
Let $\epsilon = 0.01$.
   After training, we have
    \begin{align}
        \mathcal{L}(\mathbf{Z}^{(t)}, \mathbf{V}^{(t)},\mathbf{W}^{(t)},\Xi^{(t)}([\mathbf{a}'_i\;\;\mathbf{r}_2]),\mathcal{I}(\mathbf{c}_i)) \le \epsilon,
    \end{align}
    implying that 
    \begin{align}
        \frac{\lambda}{m} \sum_{l=1}^m\langle\mathbf{w}^{(t)}_{\mathcal{I}(\mathbf{c}_i),l}, \mathbf{V}^{(t)}\Xi^{(t)}([\mathbf{a}'_i\;\;\mathbf{r}_2]))\rangle -\max_{k\neq \mathcal{I}(\mathbf{c}_i)}\frac{\lambda}{m} \sum_{l=1}^m\langle\mathbf{w}^{(t)}_{k,l}, \mathbf{V}^{(t)}\Xi^{(t)}([\mathbf{a}'_i\;\;\mathbf{r}_2]))\rangle\ge \log(\frac{\exp(-\epsilon)}{(1-\exp(-\epsilon)}).
    \end{align}
Additionally, by the stability bound for Stage 2, we have
\begin{equation}
\begin{aligned}
    &\frac{\lambda}{m} \sum_{l=1}^m\langle\mathbf{w}^{(t)}_{\mathcal{I}(\mathbf{c}_i),l}, \mathbf{V}^{(t)}\Xi^{(t)}([\mathbf{a}_i\;\;\mathbf{r}_2]))\rangle -\max_{k\neq \mathcal{I}(\mathbf{c}_i)}\frac{\lambda}{m} \sum_{l=1}^m\langle\mathbf{w}^{(t)}_{k,l}, \mathbf{V}^{(t)}\Xi^{(t)}([\mathbf{a}_i\;\;\mathbf{r}_2]))\rangle\\ \ge& \log(\frac{\exp(-\epsilon)}{(1-\exp(-\epsilon)}) - \mathcal{O}\left(\frac{mN\log^2(d)}{\kappa^2\lambda^2(d\sigma_0^2)^3}\right) - \mathcal{O}\left(\oldfrac{Nm\log^2(d)}{\lambda^2}\right)\cdot\left(\exp\left(\mathcal{O}\left(\frac{mn\log^2(d)}{\lambda^2d^{3/2}\sigma_0} \right)\right)-1\right).
\end{aligned}
\end{equation}
Therefore, under Condition 1 and large $\kappa$ regime, we have
\begin{align}
    \frac{\lambda}{m} \sum_{l=1}^m\langle\mathbf{w}^{(t)}_{\mathcal{I}(\mathbf{c}_i),l}, \mathbf{V}^{(t)}\Xi^{(t)}([\mathbf{a}_i\;\;\mathbf{r}_2]))\rangle > \max_{k\neq \mathcal{I}(\mathbf{c}_i)}\frac{\lambda}{m} \sum_{l=1}^m\langle\mathbf{w}^{(t)}_{k,l}, \mathbf{V}^{(t)}\Xi^{(t)}([\mathbf{a}_i\;\;\mathbf{r}_2]))\rangle.
\end{align}
Therefore, we have
\begin{align}
    \mathcal{L}^{0-1}_\mathcal{A}(\mathbf{Z}^{(t)},\mathbf{V}^{(t)},\mathbf{W}^{(t)}) = 0.
\end{align}
This finishes the proof.
\end{proof}

\section{Proofs of sequential training (first on \texorpdfstring{$\mathcal{S}_1\cup\mathcal{S}_2$}{S\_1 U S\_2}, then on \texorpdfstring{$\mathcal{S}_3$}{S\_3})}
We assume that Lemmas 3 and 4 hold with probability $1-2\delta'$ by the union bound.
\subsection{Properties during Stage-1 training (first on \texorpdfstring{$\mathcal{S}_1\cup\mathcal{S}_2$}{S\_1 U S\_2}, then on \texorpdfstring{$\mathcal{S}_3$}{S\_3})}
By the choice of $T_1$, Stage 1 stops before the logits saturate. Hence, we have that $1-\text{logit}_y(\mathbf{X}) = \Theta(1)$ holds for all $(\mathbf{X},y)\in\mathcal{S}_1\cup\mathcal{S}_2$. The following properties hold.
\begin{lemma}\label{lemma: s1}
	For any $t\in[0, T_1]$ with $T_1 = \Theta(\oldfrac{mn\log(d)}{\lambda^2\eta\sqrt{d}\sigma_0})$, we have
	\begin{enumerate}
		\item For all $(\mathbf{X},y) \in \biguplus_{k=1}^\kappa(\mathcal{S}_1\cup\mathcal{S}_2)\biguplus\mathcal{S}_3$, the attention scores are balanced, i.e.,
		\begin{align}
			\exp\left(-\mathcal{O}\left(\frac{n\log^2(d)}{\lambda^2d^{3/2}\sigma_0}\right)\right)\le \frac{\alpha_1(\mathbf{Z}^{(t)},\mathbf{X})}{\alpha_2(\mathbf{Z}^{(t)},\mathbf{X})} \le \exp\left(\mathcal{O}\left(\frac{n\log^2(d)}{\lambda^2d^{3/2}\sigma_0}\right)\right).
		\end{align}
		\item For all $i\in[N]$, $l_1, l_2 \in [m]$, the coefficients satisfy
		\begin{align}
			\gamma^{(T_1)}_{i,\mathcal{I}(\mathbf{b}_i),l_1} = \Theta\left(\frac{\log(d)}{\lambda\sqrt{d}\sigma_0}\right),
			\rho^{(T_1)}_{i,\mathcal{I}(\mathbf{b}_i),l_2} = \Theta\left(\frac{\log(d)}{\lambda\sqrt{d}\sigma_0}\right).
		\end{align}
		\item For all $i\in[N], j\in[d]\backslash\{\mathcal{I}(\mathbf{b}_i)\}_{i=1}^N$ and $l\in [m]$, the coefficients satisfy
		\begin{align}
			-\mathcal{O}\left(\frac{\log(d)}{\lambda d\sigma_0}\right)\le \gamma^{(T_1)}_{i,j,l} \le 0, -\mathcal{O}\left(\frac{\log(d)}{\lambda d\sigma_0}\right)\le \rho^{(T_1)}_{i,j,l} \le 0.
		\end{align}
		\item For all $k\in[N]$, $l\in[m]$ the coefficients satisfy
		\begin{align}
			\zeta_{1,\mathcal{I}(\mathbf{b}_k),l}^{(T_1)} = \Theta\left(\frac{\log(d)}{\lambda\sqrt{d}\sigma_0}\right).
		\end{align}
		\item For all $k\in[d]\backslash\{\mathcal{I}(\mathbf{b}_i)\}_{i=1}^N$, and $l\in[m]$, the coefficients satisfy
		\begin{align}
			-\mathcal{O}\left(\frac{\log(d)}{\lambda d\sigma_0}\right)\le\zeta_{1,k,l}^{(T_1)} \le 0.
		\end{align}
		\item For all $k\in[d]$, $l\in[m]$, the feature-layer weights remain unchanged, i.e.,for $l\in[m],k\in[d]$,
		\begin{align}
			\mathbf{w}^{(T_1)}_{k,l} = \mathbf{w}^{(0)}_{k,l}.
		\end{align}
	\end{enumerate}
\end{lemma}
\begin{proof}
	We first prove the first statement.
The update of matrix $\mathbf{Z}$ satisfies
\begin{equation}
	\begin{aligned}
		|\tilde{Z}^{(t+1)}_{\mathcal{I}(\mathbf{a}_i),\mathcal{I}(\mathbf{r}_1)} -\tilde{Z}^{(t)}_{\mathcal{I}(\mathbf{a}_i),\mathcal{I}(\mathbf{r}_1)}|
		=& \frac{\eta\lambda}{n\sqrt{d}}\alpha_1^{(t)}([\mathbf{a}_i,\mathbf{r}_1])\left|(\mathbf{e}_{\mathcal{I}(\mathbf{b}_i)}-\textbf{logit}^{(t)}([\mathbf{a}_i\;\;\mathbf{r}_1]))^\top\boldsymbol{\tau}^{(t)}\left(\mathbf{a}_i - \Xi^{(t)}([\mathbf{a}_i\;\; \mathbf{r}_1])\right)\right|\\
        =&\mathcal{O}\left(\frac{\eta\log(d)}{n\sqrt{d}}\right),
	\end{aligned}
\end{equation}
and 
\begin{equation}
	\begin{aligned}
		|\tilde{Z}^{(t+1)}_{\mathcal{I}(\mathbf{r}_1),\mathcal{I}(\mathbf{r}_1)} -\tilde{Z}^{(t)}_{\mathcal{I}(\mathbf{r}_1),\mathcal{I}(\mathbf{r}_1)}|
		=& \sum_{i=1}^N\frac{\eta\lambda}{n\sqrt{d}}\alpha_2^{(t)}([\mathbf{a}_i,\mathbf{r}_1])\left|(\mathbf{e}_{\mathcal{I}(\mathbf{b})}-\textbf{logit}^{(t)}([\mathbf{a}_i\;\;\mathbf{r}_1]))^\top\boldsymbol{\tau}^{(t)}\left(\mathbf{r}_1-\Xi^{(t)}([\mathbf{a}_i\;\;\mathbf{r}_1])\right)\right|\\
		=&\mathcal{O}\left(\frac{\eta\log(d)}{\sqrt{d}}\right).
	\end{aligned}
\end{equation}
Therefore, we have
\begin{align}
	|\tilde{Z}^{(T_1)}_{\mathcal{I}(\mathbf{a}_i),\mathcal{I}(\mathbf{r}_1)} -\tilde{Z}^{(0)}_{\mathcal{I}(\mathbf{a}_i),\mathcal{I}(\mathbf{r}_1)}| = \mathcal{O}\left(\frac{\eta\log(d)}{n\sqrt{d}}T_1\right) = \mathcal{O}\left(\frac{m\log^2(d)}{\lambda^2d\sigma_0}\right),
\end{align}
and
\begin{align}
	|\tilde{Z}^{(T_1)}_{\mathcal{I}(\mathbf{r}_1),\mathcal{I}(\mathbf{r}_1)} -\tilde{Z}^{(0)}_{\mathcal{I}(\mathbf{r}_1),\mathcal{I}(\mathbf{r}_1)}| = \mathcal{O}\left(\frac{\eta\log(d)}{\sqrt{d}}T_1\right) = \mathcal{O}\left(\frac{mn\log^2(d)}{\lambda^2d\sigma_0}\right).
\end{align}
By initialization, we have
\begin{align}
	|\tilde{Z}^{(0)}_{\mathcal{I}(\mathbf{a}_i),\mathcal{I}(\mathbf{r}_1)} -\tilde{Z}^{(0)}_{\mathcal{I}(\mathbf{r}_1),\mathcal{I}(\mathbf{r}_1)}| = \mathcal{O}(\sigma_0).
\end{align}
We have
\begin{equation}
\begin{aligned}
	|\tilde{Z}^{(T_1)}_{\mathcal{I}(\mathbf{a}_i),\mathcal{I}(\mathbf{r}_1)} -\tilde{Z}^{(T_1)}_{\mathcal{I}(\mathbf{r}_1),\mathcal{I}(\mathbf{r}_1)}| \le& |\tilde{Z}^{(T_1)}_{\mathcal{I}(\mathbf{a}_i),\mathcal{I}(\mathbf{r}_1)} -\tilde{Z}^{(0)}_{\mathcal{I}(\mathbf{a}_i),\mathcal{I}(\mathbf{r}_1)}| + |\tilde{Z}^{(0)}_{\mathcal{I}(\mathbf{a}_i),\mathcal{I}(\mathbf{r}_1)} -\tilde{Z}^{(0)}_{\mathcal{I}(\mathbf{r}_1),\mathcal{I}(\mathbf{r}_1)}| + 	|\tilde{Z}^{(T_1)}_{\mathcal{I}(\mathbf{r}_1),\mathcal{I}(\mathbf{r}_1)} -\tilde{Z}^{(0)}_{\mathcal{I}(\mathbf{r}_1),\mathcal{I}(\mathbf{r}_1)}|\\
	=& \mathcal{O}\left(\frac{mn\log^2(d)}{\lambda^2d\sigma_0}\right).
\end{aligned}
\end{equation}
Therefore, we have
\begin{align}
	\frac{\alpha_1(\Xi^{(t)}([\mathbf{a}_i\;\;\mathbf{r}_1]))}{\alpha_2(\Xi^{(t)}([\mathbf{a}_i\;\;\mathbf{r}_1]))} \le \exp\left(\mathcal{O}\left(\frac{mn\log^2(d)}{\lambda^2d^{3/2}\sigma_0}\right)\right).
\end{align}
Similarly, we have
\begin{align}
	\frac{\alpha_1(\Xi^{(t)}([\mathbf{a}_i\;\;\mathbf{r}_1]))}{\alpha_2(\Xi^{(t)}([\mathbf{a}_i\;\;\mathbf{r}_1]))} \ge \exp\left(-\mathcal{O}\left(\frac{mn\log(d)}{\lambda^2d^{3/2}\sigma_0}\right)\right).
\end{align}
Similarly, we can prove that
\begin{align}
    |\tilde{Z}^{(T_1)}_{\mathcal{I}(\mathbf{a}'_i),\mathcal{I}(\mathbf{r}_1)} -\tilde{Z}^{(T_1)}_{\mathcal{I}(\mathbf{r}_1),\mathcal{I}(\mathbf{r}_1)}| = \mathcal{O}\left(\frac{n\log^2(d)}{\lambda^2d\sigma_0}\right),
\end{align}
so that 
\begin{align}
    \frac{\alpha_1(\Xi^{(t)}([\mathbf{a}'_i\;\;\mathbf{r}_1]))}{\alpha_2(\Xi^{(t)}([\mathbf{a}'_i\;\;\mathbf{r}_1]))} \le \exp\left(\mathcal{O}\left(\frac{n\log(d)}{\lambda^2d^{3/2}\sigma_0}\right)\right), \frac{\alpha_1(\Xi^{(t)}([\mathbf{a}'_i\;\;\mathbf{r}_1]))}{\alpha_2(\Xi^{(t)}([\mathbf{a}'_i\;\;\mathbf{r}_1]))} \ge \exp\left(-\mathcal{O}\left(\frac{n\log^2(d)}{\lambda^2d^{3/2}\sigma_0}\right)\right).
\end{align}
	Next, we prove the second, third and fourth statements.
	Based on the gradient form,	the updates of $\mathbf{V}^{(t)}\mathbf{a}_i,\mathbf{V}^{(t)}\mathbf{a}'_i, \mathbf{V}^{(t)}\mathbf{r}_1$ for all $t\in [0,T_1]$ satisfy
\begin{equation}
	\begin{aligned}
		\mathbf{V}^{(t+1)}\mathbf{a}_i =& \mathbf{V}^{(t)}\mathbf{a}_i + \frac{\lambda\eta}{nm}\sum_{k=1}^d\sum_{l=1}^{m}(\mathbb{I}(k=\mathcal{I}(\mathbf{b}_i))-\text{logit}^{(t)}_{k}(\Xi^{(t)}([\mathbf{a}_i\;\;\mathbf{r}_1])))\mathbf{w}^{(0)}_{k,l}\\
		=& \mathbf{V}^{(t)}\mathbf{a}_i + \frac{\lambda\eta}{mn}\sum_{l=1}^m(1-\text{logit}^{(t)}_{\mathcal{I}(\mathbf{b}_i)}(\Xi^{(t)}([\mathbf{a}_i\;\;\mathbf{r}_1]))) \mathbf{w}^{(0)}_{\mathcal{I}(\mathbf{b}_i),l}\\
		&-\frac{\lambda\eta}{mn}\sum_{l=1}^m\sum_{k\neq\mathcal{I}(\mathbf{b}_i)}\text{logit}_{k}^{(t)}(\Xi^{(t)}([\mathbf{a}_i\;\;\mathbf{r}_1]))\mathbf{w}_{k,l}^{(0)},
	\end{aligned}
\end{equation}
\begin{equation}
	\begin{aligned}
		\mathbf{V}^{(t+1)}\mathbf{a}'_i =& \mathbf{V}^{(t)}\mathbf{a}'_i + \frac{\lambda\eta}{nm}\sum_{k=1}^d\sum_{l=1}^{m}(\mathbb{I}(k=\mathcal{I}(\mathbf{b}_i))-\text{logit}^{(t)}_{k}(\Xi^{(t)}([\mathbf{a}'_i\;\;\mathbf{r}_1])))\mathbf{w}^{(0)}_{k,l}\\
		=& \mathbf{V}^{(t)}\mathbf{a}'_i + \frac{\lambda\eta}{mn}\sum_{l=1}^m(1-\text{logit}^{(t)}_{\mathcal{I}(\mathbf{b}_i)}(\Xi^{(t)}([\mathbf{a}'_i\;\;\mathbf{r}_1]))) \mathbf{w}^{(0)}_{\mathcal{I}(\mathbf{b}_i),l}\\
		&-\frac{\lambda\eta}{mn}\sum_{l=1}^m\sum_{k\neq\mathcal{I}(\mathbf{b}_i)}\text{logit}_{k}^{(t)}(\Xi^{(t)}([\mathbf{a}'_i\;\;\mathbf{r}_1]))\mathbf{w}_{k,l}^{(0)},
	\end{aligned}
\end{equation}
and 
\begin{equation}
	\begin{aligned}
		\mathbf{V}^{(t+1)}\mathbf{r}_1
		= \mathbf{V}^{(t)}\mathbf{r}_1 &+ \frac{\lambda\eta}{nm}\sum_{i=1}^N\sum_{k=1}^d\sum_{l=1}^{m}(\mathbb{I}(k=\mathcal{I}(\mathbf{b}_i))-\text{logit}^{(t)}_{k}(\Xi^{(t)}([\mathbf{a}_i\;\;\mathbf{r}_1])))\mathbf{w}^{(0)}_{k,l}\\
        &+ \frac{\lambda\eta}{nm}\sum_{i=1}^N\sum_{k=1}^d\sum_{l=1}^{m}(\mathbb{I}(k=\mathcal{I}(\mathbf{b}_i))-\text{logit}^{(t)}_{k}(\Xi^{(t)}([\mathbf{a}'_i\;\;\mathbf{r}_1])))\mathbf{w}^{(0)}_{k,l},
	\end{aligned}
\end{equation}

Then, for all $t\in[0,T_1-1]$, we can simply conclude that
\begin{itemize}
	\item for all $i\in[N]$ and $l\in [m]$,
    \begin{equation}
	\begin{aligned}
		\gamma^{(t+1)}_{i,\mathcal{I}(\mathbf{b}_i),l} =& \gamma^{(t)}_{i,\mathcal{I}(\mathbf{b}_i),l} +  \frac{\lambda\eta}{mn}(1-\text{logit}^{(t)}_{\mathcal{I}(\mathbf{b}_i)}(\Xi^{(t)}([\mathbf{a}_i\;\;\mathbf{r}_1])))\\
        =& \gamma^{(t)}_{i,\mathcal{I}(\mathbf{b}_i),l} + \Theta(\frac{\lambda\eta}{mn}),\\
 	\end{aligned}
    \end{equation}
    \begin{equation}
	\begin{aligned}   
    &\rho^{(t+1)}_{i,\mathcal{I}(\mathbf{b}_i),l} = \rho^{(t)}_{i,\mathcal{I}(\mathbf{b}_i),l} +  \frac{\lambda\eta}{mn}(1-\text{logit}^{(t)}_{\mathcal{I}(\mathbf{b}_i)}(\Xi^{(t)}([\mathbf{a}'_i\;\;\mathbf{r}_1])))   =\rho^{(t)}_{i,\mathcal{I}(\mathbf{b}_i),l} + \Theta(\frac{\lambda\eta}{mn}).
	\end{aligned}
    \end{equation}
	\item for all $i,j\in[N]$ and $j\neq i$,
	\begin{align}
		&\gamma^{(t+1)}_{i,\mathcal{I}(\mathbf{b}_j),l} = \gamma^{(t)}_{i,,\mathcal{I}(\mathbf{b}_j),l} -\frac{\lambda\eta}{mn}\text{logit}_{\mathcal{I}(\mathbf{b}_i)}^{(t)}(\Xi^{(t)}([\mathbf{a}_j\;\;\mathbf{r}_1])),\\
        &\rho^{(t+1)}_{i,\mathcal{I}(\mathbf{b}_j),l} = \rho^{(t)}_{i,\mathcal{I}(\mathbf{b}_j),l} -\frac{\lambda\eta}{mn}\text{logit}_{\mathcal{I}(\mathbf{b}_i)}^{(t)}(\Xi^{(t)}([\mathbf{a}'_j\;\;\mathbf{r}_2])).
	\end{align} 
	\item for all $i\in[N]$ and $l\in[m]$,
	\begin{align}
		\zeta_{1,\mathcal{I}(\mathbf{b}_i),l}^{(t+1)} =& \zeta_{1,\mathcal{I}(\mathbf{b}_i),l}^{(t)} + \frac{\lambda\eta}{mn}(1-\text{logit}^{(t)}_{\mathcal{I}(\mathbf{b}_i)}(\Xi^{(t)}([\mathbf{a}_i\;\;\mathbf{r}_1]))) - \frac{\lambda\eta}{mn}\sum_{k\neq i} \text{logit}^{(t)}_{\mathcal{I}(\mathbf{b}_i)}(\Xi^{(t)}([\mathbf{a}_k\;\;\mathbf{r}_1])),\\
        &+ \frac{\lambda\eta}{mn}(1-\text{logit}^{(t)}_{\mathcal{I}(\mathbf{b}_i)}(\Xi^{(t)}([\mathbf{a}'_i\;\;\mathbf{r}_1]))) - \frac{\lambda\eta}{mn}\sum_{k\neq i}\text{logit}^{(t)}_{\mathcal{I}(\mathbf{b}_i)}(\Xi^{(t)}([\mathbf{a}'_k\;\;\mathbf{r}_1])) .
	\end{align}
\end{itemize}
As $\text{logit}_{k}(\mathbf{a}_i\mathbf{r}_1) <\frac{2}{\sqrt{d}}$ for all $k\in[d]\backslash\{\mathcal{I}(\mathbf{b}_i)\}$ due to the fact that $\text{logit}_{k}(\Xi^{(t)}([\mathbf{a}_i\;\;\mathbf{r}_1]))\le \text{logit}_{\mathcal{I}(\mathbf{b}_i)}(\Xi^{(t)}([\mathbf{a}_i\;\;\mathbf{r}_1]))/2$, we have
\begin{itemize}
	\item for all $i\in[N]$ and $l\in[m]$,
	\begin{align}
		&\gamma^{(T_1)}_{i,\mathcal{I}(\mathbf{b}_i),l} =  \Theta\left(\frac{\lambda\eta}{mn}T_1\right) = \Theta\left(\frac{\log(d)}{\lambda\sqrt{d}\sigma_0}\right),\\
        &\rho^{(T_1)}_{i,\mathcal{I}(\mathbf{b}_i),l} =  \Theta\left(\frac{\lambda\eta}{mn}T_1\right) = \Theta\left(\frac{\log(d)}{\lambda \sqrt{d}\sigma_0}\right),
	\end{align}
	\item for all $i\in[N]$ and $j\notin [d]\backslash\{\mathcal{I}(\mathbf{b}_i)\}_{i=1}^N$,
	\begin{align}
		&-\mathcal{O}\left(\frac{\lambda\eta}{mn\sqrt{d}}T_1\right) = -\mathcal{O}\left(\frac{\log(d)}{\lambda d\sigma_0}\right) \le\gamma^{(T_1)}_{i,j,l} \le \gamma^{(0)}_{i,j,l} = 0,\\
        &-\mathcal{O}\left(\frac{\lambda\eta}{mn\sqrt{d}}T_1\right) = -\mathcal{O}\left(\frac{\log(d)}{\lambda d\sigma_0}\right) \le\rho^{(T_1)}_{i,j,l} \le \rho^{(0)}_{i,j,l} = 0,
	\end{align} 
	\item for all $i\in[N]$ and $l\in[m]$,
	\begin{align}
		&\zeta_{1,\mathcal{I}(\mathbf{b}_i),l}^{(T_1)} = \Theta\left(\frac{\eta\lambda}{mn}T_1\right) = \Theta\left(\frac{\log(d)}{\lambda\sqrt{d}\sigma_0}\right).
	\end{align}
    \item for all $j\notin [d]\backslash\{\mathcal{I}(\mathbf{b}_i)\}_{i=1}^N$, 
    \begin{align}
        &-\mathcal{O}\left(\frac{\lambda\eta}{mn\sqrt{d}}T_1\right) = -\mathcal{O}\left(\frac{\log(d)}{\lambda d\sigma_0}\right) \le\zeta_{1,j,l}^{(T_1)} \le \zeta^{(0)}_{1,j,l} = 0.
    \end{align}
\end{itemize}

The sixth statement holds as the feature layer is not updated.

This completes the proof.
\end{proof}

\subsection{Convergence of Stage-2 training (first on \texorpdfstring{$\mathcal{S}_1\cup\mathcal{S}_2$}{S\_1 U S\_2}, then on \texorpdfstring{$\mathcal{S}_3$}{S\_3})}

\begin{lemma}
	There exists an iteration $t \in(T_1, T_2]$, where $T_2 = \Theta(\oldfrac{mN^2}{\lambda^2d\sigma_0^2\eta})$ such that
	\begin{align}
		\mathcal{L}_{\mathcal{S}_1\cup \mathcal{S}_2}(\mathbf{Z}^{(t)},\mathbf{V}^{(t)},\mathbf{W}^{(t)}) \le \frac{0.01}{N}.
	\end{align}
\end{lemma}
\begin{proof}
	Let 
	\begin{equation}
		\mathbf{w}^*_{k,l} = 
        \mathbf{w}^{(0)}_{k,l} + 
    \begin{cases}
C\cdot\log(1/\epsilon)\left(\frac{\mathbf{V}^{(0)}\mathbf{a}_i}{\|\mathbf{V}^{(0)}\mathbf{a}_i\|^2_2}+ \frac{\mathbf{V}^{(0)}\mathbf{a}'_i}{\|\mathbf{V}^{(0)}\mathbf{a}'_i\|^2_2}\right)/\lambda, & \text{ if $k=\mathcal{I}(\mathbf{b}_i)$},\\
0 & \text{ if $k\neq\mathcal{I}(\mathbf{b}_i), \forall i\notin [N]$}.\\
\end{cases}
	\end{equation}
	 for all $k\in[d], i\in[N]$ and $l\in[m]$.
	 For simplicity, we denote $\mathcal{S} = \mathcal{S}_1\cup \mathcal{S}_2$ and $n = |\mathcal{S}| = 2N$.
	Then, we have
	\begin{equation}\label{equ: increment3}
		\begin{aligned}
			&\left\|\mathbf{W}^{(t)} - \mathbf{W}^*\right\|_F^2 - \left\|\mathbf{W}^{(t+1)} - \mathbf{W}^*\right\|_F^2\\
			=&  2\eta\langle\nabla \mathcal{L}_{\mathcal{S}}(\mathbf{Z}^{(t)},\mathbf{V}^{(t)},\mathbf{W}^{(t)}), \mathbf{W}^{(t)} - \mathbf{W}^*\rangle - \eta^2 \left\|\nabla \mathcal{L}_{\mathcal{S}}(\mathbf{Z}^{(t)},\mathbf{V}^{(t)},\mathbf{W}^{(t)})\right\|_F^2\\
			=& \frac{2\eta}{n} \sum_{(\mathbf{X},y)\in\mathcal{S}}\sum_{k=1}^d \frac{\partial \mathcal{L}}{\partial f_k} \langle\nabla f_k(\mathbf{Z}^{(t)},\mathbf{V}^{(t)},\mathbf{W}^{(t)},\mathbf{X}),\mathbf{W}^{(t)}\rangle -
			\frac{2\eta}{n} \sum_{(\mathbf{X},y)\in\mathcal{S}}\sum_{k=1}^d \frac{\partial \mathcal{L}}{\partial f_k} \underbrace{\langle \nabla f_k(\mathbf{Z}^{(t)},\mathbf{V}^{(t)},\mathbf{W}^{(t)},\mathbf{X}), \mathbf{W}^*\rangle}_{A}\\
			&- \underbrace{\eta^2 \left\|\nabla \mathcal{L}_\mathcal{S}(\mathbf{Z}^{(t)},\mathbf{V}^{(t)},\mathbf{W}^{(t)})\right\|_F^2}_{B}.
		\end{aligned}
	\end{equation}
	Next, we bound the terms $A$ and $B$.
	For the term $A$, we have
	\begin{align}\label{equ: A3}
		A 
		\begin{cases}\ge \frac{C\log(1/\epsilon)}{\lambda} - \mathcal{O}(\oldfrac{\sqrt{N}\log(d)\log(1/\epsilon)}{\lambda^2\sqrt{md}\sqrt{d\sigma_0^2}})& \text{if } k = y, \\
			\le \mathcal{O}(\oldfrac{\sqrt{N}\log(d)\log(1/\epsilon)}{\lambda^2\sqrt{md}\sqrt{d\sigma_0^2}}) & \text{if } k \neq y.
		\end{cases}
	\end{align}
	For the term $B$, we have
	\begin{equation}\label{equ: B3}
	\begin{aligned}
		B \le& \eta^2 \left[ \frac{1}{n} \sum_{(\mathbf{X},y)\in\mathcal{S}}\sum_{k\in[d]}\left|\frac{\partial \mathcal{L}(\mathbf{Z}^{(t)},\mathbf{V}^{(t)},\mathbf{W}^{(t)},\mathbf{X},y)}{\partial f_k}\right| \left\| \nabla f_k(\mathbf{Z}^{(t)},\mathbf{V}^{(t)},\mathbf{W}^{(t)},\mathbf{X})\right\|_F\right]^2\\
		\le& 2\eta^2\max_{(\mathbf{X},y)\in\mathcal{S}}\|\mathbf{V}^{(t)}\Xi^{(t)}(\mathbf{X})\|_2^2\cdot\mathcal{L}_\mathcal{S}(\mathbf{Z},\mathbf{V},\mathbf{W}).
	\end{aligned}
	\end{equation}
	Next, we need to bound $\max_{(\mathbf{X},y)\in\mathcal{S}}\|\mathbf{V}^{(t)}\Xi^{(t)}(\mathbf{X})\|_2^2$. 
	By Lemma \ref{lemma: s1}, We have
	\begin{align}
		\max_{(\mathbf{X},y)\in\mathcal{S}}\|\mathbf{V}^{(t)}\Xi^{(t)}(\mathbf{X})\|_2^2 = \max_{(\mathbf{X},y)\in\mathcal{S}}\|\mathbf{V}^{(T_1)}\Xi^{(T_1)}(\mathbf{X})\|_2^2 = \mathcal{O}(Nmd\sigma_0^2\log^2(d)/\lambda^2).
	\end{align}
	
	Combining (\ref{equ: A3}), (\ref{equ: B3}), and (\ref{equ: increment3}), with $\eta = \mathcal{O}(\oldfrac{\lambda^2}{Nmd\sigma_0^2\log^2(d)})$, we have
	\begin{equation}\label{equ: descent5}
		\begin{aligned}
			&\left\|\mathbf{W}^{(t)} - \mathbf{W}^*\right\|_F^2 - \left\|\mathbf{W}^{(t+1)} - \mathbf{W}^*\right\|_F^2\\
			=&  \frac{2\eta}{n} \sum_{(\mathbf{X},y)\in\mathcal{S}}\left(\sum_{k=1}^d \frac{\partial \mathcal{L}}{\partial f_k} \langle\nabla f_k(\mathbf{Z}^{(t)},\mathbf{V}^{(t)},\mathbf{W}^{(t)},\mathbf{X}),\mathbf{W}^{(t)}\rangle - \sum_{k=1}^d \frac{\partial \mathcal{L}}{\partial f_k} \langle \nabla f_k(\mathbf{Z}^{(t)},\mathbf{V}^{(t)},\mathbf{W}^{(t)},\mathbf{X}), \mathbf{W}^*\rangle\right)\\
			&- \eta^2 \left\|\nabla \mathcal{L}_\mathcal{S}(\mathbf{Z}^{(t)},\mathbf{V}^{(t)},\mathbf{W}^{(t)})\right\|_F^2\\
			\overset{(a)}{\ge}& \frac{2\eta}{n}\sum_{(\mathbf{X},y)\in\mathcal{S}}\left(\mathcal{L}(\mathbf{Z}^{(t)},\mathbf{V}^{(t)},\mathbf{W}^{(t)},\mathbf{X}_i,y_i)-\log(1/\epsilon)\right) - \eta\mathcal{L}_\mathcal{S}(\mathbf{Z}^{(t)},\mathbf{V}^{(t)},\mathbf{W}^{(t)})\\
			=&\eta \mathcal{L}_\mathcal{S}(\mathbf{Z}^{(t)},\mathbf{V}^{(t)},\mathbf{W}^{(t)}) - 2\eta \epsilon,
		\end{aligned}
	\end{equation}
	where $(a)$ is by the homogeneity of the feature layer and convexity of the cross-entropy function.
	
	As a result, we have
	\begin{align}
		\frac{1}{T_2+1}\sum_{t'=0}^{T_2}\mathcal{L}_{\mathcal{S}_1\cup\mathcal{S}_2}(\mathbf{Z}^{(T_1+t')},\mathbf{V}^{(T_1+t')},\mathbf{W}^{(T_1+t')}) \le \frac{\left\|\mathbf{W}^{(0)} - \mathbf{W}^*\right\|_F^2}{\eta (T_2+1)} + 2\epsilon.
	\end{align}
	Here, we have
	\begin{align}
		\left\|\mathbf{W}^{(0)} - \mathbf{W}^*\right\|_F^2 = \mathcal{O}\left(\frac{mN\log^2(1/\epsilon)}{\lambda^2d\sigma_0^2}\right).
	\end{align}
    With $T_2 = \Theta(\oldfrac{mN\log^2(1/\epsilon)}{\lambda^2d\sigma_0^2\eta \epsilon})$, we have
    \begin{align}
        \frac{1}{T_2+1}\sum_{t'=0}^{T_2}\mathcal{L}_{\mathcal{S}_1\cup\mathcal{S}_2}(\mathbf{Z}^{(T_1+t')},\mathbf{V}^{(T_1+t')},\mathbf{W}^{(T_1+t')}) \le 3 \epsilon.
    \end{align}
    Letting $\epsilon_ = 0.003/N$ finishes the proof.
\end{proof}
\subsection{Convergence of Stage-3 training (first on \texorpdfstring{$\mathcal{S}_1\cup\mathcal{S}_2$}{S\_1 U S\_2}, then on \texorpdfstring{$\mathcal{S}_3$}{S\_3})}
\begin{lemma}\label{lemma: ss3}
	There exists an iteration $t \in(T_1+T_2, T_1 + T_2 + T_3]$, where $T_3 = \Theta(\oldfrac{mN^2}{\lambda^2d\sigma_0^2\eta})$ such that
	\begin{align}
		\mathcal{L}_{\mathcal{S}_3}(\mathbf{Z}^{(t)},\mathbf{V}^{(t)},\mathbf{W}^{(t)}) \le \frac{0.01}{N}.
	\end{align}
\end{lemma}
\begin{proof}
	Let 
	\begin{equation}
		\mathbf{w}^*_{k,l} = 
        \mathbf{w}^{(0)}_{k,l} + 
    \begin{cases}
C/\lambda\cdot\log(1/\epsilon)\frac{\mathbf{V}^{(0)}\mathbf{a}'_i}{\|\mathbf{V}^{(0)}\mathbf{a}'_i\|^2_2}, & \text{ if $k=\mathcal{I}(\mathbf{c}_i)$},\\
0 & \text{ if $k\neq\mathcal{I}(\mathbf{c}_i), \forall i\notin [N]$}.\\
\end{cases}
	\end{equation}
	 for all $k\in[d], i\in[N]$ and $l\in[m]$.
	 For simplicity, we denote $\mathcal{S} = \mathcal{S}_3$ and $n = |\mathcal{S}| = N$.
	Then, we have
	\begin{equation}\label{equ: increment4}
		\begin{aligned}
			&\left\|\mathbf{W}^{(t)} - \mathbf{W}^*\right\|_F^2 - \left\|\mathbf{W}^{(t+1)} - \mathbf{W}^*\right\|_F^2\\
			=&  2\eta\langle\nabla \mathcal{L}_{\mathcal{S}}(\mathbf{Z}^{(t)},\mathbf{V}^{(t)},\mathbf{W}^{(t)}), \mathbf{W}^{(t)} - \mathbf{W}^*\rangle - \eta^2 \left\|\nabla \mathcal{L}_{\mathcal{S}}(\mathbf{Z}^{(t)},\mathbf{V}^{(t)},\mathbf{W}^{(t)})\right\|_F^2\\
			=& \frac{2\eta}{n} \sum_{(\mathbf{X},y)\in\mathcal{S}}\sum_{k=1}^d \frac{\partial \mathcal{L}}{\partial f_k} \langle\nabla f_k(\mathbf{Z}^{(t)},\mathbf{V}^{(t)},\mathbf{W}^{(t)},\mathbf{X}),\mathbf{W}^{(t)}\rangle -
			\frac{2\eta}{n} \sum_{(\mathbf{X},y)\in\mathcal{S}}\sum_{k=1}^d \frac{\partial \mathcal{L}}{\partial f_k} \underbrace{\langle \nabla f_k(\mathbf{Z}^{(t)},\mathbf{V}^{(t)},\mathbf{W}^{(t)},\mathbf{X}), \mathbf{W}^*\rangle}_{A}\\
			&- \underbrace{\eta^2 \left\|\nabla \mathcal{L}_\mathcal{S}(\mathbf{Z}^{(t)},\mathbf{V}^{(t)},\mathbf{W}^{(t)})\right\|_F^2}_{B}.
		\end{aligned}
	\end{equation}
	Next, we bound the terms $A$ and $B$.
	For the term $A$, we have
	\begin{align}\label{equ: A4}
		A 
		\begin{cases}\ge \frac{C\log(1/\epsilon)}{\lambda} - \mathcal{O}(\oldfrac{\sqrt{N}\log(d)\log(1/\epsilon)}{\lambda^2\sqrt{md}\sqrt{d\sigma_0^2}})& \text{if } k = y, \\
			\le \mathcal{O}(\oldfrac{\sqrt{N}\log(d)\log(1/\epsilon)}{\lambda^2\sqrt{md}\sqrt{d\sigma_0^2}}) & \text{if } k \neq y.
		\end{cases}
	\end{align}
	For the term $B$, we have
	\begin{equation}\label{equ: B4}
	\begin{aligned}
		B \le& \eta^2 \left[ \frac{1}{n} \sum_{(\mathbf{X},y)\in\mathcal{S}}\sum_{k\in[d]}\left|\frac{\partial \mathcal{L}(\mathbf{Z}^{(t)},\mathbf{V}^{(t)},\mathbf{W}^{(t)},\mathbf{X},y)}{\partial f_k}\right| \left\| \nabla f_k(\mathbf{Z}^{(t)},\mathbf{V}^{(t)},\mathbf{W}^{(t)},\mathbf{X})\right\|_F\right]^2\\
		\le& 2\eta^2\max_{(\mathbf{X},y)\in\mathcal{S}}\|\mathbf{V}^{(t)}\Xi^{(t)}(\mathbf{X})\|_2^2\cdot\mathcal{L}_\mathcal{S}(\mathbf{Z},\mathbf{V},\mathbf{W}).
	\end{aligned}
	\end{equation}
	Next, we need to bound $\max_{(\mathbf{X},y)\in\mathcal{S}}\|\mathbf{V}^{(t)}\Xi^{(t)}(\mathbf{X})\|_2^2$. 
	By Lemma \ref{lemma: s1}, We have
	\begin{align}
		\max_{(\mathbf{X},y)\in\mathcal{S}}\|\mathbf{V}^{(t)}\Xi^{(t)}(\mathbf{X})\|_2^2 = \max_{(\mathbf{X},y)\in\mathcal{S}}\|\mathbf{V}^{(T_1)}\Xi^{(T_1)}(\mathbf{X})\|_2^2 = \mathcal{O}(Nmd\sigma_0^2\log^2(d)/\lambda^2).
	\end{align}
	
	Combining (\ref{equ: A4}), (\ref{equ: B4}), and (\ref{equ: increment4}), with $\eta = \Theta(\oldfrac{\lambda^2}{Nmd\sigma_0^2\log^2(d)})$, we have
	\begin{equation}\label{equ: descent}
		\begin{aligned}
			&\left\|\mathbf{W}^{(t)} - \mathbf{W}^*\right\|_F^2 - \left\|\mathbf{W}^{(t+1)} - \mathbf{W}^*\right\|_F^2\\
			=&  \frac{2\eta}{n} \sum_{(\mathbf{X},y)\in\mathcal{S}}\left(\sum_{k=1}^d \frac{\partial \mathcal{L}}{\partial f_k} \langle\nabla f_k(\mathbf{Z}^{(t)},\mathbf{V}^{(t)},\mathbf{W}^{(t)},\mathbf{X}),\mathbf{W}^{(t)}\rangle - \sum_{k=1}^d \frac{\partial \mathcal{L}}{\partial f_k} \langle \nabla f_k(\mathbf{Z}^{(t)},\mathbf{V}^{(t)},\mathbf{W}^{(t)},\mathbf{X}), \mathbf{W}^*\rangle\right)\\
			&- \eta^2 \left\|\nabla \mathcal{L}_\mathcal{S}(\mathbf{Z}^{(t)},\mathbf{V}^{(t)},\mathbf{W}^{(t)})\right\|_F^2\\
			\overset{(a)}{\ge}& \frac{2\eta}{n}\sum_{(\mathbf{X},y)\in\mathcal{S}}\left(\mathcal{L}(\mathbf{Z}^{(t)},\mathbf{V}^{(t)},\mathbf{W}^{(t)},\mathbf{X}_i,y_i)-\log(1/\epsilon)\right) - \eta\mathcal{L}_\mathcal{S}(\mathbf{Z}^{(t)},\mathbf{V}^{(t)},\mathbf{W}^{(t)})\\
			=&\eta \mathcal{L}_\mathcal{S}(\mathbf{Z}^{(t)},\mathbf{V}^{(t)},\mathbf{W}^{(t)}) - 2\eta \epsilon,
		\end{aligned}
	\end{equation}
	where $(a)$ is by the homogeneity of the feature layer and convexity of the cross-entropy function.
	
	Then, we have
	\begin{align}
		\frac{1}{T_3+1}\sum_{t'=0}^{T_3}\mathcal{L}_{\mathcal{S}_3}(\mathbf{Z}^{(T_1+T_2+t')},\mathbf{V}^{(T_1+T_2+t')},\mathbf{W}^{(T_1+T_2+t')}) \le \frac{\left\|\mathbf{W}^{(0)} - \mathbf{W}^*\right\|_F^2}{\eta (T_3+1)} + 2\epsilon.
	\end{align}
	Here, we have
	\begin{align}
		\left\|\mathbf{W}^{(0)} - \mathbf{W}^*\right\|_F^2 = \mathcal{O}\left(\frac{mN\log^2(1/\epsilon)}{\lambda^2d\sigma_0^2}\right)
	\end{align}
	As a result, when $T_3 =\Theta(\oldfrac{mN\log^2(1/\epsilon)}{\lambda^2d\sigma_0^2\eta \epsilon})$ we have
	\begin{align}
		\frac{1}{T_3+1}\sum_{t'=0}^{T_3}\mathcal{L}_\mathcal{S}(\mathbf{Z}^{(T_1+T_2+t')},\mathbf{V}^{(T_1+T_2+t')},\mathbf{W}^{(T_1+T_2+t')}) \le 3\epsilon.
	\end{align}
    
    Letting $\epsilon = 0.003/N$ finishes the proof.
\end{proof}
\subsection{Feature resemblance (first on \texorpdfstring{$\mathcal{S}_1\cup\mathcal{S}_2$}{S\_1 U S\_2}, then on \texorpdfstring{$\mathcal{S}_3$}{S\_3})}

\begin{proposition}[Feature resemblance]
After the Stage-1 training, for all $i\in[N]$, we have
    $$\frac{\langle\mathbf{V}^{(T_1)}\mathbf{a}_i, \mathbf{V}^{(T_1)}\mathbf{a}'_i\rangle}{\|\mathbf{V}^{(T_1)}\mathbf{a}_i\|_2\|\mathbf{V}^{(T_1)}\mathbf{a}'_i\|_2}\ge 1 - 2\cdot\frac{\mathcal{O}(\frac{nm\log^2(d)}{(d\lambda^2)}+d\sigma_0^2)}{\mathcal{O}(\frac{nm\log^2(d)}{(d\lambda^2)}+d\sigma_0^2)+\Theta(\frac{m\log^2(d)}{\lambda^2})}.$$
\end{proposition}
\begin{proof}
    Based on Lemma \ref{lemma: s1}, both $\mathbf{V}^{(T_1)}\mathbf{a}$ and $\mathbf{V}^{(T_1)}\mathbf{a}'$ have large coefficients on the components $\mathbf{w}^{(0)}_{\mathcal{I}(\mathbf{b}_i),l}$ for all $l\in[m]$.
    Furthermore, other components together have a squared norm $\mathcal{O}(\frac{n\log^2(d)}{(d\lambda^2)}+d\sigma_0^2)$.
    Therefore, by the double-angle formula and Condition \ref{condition:condition}, we have 
    \begin{equation}
    \begin{aligned}
        \frac{\langle\mathbf{V}^{(T_1)}\mathbf{a}_i, \mathbf{V}^{(T_1)}\mathbf{a}'_i\rangle}{\|\mathbf{V}^{(T_1)}\mathbf{a}_i\|_2\|\mathbf{V}^{(T_1)}\mathbf{a}'_i\|_2} =& 1 - 2\cdot\frac{\mathcal{O}(\frac{nm\log^2(d)}{(d\lambda^2)}+d\sigma_0^2)}{\mathcal{O}(\frac{nm\log^2(d)}{(d\lambda^2)}+d\sigma_0^2)+\Theta(\frac{m\log^2(d)}{\lambda^2})}\\
        =& 1- \mathcal{O}(\frac{n}{d})\\
        =& 1-o(1). 
    \end{aligned}
    \end{equation}
    This finishes the proof.
\end{proof}

\subsection{Proof of success rate bound (first on \texorpdfstring{$\mathcal{S}_1\cup\mathcal{S}_2$}{S\_1 U S\_2}, then on \texorpdfstring{$\mathcal{S}_3$}{S\_3})}
\begin{proof}
Based on Lemma \ref{lemma: ss3}, for $T_3 = \Theta(\oldfrac{mN^2}{\lambda^2d\sigma_0^2\eta })$, there exists $t\in(T_1+T_2,T_1+T_2+T_3]$, we have
\begin{align}\label{equ:loss}
   \mathcal{L}_{\mathcal{S}_3}(\mathbf{Z}^{(t)}, \mathbf{V}^{(t)},\mathbf{W}^{(t)}) \le \frac{0.01}{N},
\end{align}
and 
\begin{align}
    \max_k\frac{1}{m}\sum_{l=1}^m\langle\mathbf{w}^{(t)}_{k,l},\mathbf{V}^{(t)}\mathbf{a}_i\rangle\le \mathcal{O}(\frac{\eta\lambda}{mn}T_3\|\mathbf{V}^{(t)}\mathbf{a}_i\|_2^2) = \mathcal{O}\left(mN\frac{\log^2(d)}{\lambda^3d\sigma_0^2}\right).
\end{align}
Let $\epsilon = 0.01$.
    After training, from equation (\ref{equ:loss}), we have
    \begin{align}
        \mathcal{L}(\mathbf{Z}^{(t)}, \mathbf{V}^{(t)},\mathbf{W}^{(t)},\Xi^{(t)}([\mathbf{a}'_i\,\,\mathbf{r}_2]),\mathcal{I}(\mathbf{c}_i)) \le \epsilon,
    \end{align}
    implying that 
    \begin{align}
        \frac{\lambda}{m} \sum_{l=1}^m\langle\mathbf{w}_{\mathcal{I}(\mathbf{c}_i),l}, \mathbf{V}^{(t)}\Xi^{(t)}([\mathbf{a}'_i\;\;\mathbf{r}_2]))\rangle -\max_{k\neq \mathcal{I}(\mathbf{c}_i)}\frac{\lambda}{m} \sum_{l=1}^m\langle\mathbf{w}_{k,l}, \mathbf{V}^{(t)}\Xi^{(t)}([\mathbf{a}'_i\;\;\mathbf{r}_2]))\rangle\ge \log(\frac{\exp(-\epsilon)}{(1-\exp(-\epsilon)}).
    \end{align}
As a result, we have
\begin{equation}
\begin{aligned}
    &\frac{\lambda}{m} \sum_{l=1}^m\langle\mathbf{w}_{\mathcal{I}(\mathbf{c}_i),l}, \mathbf{V}^{(t)}\Xi^{(t)}([\mathbf{a}_i\;\;\mathbf{r}_2]))\rangle -\max_{k\neq \mathcal{I}(\mathbf{c}_i)}\frac{\lambda}{m} \sum_{l=1}^m\langle\mathbf{w}_{k,l}, \mathbf{V}^{(t)}\Xi^{(t)}([\mathbf{a}_i\;\;\mathbf{r}_2]))\rangle\\
    \ge& \log(\frac{\exp(-\epsilon)}{(1-\exp(-\epsilon)}) - \mathcal{O}\left(\frac{nN\log^2(d)\log^2(1/\epsilon)}{d\epsilon\lambda^2d\sigma_0^2}\right)- \mathcal{O}\left(\exp\left(\tilde{\mathcal{O}}\left(\frac{n\log(d)}{\lambda^2\sqrt{d}}\right)\right)-1\right) \frac{Nd\sigma_0^2\log^2(d)}{\lambda^2}.
\end{aligned}
\end{equation}
Therefore, under Condition 1, we have
\begin{align}
    \frac{\lambda}{m} \sum_{l=1}^m\langle\mathbf{w}_{\mathcal{I}(\mathbf{c}_i),l}, \mathbf{V}^{(t)}\Xi^{(t)}([\mathbf{a}_i\;\;\mathbf{r}_2]))\rangle > \max_{k\neq \mathcal{I}(\mathbf{c}_i)}\frac{\lambda}{m} \sum_{l=1}^m\langle\mathbf{w}_{k,l}, \mathbf{V}^{(t)}\Xi^{(t)}([\mathbf{a}_i\;\;\mathbf{r}_2]))\rangle.
\end{align}
Therefore, we have
\begin{align}
    \mathcal{L}^{0-1}_\mathcal{A}(\mathbf{Z}^{(t)},\mathbf{V}^{(t)},\mathbf{W}^{(t)}) = 0.
\end{align}

This finishes the proof.
\end{proof}

\section{Proof of sequential training (first on \texorpdfstring{$\mathcal{S}_1\cup\mathcal{S}_3$}{S\_1 U S\_3}, then on \texorpdfstring{$\mathcal{S}_2$}{S\_2})}
We assume that Lemmas 3 and 4 hold with probability $1-2\delta'$ by the union bound.
\subsection{Properties during Stage-1 training (first on \texorpdfstring{$\mathcal{S}_1\cup\mathcal{S}_3$}{S\_1 U S\_3}, then on \texorpdfstring{$\mathcal{S}_2$}{S\_2})}
By the choice of $T_1$, Stage 1 stops before the logits saturate. Hence, we have that $1-\text{logit}_y(\mathbf{X}) = \Theta(1)$ holds for all $(\mathbf{X},y)\in\mathcal{S}_1\cup\mathcal{S}_3$. The following properties hold.
\begin{lemma}\label{lemma: s2}
	For any $t\in[0, T_1]$ with $T_1 = \Theta(\oldfrac{mn\log(d)}{\lambda^2\eta \sqrt{d}\sigma_0})$, we have
	\begin{enumerate}
		\item For all data $(\mathbf{X},y)\in\mathcal{S}_1\cup\mathcal{S}_3$, the attention scores are balanced, i.e.,
		\begin{align}
			0.7\le \frac{\alpha_1(\mathbf{Z}^{(t)},\mathbf{X})}{\alpha_2(\mathbf{Z}^{(t)},\mathbf{X})} \le 1.5.
		\end{align}
		\item For all $i\in[N]$, $l_1, l_2 \in [m]$, the coefficients satisfy
		\begin{align}
			\gamma^{(T_1)}_{i,\mathcal{I}(\mathbf{b}_i),l_1} = \Theta\left(\frac{\log(d)}{\lambda \sqrt{d}\sigma_0}\right),
			\rho^{(T_1)}_{i,\mathcal{I}(\mathbf{c}_i),l_2} = \Theta\left(\frac{\log(d)}{\lambda \sqrt{d}\sigma_0}\right).
		\end{align}
		\item For all $i\in[N], j_1\in[d]\backslash\{\mathcal{I}(\mathbf{b}_i)\}_{i=1}^N, j_2 \in [d]\backslash\{\mathcal{I}(\mathbf{c}_i)\}_{i=1}^N$ and $l\in [m]$, the coefficients satisfy
		\begin{align}
			-\mathcal{O}\left(\frac{\log(d)}{\lambda d\sigma_0}\right)\le \gamma^{(T_1)}_{i,j_1,l} \le 0, -\mathcal{O}\left(\frac{\log(d)}{\lambda d\sigma_0}\right)\le \rho^{(T_1)}_{i,j_2,l} \le 0.
		\end{align}
		\item For all $k\in[N]$, $l\in[m]$ the coefficients satisfy
		\begin{align}
			\zeta_{1,\mathcal{I}(\mathbf{b}_k),l}^{(T_1)} = \Theta\left(\frac{\log(d)}{\lambda \sqrt{d}\sigma_0}\right),\zeta_{2,\mathcal{I}(\mathbf{c}_k),l}^{(T_1)} = \Theta\left(\frac{\log(d)}{\lambda \sqrt{d}\sigma_0}\right).
		\end{align}
		\item For all $k_1\in[d]\backslash\{\mathcal{I}(\mathbf{b}_i)\}_{i=1}^N$, $k_2\in[d]\backslash\{\mathcal{I}(\mathbf{c}_i)\}_{i=1}^N$, and $l\in[m]$, the coefficients satisfy
		\begin{align}
			-\mathcal{O}\left(\frac{\log(d)}{\lambda d\sigma_0}\right)\le\zeta_{1,k_1,l}^{(T_1)} \le 0, -\mathcal{O}\left(\frac{\log(d)}{\lambda d\sigma_0}\right)\le\zeta_{2,k_2,l}^{(T_1)} \le 0.
		\end{align}
		\item For all $k\in[d]$, $l\in[m]$, the feature-layer weights remain unchanged, i.e.,for $l\in[m],k\in[d]$,
		\begin{align}
			\mathbf{w}^{(T_1)}_{k,l} = \mathbf{w}^{(0)}_{k,l}.
		\end{align}
	\end{enumerate}
\end{lemma}
\begin{proof}
	We first prove the first statement.
The update of matrix $\mathbf{Z}$ satisfies
\begin{equation}
	\begin{aligned}
		|\tilde{Z}^{(t+1)}_{\mathcal{I}(\mathbf{a}_i),\mathcal{I}(\mathbf{r}_1)} -\tilde{Z}^{(t)}_{\mathcal{I}(\mathbf{a}_i),\mathcal{I}(\mathbf{r}_1)}|
		=& \frac{\eta\lambda}{n\sqrt{d}}\alpha_1^{(t)}([\mathbf{a}_i,\mathbf{r}_1])\left|(\mathbf{e}_{\mathcal{I}(\mathbf{b}_i)}-\textbf{logit}^{(t)}([\mathbf{a}_i\;\;\mathbf{r}_1]))^\top\boldsymbol{\tau}^{(t)}\left(\mathbf{a}_i - \Xi^{(t)}([\mathbf{a}_i\;\; \mathbf{r}_1])\right)\right|\\
		=&\mathcal{O}\left(\frac{\eta\log(d)}{n\sqrt{d}}\right),
	\end{aligned}
\end{equation}
and 
\begin{equation}
	\begin{aligned}
		|\tilde{Z}^{(t+1)}_{\mathcal{I}(\mathbf{r}_1),\mathcal{I}(\mathbf{r}_1)} -\tilde{Z}^{(t)}_{\mathcal{I}(\mathbf{r}_1),\mathcal{I}(\mathbf{r}_1)}|
		=& \sum_{i=1}^N\frac{\eta\lambda}{n\sqrt{d}}\alpha_2^{(t)}([\mathbf{a}_i,\mathbf{r}_1])\left|(\mathbf{e}_{\mathcal{I}(\mathbf{b})}-\textbf{logit}^{(t)}([\mathbf{a}_i\;\;\mathbf{r}_1]))^\top\boldsymbol{\tau}^{(t)}\left(\mathbf{r}_1-\Xi^{(t)}([\mathbf{a}_i\;\;\mathbf{r}_1])\right)\right|\\
		=&\mathcal{O}\left(\frac{\eta\log(d)}{\sqrt{d}}\right).
	\end{aligned}
\end{equation}
Therefore, we have
\begin{align}
	|\tilde{Z}^{(T_1)}_{\mathcal{I}(\mathbf{a}_i),\mathcal{I}(\mathbf{r}_1)} -\tilde{Z}^{(0)}_{\mathcal{I}(\mathbf{a}_i),\mathcal{I}(\mathbf{r}_1)}| = \tilde{\mathcal{O}}\left(\frac{\eta}{n\sqrt{d}}T_1\right) = \mathcal{O}\left(\frac{m\log^2(d)}{\lambda^2\sqrt{d}}\right),
\end{align}
and
\begin{align}
	|\tilde{Z}^{(T_1)}_{\mathcal{I}(\mathbf{r}_1),\mathcal{I}(\mathbf{r}_1)} -\tilde{Z}^{(0)}_{\mathcal{I}(\mathbf{r}_1),\mathcal{I}(\mathbf{r}_1)}| = \tilde{\mathcal{O}}\left(\frac{\eta}{\sqrt{d}}T_1\right) = \mathcal{O}\left(\frac{mn\log^2(d)}{\lambda^2\sqrt{d}}\right).
\end{align}
By initialization, we have
\begin{align}
	|\tilde{Z}^{(0)}_{\mathcal{I}(\mathbf{a}_i),\mathcal{I}(\mathbf{r}_1)} -\tilde{Z}^{(0)}_{\mathcal{I}(\mathbf{r}_1),\mathcal{I}(\mathbf{r}_1)}| = \mathcal{O}(\sigma_0).
\end{align}
We have
\begin{equation}
\begin{aligned}
	|\tilde{Z}^{(T_1)}_{\mathcal{I}(\mathbf{a}_i),\mathcal{I}(\mathbf{r}_1)} -\tilde{Z}^{(T_1)}_{\mathcal{I}(\mathbf{r}_1),\mathcal{I}(\mathbf{r}_1)}| \le& |\tilde{Z}^{(T_1)}_{\mathcal{I}(\mathbf{a}_i),\mathcal{I}(\mathbf{r}_1)} -\tilde{Z}^{(0)}_{\mathcal{I}(\mathbf{a}_i),\mathcal{I}(\mathbf{r}_1)}| + |\tilde{Z}^{(0)}_{\mathcal{I}(\mathbf{a}_i),\mathcal{I}(\mathbf{r}_1)} -\tilde{Z}^{(0)}_{\mathcal{I}(\mathbf{r}_1),\mathcal{I}(\mathbf{r}_1)}| + 	|\tilde{Z}^{(T_1)}_{\mathcal{I}(\mathbf{r}_1),\mathcal{I}(\mathbf{r}_1)} -\tilde{Z}^{(0)}_{\mathcal{I}(\mathbf{r}_1),\mathcal{I}(\mathbf{r}_1)}|\\
	=& \mathcal{O}\left(\frac{mn\log^2(d)}{\lambda^2\sqrt{d}}\right).
\end{aligned}
\end{equation}
Therefore, we have
\begin{align}
	\frac{\alpha_1(\Xi^{(t)}([\mathbf{a}_i\;\;\mathbf{r}_1]))}{\alpha_2(\Xi^{(t)}([\mathbf{a}_i\;\;\mathbf{r}_1]))} \le \exp\left(\mathcal{O}\left(\frac{mn\log^2(d)}{\lambda^2d}\right)\right).
\end{align}
Similarly, we have
\begin{align}
	\frac{\alpha_1(\Xi^{(t)}([\mathbf{a}_i\;\;\mathbf{r}_1]))}{\alpha_2(\Xi^{(t)}([\mathbf{a}_i\;\;\mathbf{r}_1]))} \ge \exp\left(-\mathcal{O}\left(\frac{mn\log^2(d)}{\lambda^2d}\right)\right).
\end{align}
Similarly, we can prove that
\begin{align}
    |\tilde{Z}^{(T_1)}_{\mathcal{I}(\mathbf{a}'_i),\mathcal{I}(\mathbf{r}_2)} -\tilde{Z}^{(T_1)}_{\mathcal{I}(\mathbf{r}_2),\mathcal{I}(\mathbf{r}_2)}| = \mathcal{O}\left(\frac{mn\log^2(d)}{\lambda^2d}\right),
\end{align}
so that 
\begin{align}
    \frac{\alpha_1(\Xi^{(t)}([\mathbf{a}'_i\;\;\mathbf{r}_2]))}{\alpha_2(\Xi^{(t)}([\mathbf{a}'_i\;\;\mathbf{r}_2]))} \le \exp\left(\mathcal{O}\left(\frac{mn\log^2(d)}{\lambda^2d}\right)\right), \frac{\alpha_1(\Xi^{(t)}([\mathbf{a}'_i\;\;\mathbf{r}_2]))}{\alpha_2(\Xi^{(t)}([\mathbf{a}'_i\;\;\mathbf{r}_2]))} \ge \exp\left(-\mathcal{O}\left(\frac{mn\log^2(d)}{\lambda^2d}\right)\right).
\end{align}
	Next, we prove the second, third and fourth statements.
	Based on the gradient form,	the updates of $\mathbf{V}^{(t)}\mathbf{a}_i,\mathbf{V}^{(t)}\mathbf{a}'_i, \mathbf{V}^{(t)}\mathbf{r}_1,\mathbf{V}^{(t)}\mathbf{r}_2$ for all $t\in [0,T_1]$ satisfy
\begin{equation}
	\begin{aligned}
		\mathbf{V}^{(t+1)}\mathbf{a}_i =& \mathbf{V}^{(t)}\mathbf{a}_i + \frac{\lambda\eta}{nm}\sum_{k=1}^d\sum_{l=1}^{m}(\mathbb{I}(k=\mathcal{I}(\mathbf{b}_i))-\text{logit}^{(t)}_{k}(\Xi^{(t)}([\mathbf{a}_i\;\;\mathbf{r}_1])))\mathbf{w}^{(0)}_{k,l}\\
		=& \mathbf{V}^{(t)}\mathbf{a}_i + \frac{\lambda\eta}{mn}\sum_{l=1}^m(1-\text{logit}^{(t)}_{\mathcal{I}(\mathbf{b}_i)}(\Xi^{(t)}([\mathbf{a}_i\;\;\mathbf{r}_1]))) \mathbf{w}^{(0)}_{\mathcal{I}(\mathbf{b}_i),l}\\
		&-\frac{\lambda\eta}{mn}\sum_{l=1}^m\sum_{k\neq\mathcal{I}(\mathbf{b}_i)}\text{logit}_{k}^{(t)}(\Xi^{(t)}([\mathbf{a}_i\;\;\mathbf{r}_1]))\mathbf{w}_{k,l}^{(0)},
	\end{aligned}
\end{equation}
\begin{equation}
	\begin{aligned}
		\mathbf{V}^{(t+1)}\mathbf{a}'_i =& \mathbf{V}^{(t)}\mathbf{a}'_i + \frac{\lambda\eta}{nm}\sum_{k=1}^d\sum_{l=1}^{m}(\mathbb{I}(k=\mathcal{I}(\mathbf{c}_i))-\text{logit}^{(t)}_{k}(\Xi^{(t)}([\mathbf{a}'_i\;\;\mathbf{r}_2])))\mathbf{w}^{(0)}_{k,l}\\
		=& \mathbf{V}^{(t)}\mathbf{a}'_i + \frac{\lambda\eta}{mn}\sum_{l=1}^m(1-\text{logit}^{(t)}_{\mathcal{I}(\mathbf{c}_i)}(\Xi^{(t)}([\mathbf{a}'_i\;\;\mathbf{r}_2]))) \mathbf{w}^{(0)}_{\mathcal{I}(\mathbf{c}_i),l}\\
		&-\frac{\lambda\eta}{mn}\sum_{l=1}^m\sum_{k\neq\mathcal{I}(\mathbf{c}_i)}\text{logit}_{k}^{(t)}(\Xi^{(t)}([\mathbf{a}'_i\;\;\mathbf{r}_2]))\mathbf{w}_{k,l}^{(0)},
	\end{aligned}
\end{equation}
and 
\begin{equation}
	\begin{aligned}
		&\mathbf{V}^{(t+1)}\mathbf{r}_1
		= \mathbf{V}^{(t)}\mathbf{r}_1 + \frac{\lambda\eta}{nm}\sum_{i=1}^N\sum_{k=1}^d\sum_{l=1}^{m}(\mathbb{I}(k=\mathcal{I}(\mathbf{b}_i))-\text{logit}^{(t)}_{k}(\Xi^{(t)}([\mathbf{a}_i\;\;\mathbf{r}_1])))\mathbf{w}^{(0)}_{k,l},
	\end{aligned}
\end{equation}
\begin{equation}
	\begin{aligned}
		&\mathbf{V}^{(t+1)}\mathbf{r}_2
		= \mathbf{V}^{(t)}\mathbf{r}_2 + \frac{\lambda\eta}{nm}\sum_{i=1}^N\sum_{k=1}^d\sum_{l=1}^{m}(\mathbb{I}(k=\mathcal{I}(\mathbf{c}_i))-\text{logit}^{(t)}_{k}(\Xi^{(t)}([\mathbf{a}'_i\;\;\mathbf{r}_2])))\mathbf{w}^{(0)}_{k,l}.
	\end{aligned}
\end{equation}
Then, for all $t\in[0,T_1-1]$, we can simply conclude that
\begin{itemize}
	\item for all $i\in[N]$ and $l\in [m]$,
	\begin{align}
		&\gamma^{(t+1)}_{i,\mathcal{I}(\mathbf{b}_i),l} = \gamma^{(t)}_{i,\mathcal{I}(\mathbf{b}_i),l} +  \frac{\lambda\eta}{mn}(1-\text{logit}^{(t)}_{\mathcal{I}(\mathbf{b}_i)}(\Xi^{(t)}([\mathbf{a}_i\;\;\mathbf{r}_1]))) = \gamma^{(t)}_{i,\mathcal{I}(\mathbf{b}_i),l} + \Theta(\frac{\lambda\eta}{mn}),\\
        &\rho^{(t+1)}_{i,\mathcal{I}(\mathbf{c}_i),l} = \rho^{(t)}_{i,\mathcal{I}(\mathbf{c}_i),l} +  \frac{\lambda\eta}{mn}(1-\text{logit}^{(t)}_{\mathcal{I}(\mathbf{c}_i)}(\Xi^{(t)}([\mathbf{a}'_i\;\;\mathbf{r}_2]))) = \rho^{(t)}_{i,\mathcal{I}(\mathbf{c}_i),l} + \Theta(\frac{\lambda\eta}{mn}).
	\end{align}
	\item for all $i,j\in[N]$ and $j\neq i$,
	\begin{align}
		&\gamma^{(t+1)}_{i,\mathcal{I}(\mathbf{b}_j),l} = \gamma^{(t)}_{i,,\mathcal{I}(\mathbf{b}_j),l} -\frac{\lambda\eta}{mn}\text{logit}_{\mathcal{I}(\mathbf{b}_i)}^{(t)}(\Xi^{(t)}([\mathbf{a}_j\;\;\mathbf{r}_1])),\\
        &\rho^{(t+1)}_{i,\mathcal{I}(\mathbf{c}_j),l} = \rho^{(t)}_{i,\mathcal{I}(\mathbf{c}_j),l} -\frac{\lambda\eta}{mn}\text{logit}_{\mathcal{I}(\mathbf{c}_i)}^{(t)}(\Xi^{(t)}([\mathbf{a}'_j\;\;\mathbf{r}_2])).
	\end{align} 
	\item for all $i\in[N]$ and $l\in[m]$,
	\begin{align}
		&\zeta_{1,i,l}^{(t+1)} = \zeta_{1,i,l}^{(t)} + \frac{\lambda\eta}{mn}(1-\text{logit}^{(t)}_{\mathcal{I}(\mathbf{b}_i)}(\Xi^{(t)}([\mathbf{a}_i\;\;\mathbf{r}_1]))) - \frac{\lambda\eta}{mn}\sum_{k\neq i} \text{logit}^{(t)}_{\mathcal{I}(\mathbf{b}_i)}(\Xi^{(t)}([\mathbf{a}_k\;\;\mathbf{r}_1])),\\
        &\zeta_{2,i,l}^{(t+1)} = \zeta_{2,i,l}^{(t)} + \frac{\lambda\eta}{mn}(1-\text{logit}^{(t)}_{\mathcal{I}(\mathbf{c}_i)}(\Xi^{(t)}([\mathbf{a}'_i\;\;\mathbf{r}_2]))) - \frac{\lambda\eta}{mn}\sum_{k\neq i}\text{logit}^{(t)}_{\mathcal{I}(\mathbf{c}_i)}(\Xi^{(t)}([\mathbf{a}'_k\;\;\mathbf{r}_2])).
	\end{align}
\end{itemize}
As $\text{logit}_{k}(\mathbf{a}_i\mathbf{r}_1) <\frac{2}{\sqrt{d}}$ for all $k\in[d]\backslash\{\mathcal{I}(\mathbf{b}_i)\}$, we have
\begin{itemize}
	\item for all $i\in[N]$ and $l\in[m]$,
	\begin{align}
		&\gamma^{(T_1)}_{i,\mathcal{I}(\mathbf{b}_i),l} =  \Theta\left(\frac{\lambda\eta}{mn}T_1\right) = \Theta\left(\frac{\log(d)}{\lambda \sqrt{d}\sigma_0}\right),\\
        &\rho^{(T_1)}_{i,\mathcal{I}(\mathbf{c}_i),l} =  \Theta\left(\frac{\lambda\eta}{mn}T_1\right) = \Theta\left(\frac{\log(d)}{\lambda \sqrt{d}\sigma_0}\right),
	\end{align}
	\item for all $i\in[N]$ and $j_1\notin [d]\backslash\{\mathcal{I}(\mathbf{b}_i)\}_{i=1}^N, j_2 \notin [d]\backslash\{\mathcal{I}(\mathbf{c}_i)\}_{i=1}^N$,
	\begin{align}
		&-\mathcal{O}\left(\frac{\lambda\eta}{mn\sqrt{d}}T_1\right) = -\mathcal{O}\left(\frac{\log(d)}{\lambda d\sigma_0}\right) \le\gamma^{(T_1)}_{i,j_1,l} \le \gamma^{(0)}_{i,j_1,l} = 0,\\
        &-\mathcal{O}\left(\frac{\lambda\eta}{mn\sqrt{d}}T_1\right) = -\mathcal{O}\left(\frac{\log(d)}{\lambda d\sigma_0}\right) \le\rho^{(T_1)}_{i,j_2,l} \le \rho^{(0)}_{i,j_2,l} = 0,
	\end{align} 
	\item for all $i\in[N]$ and $l\in[m]$,
	\begin{align}
		&\zeta_{1,\mathcal{I}(\mathbf{b}_i),l}^{(T_1)} = \Theta\left(\frac{\eta\lambda}{mn}T_1\right) = \Theta\left(\frac{\log(d)}{\lambda\sqrt{d}\sigma_0}\right),\\
        &\zeta_{2,\mathcal{I}(\mathbf{c}_i),l}^{(T_1)} = \Theta\left(\frac{\eta\lambda}{mn}T_1\right) = \Theta\left(\frac{\log(d)}{\lambda \sqrt{d}\sigma_0}\right).
	\end{align}
    \item for all $j_1\notin [d]\backslash\{\mathcal{I}(\mathbf{b}_i)\}_{i=1}^N$ and $j_2 \notin [d]\backslash\{\mathcal{I}(\mathbf{c}_i)\}_{i=1}^N$, 
    \begin{align}
        &-\mathcal{O}\left(\frac{\lambda\eta}{mn\sqrt{d}}T_1\right) = -\mathcal{O}\left(\frac{\log(d)}{\lambda d\sigma_0}\right) \le\zeta_{1,j_1,l}^{(T_1)} \le \zeta^{(0)}_{1,j_1,l} = 0,\\
        &-\mathcal{O}\left(\frac{\lambda\eta}{mn\sqrt{d}}T_1\right) = -\mathcal{O}\left(\frac{\log(d)}{\lambda d\sigma_0}\right) \le\zeta_{2,j_2,l}^{(T_1)}\le \zeta^{(0)}_{2,j_2,l} = 0.
    \end{align}
\end{itemize}

The sixth statement holds as the feature layer is not updated.
This completes the proof.
\end{proof}

\subsection{Convergence of Stage-2 training (first on \texorpdfstring{$\mathcal{S}_1\cup\mathcal{S}_3$}{S\_1 U S\_3}, then on \texorpdfstring{$\mathcal{S}_2$}{S\_2})}
\begin{lemma}
	There exists an iteration $t \in(T_1, T_1 + T_2]$, where $T_2 = \Theta(\oldfrac{mN^2}{\lambda^2d\sigma_0^2\eta })$ such that
	\begin{align}
		\mathcal{L}_{\mathcal{S}_1\cup\mathcal{S}_3}(\mathbf{Z}^{(t)},\mathbf{V}^{(t)},\mathbf{W}^{(t)}) \le \frac{0.01}{N}.
	\end{align}
\end{lemma}
\begin{proof}
	\begin{equation}
		\mathbf{w}^*_{k,l} = 
        \mathbf{w}^{(0)}_{k,l} + 
    \begin{cases}
\oldfrac{C}{\lambda}\cdot\log(1/\epsilon)\oldfrac{\mathbf{V}^{(0)}\mathbf{a}_i}{\|\mathbf{V}^{(0)}\mathbf{a}_i\|_2^2}, & \text{ if $k=\mathcal{I}(\mathbf{b}_i)$},\\
\oldfrac{C}{\lambda}\cdot\log(1/\epsilon)\oldfrac{\mathbf{V}^{(0)}\mathbf{a}'_i}{\|\mathbf{V}^{(0)}\mathbf{a}'_i\|_2^2}, & \text{ if $k=\mathcal{I}(\mathbf{c}_i)$},\\
0, & \text{ if $k\notin\{\mathcal{I}(\mathbf{b}_i),\mathcal{I}(\mathbf{c}_i)\}_{i=1}^n$}.
\end{cases}
	\end{equation}
	 for all $k\in[d], i\in[N]$ and $l\in[m]$.
	 For simplicity, we denote $\mathcal{S} = \mathcal{S}_1\cup \mathcal{S}_3$ and $n = |\mathcal{S}|$.
	Then, we have
	\begin{equation}\label{equ: increment2}
		\begin{aligned}
			&\left\|\mathbf{W}^{(t)} - \mathbf{W}^*\right\|_F^2 - \left\|\mathbf{W}^{(t+1)} - \mathbf{W}^*\right\|_F^2\\
			=&  2\eta\langle\nabla \mathcal{L}_{\mathcal{S}}(\mathbf{Z}^{(t)},\mathbf{V}^{(t)},\mathbf{W}^{(t)}), \mathbf{W}^{(t)} - \mathbf{W}^*\rangle - \eta^2 \left\|\nabla \mathcal{L}_{\mathcal{S}}(\mathbf{Z}^{(t)},\mathbf{V}^{(t)},\mathbf{W}^{(t)})\right\|_F^2\\
			=& \frac{2\eta}{n} \sum_{(\mathbf{X},y)\in\mathcal{S}}\sum_{k=1}^d \frac{\partial \mathcal{L}}{\partial f_k} \langle\nabla f_k(\mathbf{Z}^{(t)},\mathbf{V}^{(t)},\mathbf{W}^{(t)},\mathbf{X}),\mathbf{W}^{(t)}\rangle -
			\frac{2\eta}{n} \sum_{(\mathbf{X},y)\in\mathcal{S}}\sum_{k=1}^d \frac{\partial \mathcal{L}}{\partial f_k} \underbrace{\langle \nabla f_k(\mathbf{Z}^{(t)},\mathbf{V}^{(t)},\mathbf{W}^{(t)},\mathbf{X}), \mathbf{W}^*\rangle}_{A}\\
			&- \underbrace{\eta^2 \left\|\nabla \mathcal{L}_\mathcal{S}(\mathbf{Z}^{(t)},\mathbf{V}^{(t)},\mathbf{W}^{(t)})\right\|_F^2}_{B}.
		\end{aligned}
	\end{equation}
	Next, we bound the terms $A$ and $B$.
	For the term $A$, we have
	\begin{align}\label{equ: A2}
		A 
		\begin{cases}\ge \frac{C\log(1/\epsilon)}{\lambda} - \mathcal{O}(\oldfrac{\sqrt{N}\log(d)\log(1/\epsilon)}{\lambda^2\sqrt{md}\sqrt{d\sigma_0^2}})& \text{if } k = y, \\
			\le \mathcal{O}(\oldfrac{\sqrt{N}\log(d)\log(1/\epsilon)}{\lambda^2\sqrt{md}\sqrt{d\sigma_0^2}}) & \text{if } k \neq y.
		\end{cases}
	\end{align}
	For the term $B$, we have
	\begin{equation}\label{equ: B2}
	\begin{aligned}
		B \le& \eta^2 \left[ \frac{1}{n} \sum_{(\mathbf{X},y)\in\mathcal{S}}\sum_{k\in[d]}\left|\frac{\partial \mathcal{L}(\mathbf{Z}^{(t)},\mathbf{V}^{(t)},\mathbf{W}^{(t)},\mathbf{X},y)}{\partial f_k}\right| \left\| \nabla f_k(\mathbf{Z}^{(t)},\mathbf{V}^{(t)},\mathbf{W}^{(t)},\mathbf{X})\right\|_F\right]^2\\
		\le& 2\eta^2\max_{(\mathbf{X},y)\in\mathcal{S}}\|\mathbf{V}^{(t)}\Xi^{(t)}(\mathbf{X})\|_2^2\cdot\mathcal{L}_\mathcal{S}(\mathbf{Z},\mathbf{V},\mathbf{W}).
	\end{aligned}
	\end{equation}
	Next, we need to bound $\max_{(\mathbf{X},y)\in\mathcal{S}}\|\mathbf{V}^{(t)}\Xi^{(t)}(\mathbf{X})\|_2^2$. 
	By Lemma \ref{lemma: s2}, We have
	\begin{align}
		\max_{(\mathbf{X},y)\in\mathcal{S}}\|\mathbf{V}^{(t)}\Xi^{(t)}(\mathbf{X})\|_2^2 = \max_{(\mathbf{X},y)\in\mathcal{S}}\|\mathbf{V}^{(T_1)}\Xi^{(T_1)}(\mathbf{X})\|_2^2 = \mathcal{O}(Nmd\sigma_0^2\log^2(d)/\lambda^2).
	\end{align}
	
	Combining (\ref{equ: A2}), (\ref{equ: B2}), and (\ref{equ: increment2}), with $\eta = \Theta(\oldfrac{\lambda^2}{mNd\sigma_0^2\log^2(d)})$, we have
	\begin{equation}\label{equ: descent2}
		\begin{aligned}
			&\left\|\mathbf{W}^{(t)} - \mathbf{W}^*\right\|_F^2 - \left\|\mathbf{W}^{(t+1)} - \mathbf{W}^*\right\|_F^2\\
			=&  \frac{2\eta}{n} \sum_{(\mathbf{X},y)\in\mathcal{S}}\left(\sum_{k=1}^d \frac{\partial \mathcal{L}}{\partial f_k} \langle\nabla f_k(\mathbf{Z}^{(t)},\mathbf{V}^{(t)},\mathbf{W}^{(t)},\mathbf{X}),\mathbf{W}^{(t)}\rangle - \sum_{k=1}^d \frac{\partial \mathcal{L}}{\partial f_k} \langle \nabla f_k(\mathbf{Z}^{(t)},\mathbf{V}^{(t)},\mathbf{W}^{(t)},\mathbf{X}), \mathbf{W}^*\rangle\right)\\
			&- \eta^2 \left\|\nabla \mathcal{L}_\mathcal{S}(\mathbf{Z}^{(t)},\mathbf{V}^{(t)},\mathbf{W}^{(t)})\right\|_F^2\\
			\overset{(a)}{\ge}& \frac{2\eta}{n}\sum_{(\mathbf{X},y)\in\mathcal{S}}\left(\mathcal{L}(\mathbf{Z}^{(t)},\mathbf{V}^{(t)},\mathbf{W}^{(t)},\mathbf{X}_i,y_i)-\log(1/\epsilon)\right) - \eta\mathcal{L}_\mathcal{S}(\mathbf{Z}^{(t)},\mathbf{V}^{(t)},\mathbf{W}^{(t)})\\
			=&\eta \mathcal{L}_\mathcal{S}(\mathbf{Z}^{(t)},\mathbf{V}^{(t)},\mathbf{W}^{(t)}) - 2\eta \epsilon,
		\end{aligned}
	\end{equation}
	where $(a)$ is by the homogeneity of the feature layer and convexity of the cross-entropy function.
	
	Rearranging (\ref{equ: descent2}), we have
	\begin{align}
		\frac{1}{T_2+1}\sum_{t'=0}^{T_2}\mathcal{L}_{\mathcal{S}_1\cup\mathcal
        S_3}(\mathbf{Z}^{(T_1+t')},\mathbf{V}^{(T_1+t')},\mathbf{W}^{(T_1+t')}) \le \frac{\left\|\mathbf{W}^{(0)} - \mathbf{W}^*\right\|_F^2}{\eta (T_2+1)} + 2\epsilon.
	\end{align}
	Here, we have
	\begin{align}
		\left\|\mathbf{W}^{(0)} - \mathbf{W}^*\right\|_F^2 = \mathcal{O}(\frac{mN\log^2(1/\epsilon)}{\lambda^2 d\sigma_0^2})
	\end{align}
	As a result, we have
	\begin{align}
		\frac{1}{T_2+1}\sum_{t'=0}^{T_2}\mathcal{L}_{\mathcal{S}_1\cup\mathcal
        S_3}(\mathbf{Z}^{(T_1 + t')},\mathbf{V}^{(T_1+ t')},\mathbf{W}^{(T_1+t')}) \le \frac{\left\|\mathbf{W}^{(0)} - \mathbf{W}^*\right\|_F^2}{\eta (T_2+1)} + 2\epsilon.
	\end{align}
    This finishes the proof.
\end{proof}

\subsection{Convergence of Stage-3 training (first on \texorpdfstring{$\mathcal{S}_1\cup\mathcal{S}_3$}{S\_1 U S\_3}, then on \texorpdfstring{$\mathcal{S}_2$}{S\_2})}
\begin{lemma}
	There exists an iteration $t \in(T_1+T_2, T_1 + T_2 + T_3]$, where $T_3 = \Theta(\oldfrac{mN^2}{\lambda^2d\sigma_0^2\eta })$ such that
	\begin{align}
		\mathcal{L}_{\mathcal{S}_2}(\mathbf{Z}^{(t)},\mathbf{V}^{(t)},\mathbf{W}^{(t)}) \le \frac{0.01}{N}.
	\end{align}
\end{lemma}
\begin{proof}
	Let 
	\begin{equation}
		\mathbf{w}^*_{k,l} = 
        \mathbf{w}^{(0)}_{k,l} + 
    \begin{cases}
\frac{C}{\lambda}\cdot\log(1/\epsilon)\mathbf{V}^{(0)}\mathbf{a}'_i/\left\|\mathbf{V}^{(0)}\mathbf{a}'_i\right\|_2^2, & \text{ if $k=\mathcal{I}(\mathbf{b}_i)$},\\
0, & \text{ if $k \neq \mathcal{I}(\mathbf{b}_i), \forall i\in[N]$}.
\end{cases}
	\end{equation}
	 for all $k\in[d], i\in[N]$ and $l\in[m]$.
	 For simplicity, we denote $\mathcal{S} = \mathcal{S}_2$ and $n = |\mathcal{S}|$.
	Then, we have
	\begin{equation}\label{equ: increment}
		\begin{aligned}
			&\left\|\mathbf{W}^{(t)} - \mathbf{W}^*\right\|_F^2 - \left\|\mathbf{W}^{(t+1)} - \mathbf{W}^*\right\|_F^2\\
			=&  2\eta\langle\nabla \mathcal{L}_{\mathcal{S}}(\mathbf{Z}^{(t)},\mathbf{V}^{(t)},\mathbf{W}^{(t)}), \mathbf{W}^{(t)} - \mathbf{W}^*\rangle - \eta^2 \left\|\nabla \mathcal{L}_{\mathcal{S}}(\mathbf{Z}^{(t)},\mathbf{V}^{(t)},\mathbf{W}^{(t)})\right\|_F^2\\
			=& \frac{2\eta}{n} \sum_{(\mathbf{X},y)\in\mathcal{S}}\sum_{k=1}^d \frac{\partial \mathcal{L}}{\partial f_k} \langle\nabla f_k(\mathbf{Z}^{(t)},\mathbf{V}^{(t)},\mathbf{W}^{(t)},\mathbf{X}),\mathbf{W}^{(t)}\rangle -
			\frac{2\eta}{n} \sum_{(\mathbf{X},y)\in\mathcal{S}}\sum_{k=1}^d \frac{\partial \mathcal{L}}{\partial f_k} \underbrace{\langle \nabla f_k(\mathbf{Z}^{(t)},\mathbf{V}^{(t)},\mathbf{W}^{(t)},\mathbf{X}), \mathbf{W}^*\rangle}_{A}\\
			&- \underbrace{\eta^2 \left\|\nabla \mathcal{L}_\mathcal{S}(\mathbf{Z}^{(t)},\mathbf{V}^{(t)},\mathbf{W}^{(t)})\right\|_F^2}_{B}.
		\end{aligned}
	\end{equation}
	Next, we bound the terms $A$ and $B$.
	For the term $A$, we have
	\begin{align}\label{equ: A}
		A 
		\begin{cases}\ge \frac{C\log(1/\epsilon)}{\lambda} - \mathcal{O}(\oldfrac{\sqrt{N}\log(d)\log(1/\epsilon)}{\lambda^2\sqrt{md}\sqrt{d\sigma_0^2}})& \text{if } k = y, \\
			\le \mathcal{O}(\oldfrac{\sqrt{N}\log(d)\log(1/\epsilon)}{\lambda^2\sqrt{md}\sqrt{d\sigma_0^2}}) & \text{if } k \neq y.
		\end{cases}
	\end{align}
	For the term $B$, we have
	\begin{equation}\label{equ: B}
	\begin{aligned}
		B \le& \eta^2 \left[ \frac{1}{n} \sum_{(\mathbf{X},y)\in\mathcal{S}}\sum_{k\in[d]}\left|\frac{\partial \mathcal{L}(\mathbf{Z}^{(t)},\mathbf{V}^{(t)},\mathbf{W}^{(t)},\mathbf{X},y)}{\partial f_k}\right| \left\| \nabla f_k(\mathbf{Z}^{(t)},\mathbf{V}^{(t)},\mathbf{W}^{(t)},\mathbf{X})\right\|_F\right]^2\\
		\le& 2\eta^2\max_{(\mathbf{X},y)\in\mathcal{S}}\|\mathbf{V}^{(t)}\Xi^{(t)}(\mathbf{X})\|_2^2\cdot\mathcal{L}_\mathcal{S}(\mathbf{Z},\mathbf{V},\mathbf{W}).
	\end{aligned}
	\end{equation}
	Next, we need to bound $\max_{(\mathbf{X},y)\in\mathcal{S}}\|\mathbf{V}^{(t)}\Xi^{(t)}(\mathbf{X})\|_2^2$. 
	By Lemma \ref{lemma: s2}, we have
	\begin{align}
		\max_{(\mathbf{X},y)\in\mathcal{S}}\|\mathbf{V}^{(t)}\Xi^{(t)}(\mathbf{X})\|_2^2 = \max_{(\mathbf{X},y)\in\mathcal{S}}\|\mathbf{V}^{(T_1)}\Xi^{(T_1)}(\mathbf{X})\|_2^2 = \mathcal{O}(Nd\sigma_0^2\log^2(d)/\lambda^2).
	\end{align}
	
	Combining (\ref{equ: A}), (\ref{equ: B}), and (\ref{equ: increment}), with $\eta = \Theta(\frac{\lambda^2}{Nd\sigma_0^2\log^2(d)})$, we have
	\begin{equation}\label{equ: descent3}
		\begin{aligned}
			&\left\|\mathbf{W}^{(t)} - \mathbf{W}^*\right\|_F^2 - \left\|\mathbf{W}^{(t+1)} - \mathbf{W}^*\right\|_F^2\\
			=&  \frac{2\eta}{n} \sum_{(\mathbf{X},y)\in\mathcal{S}}\left(\sum_{k=1}^d \frac{\partial \mathcal{L}}{\partial f_k} \langle\nabla f_k(\mathbf{Z}^{(t)},\mathbf{V}^{(t)},\mathbf{W}^{(t)},\mathbf{X}),\mathbf{W}^{(t)}\rangle - \sum_{k=1}^d \frac{\partial \mathcal{L}}{\partial f_k} \langle \nabla f_k(\mathbf{Z}^{(t)},\mathbf{V}^{(t)},\mathbf{W}^{(t)},\mathbf{X}), \mathbf{W}^*\rangle\right)\\
			&- \eta^2 \left\|\nabla \mathcal{L}_\mathcal{S}(\mathbf{Z}^{(t)},\mathbf{V}^{(t)},\mathbf{W}^{(t)})\right\|_F^2\\
			\overset{(a)}{\ge}& \frac{2\eta}{n}\sum_{(\mathbf{X},y)\in\mathcal{S}}\left(\mathcal{L}(\mathbf{Z}^{(t)},\mathbf{V}^{(t)},\mathbf{W}^{(t)},\mathbf{X}_i,y_i)-\log(1/\epsilon)\right) - \eta\mathcal{L}_\mathcal{S}(\mathbf{Z}^{(t)},\mathbf{V}^{(t)},\mathbf{W}^{(t)})\\
			=&\eta \mathcal{L}_\mathcal{S}(\mathbf{Z}^{(t)},\mathbf{V}^{(t)},\mathbf{W}^{(t)}) - 2\eta \epsilon,
		\end{aligned}
	\end{equation}
	where $(a)$ is by the homogeneity of the feature layer and convexity of the cross-entropy function.
	
	Rearranging (\ref{equ: descent3}), we have
	\begin{align}
		\frac{1}{T_3+1}\sum_{t'=0}^{T_3}\mathcal{L}_{\mathcal{S}_2}(\mathbf{Z}^{(T_1+T_2+t')},\mathbf{V}^{(T_1+T_2+t')},\mathbf{W}^{(T_1+T_2+t')}) \le \frac{\left\|\mathbf{W}^{(0)} - \mathbf{W}^*\right\|_F^2}{\eta (T_3+1)} + 2\epsilon.
	\end{align}
	Here, we have
	\begin{align}
		\left\|\mathbf{W}^{(0)} - \mathbf{W}^*\right\|_F^2 = \mathcal{O}(\frac{mN\log^2(1/\epsilon)}{\lambda^2d\sigma_0^2}).
	\end{align}
	As a result, we have
	\begin{align}
		\frac{1}{T_3+1}\sum_{t'=0}^{T_3}\mathcal{L}_{\mathcal{S}_2}(\mathbf{Z}^{(T_1+T_2+t')},\mathbf{V}^{(T_1+T_2+t')},\mathbf{W}^{(T_1+T_2+t')}) \le 3\epsilon.
	\end{align}
    Letting $\epsilon = 0.003/N$ finishes the proof.
\end{proof}
\subsection{Feature similarity (first on \texorpdfstring{$\mathcal{S}_1\cup\mathcal{S}_3$}{S\_1 U S\_3}, then on \texorpdfstring{$\mathcal{S}_2$}{S\_2})}
\begin{proposition}
    After the Stage-1 training, for all $i\in[N]$, we have
    $$\frac{\langle\mathbf{V}^{(T_1)}\mathbf{a}_i, \mathbf{V}^{(T_1)}\mathbf{a}'_i\rangle}{\|\mathbf{V}^{(T_1)}\mathbf{a}_i\|_2\|\mathbf{V}^{(T_1)}\mathbf{a}'_i\|_2} = o(1).$$
\end{proposition}
\begin{proof}
    Based on Lemma \ref{lemma: s2}, $\mathbf{V}^{(T_1)}\mathbf{a}$ and $\mathbf{V}^{(T_1)}\mathbf{a}'$ have large coefficients on two parts of nearly orthogonal components $\mathbf{w}^{(0)}_{\mathcal{I}(\mathbf{b}_i),l}$ and $\mathbf{w}^{(0)}_{\mathcal{I}(\mathbf{c}_i),l}$ for all $l\in[m]$.
    As a result, we have
    \begin{align}
        \left|\langle\mathbf{V}^{(T_1)}\mathbf{a}_i, \mathbf{V}^{(T_1)}\mathbf{a}'_i\rangle\right| = \tilde{\mathcal{O}}(m\frac{\log^2(d)}{\sqrt{d}\lambda^2} + \frac{1}{\sqrt{d}})
    \end{align}
    Therefore, we have 
    \begin{align}       \frac{|\langle\mathbf{V}^{(T_1)}\mathbf{a}_i, \mathbf{V}^{(T_1)}\mathbf{a}'_i\rangle|}{\|\mathbf{V}^{(T_1)}\mathbf{a}_i\|_2\|\mathbf{V}^{(T_1)}\mathbf{a}'_i\|_2} = \frac{\tilde{\mathcal{O}}(m\frac{\log^2(d)}{\sqrt{d}\lambda^2} + \frac{1}{\sqrt{d}})}{\Theta(m\log^2(d)/\lambda^2 + d\sigma_0^2)} = o(1).
    \end{align}
    This finishes the proof.
\end{proof}
\subsection{Proof of success rate bound (first on \texorpdfstring{$\mathcal{S}_1\cup\mathcal{S}_3$}{S\_1 U S\_3}, then on \texorpdfstring{$\mathcal{S}_2$}{S\_2})}
As at initialization and during training, all the embedding distributions of $i,j \in [N]$ and the model initialization are symmetric.
Then, for all $(\mathbf{X},y)\in\mathcal{A}$ and $t\in[T_1+T_2+T_3]$, we have
\begin{align}
    \mathbb{P}[f_{\mathcal{I}(\mathbf{c}_i)}(\mathbf{Z}^{(t)},\mathbf{V}^{(t)},\mathbf{W}^{(t)},\mathbf{X})>\max_{j\neq \mathcal{I}(\mathbf{c}_i)} f_j(\mathbf{Z}^{(t)},\mathbf{V}^{(t)},\mathbf{W}^{(t)},\mathbf{X}) ] = \frac{1}{N}.
\end{align}
Therefore, we have
\begin{align}
    \mathcal{L}^{0-1}_\mathcal{A}(\mathbf{Z}^{(t)},\mathbf{V}^{(t)},\mathbf{W}^{(t)}) = 1-\frac{1}{N}.
\end{align}
This finishes the proof.

\section{Proofs of two-hop reasoning with identity bridges}

The correspondence of two-hop reasoning construction with the analogical-reasoning construction is
\[
\mathcal{S}_1 \leftrightarrow \mathcal{H}_1,\qquad
\mathcal{S}_2 \leftrightarrow \mathcal{IB},\qquad
\mathcal{S}_3 \leftrightarrow \mathcal{H}_2.
\]
Equivalently,
\[
(\mathbf{a}_i,\mathbf{r}_1,\mathbf{b}_i),\quad (\mathbf{a}'_i,\mathbf{r}_1,\mathbf{b}_i),\quad (\mathbf{a}'_i,\mathbf{r}_2,\mathbf{c}_i)
\]
is replaced by
\[
(\mathbf{a}_i,\mathbf{r}_1,\mathbf{b}_i),\quad (\mathbf{b}_i,\mathbf{r}_3,\mathbf{b}_i),\quad (\mathbf{b}_i,\mathbf{r}_2,\mathbf{c}_i).
\]
The only notational difference is that the second premise uses the identity relation $\mathbf{r}_3$ instead of $\mathbf{r}_1$. 
The proofs of two-hop reasoning with the identity bridge are similar to those of joint training. 
We omit it here.

\section{Proofs of two-hop reasoning without the identity bridge}
We assume that Lemmas 3 and 5 hold with probability $1-2\delta'$ by union bound. We have the following lemma.
\begin{lemma}\label{lemma: decomposition2}
	For all $t\in[T]$, $i\in[N]$, $k\in[d]$, $l,h\in[m], q\in[2]$, the model updates $\mathbf{w}_{k,l}^{(t)}$ and $\mathbf{V}^{(t)}\mathbf{a}_i$ are linear combinations of initialization vectors, i.e., there exists coefficients $\gamma^{(t)}_{i,k,l},\rho^{(t)}_{i,k,l}, \zeta^{(t)}_{q,k,l},\beta^{(t)}_{1,k,l,q,h},\beta^{(t)}_{2,k,l,j}, \beta_{3,k,l,j}^{(t)},\beta^{(t)}_{4,k,l,q}\in\mathbb{R}$ such that
	\begin{align}
		\mathbf{V}^{(t)}\mathbf{a}_i =& \sum_{k=1}^d\sum_{l=1}^{m}\gamma^{(t)}_{i,k,l}\mathbf{w}^{(0)}_{k,l} + \mathbf{V}^{(0)}\mathbf{a}_i,\\
		\mathbf{V}^{(t)}\mathbf{b}_i =& \sum_{k=1}^d\sum_{l=1}^{m}\rho^{(t)}_{i,k,l}\mathbf{w}^{(0)}_{k,l} + \mathbf{V}^{(0)}\mathbf{b}_i,\\
		\mathbf{V}^{(t)}\mathbf{r}_q =& \sum_{k=1}^d\sum_{l=1}^{m}
		\zeta_{q,k,l}^{(t)}\mathbf{w}^{(0)}_{k,l} +  \mathbf{V}^{(0)}\mathbf{r}_q,\\
		\mathbf{w}_{k,l}^{(t)} =& \sum_{q=1}^d\sum_{h=1}^{m}\beta^{(t)}_{1,k,l,q,h}\mathbf{w}^{(0)}_{q,h} + \sum_{j=1}^{N}\beta^{(t)}_{2,k,l,j}\mathbf{V}^{(0)}\mathbf{a}_j + \sum_{j=1}^{N}\beta^{(t)}_{3,k,l,j}\mathbf{V}^{(0)}\mathbf{b}_j +\sum_{q\in[2]}\beta^{(t)}_{4,k,l,q}\mathbf{V}^{(0)}\mathbf{r}_q.
	\end{align}
\end{lemma}
This lemma holds as the gradient updates are linear combinations of these randomly initialized vectors.
In the setting without identity bridges, the training set is $\mathcal H_1\cup \mathcal H_2$, and hence only the relation tokens $\mathbf{r}_1$ and $\mathbf{r}_2$ appear in training. The identity relation $\mathbf{r}_3$ is not used and is omitted from the following coefficient decompositions.
\subsection{Properties during Stage-1 training (two-hop training w/o identity bridge)}
By the choice of $T_1$, Stage 1 stops before the logits saturate. Hence, we have that $1-\text{logit}_y(\mathbf{X}) = \Theta(1)$ holds for all $(\mathbf{X},y)\in\mathcal{H}_1\cup\mathcal
{H}_2$. The following properties hold.
\begin{lemma}\label{lemma: h1}
	For any $t\in[0, T_1]$ with $T_1 = \Theta(\oldfrac{mn\log(d)}{\lambda^2\eta\sqrt{d}\sigma_0})$, we have
	\begin{enumerate}
		\item For all $(\mathbf{X},y)\in \mathcal{H}_1\cup\mathcal{H}_2$, the attention scores are balanced, i.e.,
		\begin{align}
			0.7\le \frac{\alpha_1(\mathbf{Z}^{(t)},\mathbf{X})}{\alpha_2(\mathbf{Z}^{(t)},\mathbf{X})} \le 1.5.
		\end{align}
		\item For all $i\in[N]$, $l_1, l_2 \in [m]$, the coefficients satisfy
		\begin{align}
			\gamma^{(T_1)}_{i,\mathcal{I}(\mathbf{b}_i),l_1} = \Theta\left(\frac{\log(d)}{\lambda \sqrt{d}\sigma_0}\right),
			\rho^{(T_1)}_{i,\mathcal{I}(\mathbf{c}_i),l_2} = \Theta\left(\frac{\log(d)}{\lambda \sqrt{d}\sigma_0}\right).
		\end{align}
		\item For all $i\in[N], j_1\in[d]\backslash\{\mathcal{I}(\mathbf{b}_i)\}_{i=1}^N, j_2 \in [d]\backslash\{\mathcal{I}(\mathbf{c}_i)\}_{i=1}^N$ and $l\in [m]$, the coefficients satisfy
		\begin{align}
			-\mathcal{O}\left(\frac{\log(d)}{\lambda d\sigma_0}\right)\le \gamma^{(T_1)}_{i,j_1,l} \le 0, -\mathcal{O}\left(\frac{\log(d)}{\lambda d\sigma_0}\right)\le \rho^{(T_1)}_{i,j_2,l} \le 0.
		\end{align}
		\item For all $k\in[N]$, $l\in[m]$ the coefficients satisfy
		\begin{align}
			\zeta_{1,\mathcal{I}(\mathbf{b}_k),l}^{(T_1)} = \Theta\left(\frac{\log(d)}{\lambda \sqrt{d}\sigma_0}\right),\zeta_{2,\mathcal{I}(\mathbf{c}_k),l}^{(T_1)} = \Theta\left(\frac{\log(d)}{\lambda \sqrt{d}\sigma_0}\right).
		\end{align}
		\item For all $k_1\in[d]\backslash\{\mathcal{I}(\mathbf{b}_i)\}_{i=1}^N$, $k_2\in[d]\backslash\{\mathcal{I}(\mathbf{c}_i)\}_{i=1}^N$, and $l\in[m]$, the coefficients satisfy
		\begin{align}
			-\mathcal{O}\left(\frac{\log(d)}{\lambda d\sigma_0}\right)\le\zeta_{1,k_1,l}^{(T_1)} \le 0, -\mathcal{O}\left(\frac{\log(d)}{\lambda d\sigma_0}\right)\le\zeta_{2,k_2,l}^{(T_1)} \le 0.
		\end{align}
		\item For all $k\in[d]$, $l\in[m]$, the feature-layer weights remain unchanged, i.e.,
		\begin{align}
			\mathbf{w}^{(T_1)}_{k,l} = \mathbf{w}^{(0)}_{k,l}.
		\end{align}
	\end{enumerate}
\end{lemma}
\begin{proof}
	We first prove the first statement.
The update of matrix $\mathbf{Z}$ satisfies
\begin{equation}
	\begin{aligned}
		|\tilde{Z}^{(t+1)}_{\mathcal{I}(\mathbf{a}_i),\mathcal{I}(\mathbf{r}_1)} -\tilde{Z}^{(t)}_{\mathcal{I}(\mathbf{a}_i),\mathcal{I}(\mathbf{r}_1)}|
		=& \frac{\eta\lambda}{n\sqrt{d}}\alpha_1^{(t)}([\mathbf{a}_i,\mathbf{r}_1])\left|(\mathbf{e}_{\mathcal{I}(\mathbf{b}_i)}-\textbf{logit}^{(t)}([\mathbf{a}_i\;\;\mathbf{r}_1]))^\top\boldsymbol{\tau}^{(t)}\left(\mathbf{a}_i - \Xi^{(t)}([\mathbf{a}_i\;\; \mathbf{r}_1])\right)\right|\\
		=&\mathcal{O}\left(\frac{\eta\log(d)}{n\sqrt{d}}\right),
	\end{aligned}
\end{equation}
and 
\begin{equation}
	\begin{aligned}
		|\tilde{Z}^{(t+1)}_{\mathcal{I}(\mathbf{r}_1),\mathcal{I}(\mathbf{r}_1)} -\tilde{Z}^{(t)}_{\mathcal{I}(\mathbf{r}_1),\mathcal{I}(\mathbf{r}_1)}|
		=& \sum_{i=1}^N\frac{\eta\lambda}{n\sqrt{d}}\alpha_2^{(t)}([\mathbf{a}_i,\mathbf{r}_1])\left|(\mathbf{e}_{\mathcal{I}(\mathbf{b})}-\textbf{logit}^{(t)}([\mathbf{a}_i\;\;\mathbf{r}_1]))^\top\boldsymbol{\tau}^{(t)}\left(\mathbf{r}_1-\Xi^{(t)}([\mathbf{a}_i\;\;\mathbf{r}_1])\right)\right|\\
		=&\mathcal{O}\left(\frac{\eta\log(d)}{\sqrt{d}}\right).
	\end{aligned}
\end{equation}
Therefore, we have
\begin{align}
	|\tilde{Z}^{(T_1)}_{\mathcal{I}(\mathbf{a}_i),\mathcal{I}(\mathbf{r}_1)} -\tilde{Z}^{(0)}_{\mathcal{I}(\mathbf{a}_i),\mathcal{I}(\mathbf{r}_1)}| = \mathcal{O}\left(\frac{\eta\log(d)}{n\sqrt{d}}T_1\right) = \mathcal{O}\left(\frac{m\log^2(d)}{\lambda^2\sqrt{d}}\right),
\end{align}
and
\begin{align}
	|\tilde{Z}^{(T_1)}_{\mathcal{I}(\mathbf{r}_1),\mathcal{I}(\mathbf{r}_1)} -\tilde{Z}^{(0)}_{\mathcal{I}(\mathbf{r}_1),\mathcal{I}(\mathbf{r}_1)}| = \mathcal{O}\left(\frac{\eta\log(d)}{m\sqrt{d}}T_1\right) = \mathcal{O}\left(\frac{mn\log^2(d)}{\lambda^2\sqrt{d}}\right).
\end{align}
By initialization, we have
\begin{align}
	|\tilde{Z}^{(0)}_{\mathcal{I}(\mathbf{a}_i),\mathcal{I}(\mathbf{r}_1)} -\tilde{Z}^{(0)}_{\mathcal{I}(\mathbf{r}_1),\mathcal{I}(\mathbf{r}_1)}| = \mathcal{O}(\sigma_0).
\end{align}
We have
\begin{equation}
\begin{aligned}
	|\tilde{Z}^{(T_1)}_{\mathcal{I}(\mathbf{a}_i),\mathcal{I}(\mathbf{r}_1)} -\tilde{Z}^{(T_1)}_{\mathcal{I}(\mathbf{r}_1),\mathcal{I}(\mathbf{r}_1)}| \le& |\tilde{Z}^{(T_1)}_{\mathcal{I}(\mathbf{a}_i),\mathcal{I}(\mathbf{r}_1)} -\tilde{Z}^{(0)}_{\mathcal{I}(\mathbf{a}_i),\mathcal{I}(\mathbf{r}_1)}| + |\tilde{Z}^{(0)}_{\mathcal{I}(\mathbf{a}_i),\mathcal{I}(\mathbf{r}_1)} -\tilde{Z}^{(0)}_{\mathcal{I}(\mathbf{r}_1),\mathcal{I}(\mathbf{r}_1)}| + 	|\tilde{Z}^{(T_1)}_{\mathcal{I}(\mathbf{r}_1),\mathcal{I}(\mathbf{r}_1)} -\tilde{Z}^{(0)}_{\mathcal{I}(\mathbf{r}_1),\mathcal{I}(\mathbf{r}_1)}|\\
	=& \mathcal{O}\left(\frac{mn\log^2(d)}{\lambda^2\sqrt{d}}\right).
\end{aligned}
\end{equation}
Therefore, we have
\begin{align}
	\frac{\alpha_1(\Xi^{(t)}([\mathbf{a}_i\;\;\mathbf{r}_1]))}{\alpha_2(\Xi^{(t)}([\mathbf{a}_i\;\;\mathbf{r}_1]))} \le 1.2.
\end{align}
Similarly, we have
\begin{align}
	\frac{\alpha_1(\Xi^{(t)}([\mathbf{a}_i\;\;\mathbf{r}_1]))}{\alpha_2(\Xi^{(t)}([\mathbf{a}_i\;\;\mathbf{r}_1]))} \ge 0.8.
\end{align}
Similarly, we can prove that
\begin{align}
    |\tilde{Z}^{(T_1)}_{\mathcal{I}(\mathbf{b}_i),\mathcal{I}(\mathbf{r}_2)} -\tilde{Z}^{(T_1)}_{\mathcal{I}(\mathbf{r}_2),\mathcal{I}(\mathbf{r}_2)}| = \mathcal{O}\left(\frac{mn\log^2(d)}{\lambda^2\sqrt{d}}\right),
\end{align}
so that 
\begin{align}
    \frac{\alpha_1(\Xi^{(t)}([\mathbf{b}_i\;\;\mathbf{r}_2]))}{\alpha_2(\Xi^{(t)}([\mathbf{b}_i\;\;\mathbf{r}_2]))} \le 1.2, \frac{\alpha_1(\Xi^{(t)}([\mathbf{b}_i\;\;\mathbf{r}_2]))}{\alpha_2(\Xi^{(t)}([\mathbf{b}_i\;\;\mathbf{r}_2]))} \ge 0.8.
\end{align}
	Next, we prove the second, third and fourth statements.
	Based on the gradient form,	the updates of $\mathbf{V}^{(t)}\mathbf{a}_i,\mathbf{V}^{(t)}\mathbf{b}_i, \mathbf{V}^{(t)}\mathbf{r}_1,\mathbf{V}^{(t)}\mathbf{r}_2$ for all $t\in [0,T_1]$ satisfy
\begin{equation}
	\begin{aligned}
		\mathbf{V}^{(t+1)}\mathbf{a}_i =& \mathbf{V}^{(t)}\mathbf{a}_i + \frac{\lambda\eta}{nm}\sum_{k=1}^d\sum_{l=1}^{m}(\mathbb{I}(k=\mathcal{I}(\mathbf{b}_i))-\text{logit}^{(t)}_{k}(\Xi^{(t)}([\mathbf{a}_i\;\;\mathbf{r}_1])))\mathbf{w}^{(0)}_{k,l}\\
		=& \mathbf{V}^{(t)}\mathbf{a}_i + \frac{\lambda\eta}{mn}\sum_{l=1}^m(1-\text{logit}^{(t)}_{\mathcal{I}(\mathbf{b}_i)}(\Xi^{(t)}([\mathbf{a}_i\;\;\mathbf{r}_1]))) \mathbf{w}^{(0)}_{\mathcal{I}(\mathbf{b}_i),l}\\
		&-\frac{\lambda\eta}{mn}\sum_{l=1}^m\sum_{k\neq\mathcal{I}(\mathbf{b}_i)}\text{logit}_{k}^{(t)}(\Xi^{(t)}([\mathbf{a}_i\;\;\mathbf{r}_1]))\mathbf{w}_{k,l}^{(0)},
	\end{aligned}
\end{equation}
\begin{equation}
	\begin{aligned}
		\mathbf{V}^{(t+1)}\mathbf{b}_i =& \mathbf{V}^{(t)}\mathbf{b}_i + \frac{\lambda\eta}{nm}\sum_{k=1}^d\sum_{l=1}^{m}(\mathbb{I}(k=\mathcal{I}(\mathbf{c}_i))-\text{logit}^{(t)}_{k}(\Xi^{(t)}([\mathbf{b}_i\;\;\mathbf{r}_2])))\mathbf{w}^{(0)}_{k,l}\\
		=& \mathbf{V}^{(t)}\mathbf{b}_i + \frac{\lambda\eta}{mn}\sum_{l=1}^m(1-\text{logit}^{(t)}_{\mathcal{I}(\mathbf{c}_i)}(\Xi^{(t)}([\mathbf{b}_i\;\;\mathbf{r}_2]))) \mathbf{w}^{(0)}_{\mathcal{I}(\mathbf{c}_i),l}\\
		&-\frac{\lambda\eta}{mn}\sum_{l=1}^m\sum_{k\neq\mathcal{I}(\mathbf{c}_i)}\text{logit}_{k}^{(t)}(\Xi^{(t)}([\mathbf{b}_i\;\;\mathbf{r}_2]))\mathbf{w}_{k,l}^{(0)},
	\end{aligned}
\end{equation}
and 
\begin{equation}
	\begin{aligned}
		&\mathbf{V}^{(t+1)}\mathbf{r}_1
		= \mathbf{V}^{(t)}\mathbf{r}_1 + \frac{\lambda\eta}{nm}\sum_{i=1}^N\sum_{k=1}^d\sum_{l=1}^{m}(\mathbb{I}(k=\mathcal{I}(\mathbf{b}_i))-\text{logit}^{(t)}_{k}(\Xi^{(t)}([\mathbf{a}_i\;\;\mathbf{r}_1])))\mathbf{w}^{(0)}_{k,l},
	\end{aligned}
\end{equation}
\begin{equation}
	\begin{aligned}
		&\mathbf{V}^{(t+1)}\mathbf{r}_2
		= \mathbf{V}^{(t)}\mathbf{r}_2 + \frac{\lambda\eta}{nm}\sum_{i=1}^N\sum_{k=1}^d\sum_{l=1}^{m}(\mathbb{I}(k=\mathcal{I}(\mathbf{c}_i))-\text{logit}^{(t)}_{k}(\Xi^{(t)}([\mathbf{b}_i\;\;\mathbf{r}_2])))\mathbf{w}^{(0)}_{k,l}.
	\end{aligned}
\end{equation}
Then, for all $t\in[0,T_1-1]$, we can simply conclude that
\begin{itemize}
	\item for all $i\in[N]$ and $l\in [m]$,
	\begin{align}
		&\gamma^{(t+1)}_{i,\mathcal{I}(\mathbf{b}_i),l} = \gamma^{(t)}_{i,\mathcal{I}(\mathbf{b}_i),l} +  \frac{\lambda\eta}{mn}(1-\text{logit}^{(t)}_{\mathcal{I}(\mathbf{b}_i)}(\Xi^{(t)}([\mathbf{a}_i\;\;\mathbf{r}_1]))) = \gamma^{(t)}_{i,\mathcal{I}(\mathbf{b}_i),l} + \Theta(\frac{\lambda\eta}{mn}),\\
        &\rho^{(t+1)}_{i,\mathcal{I}(\mathbf{c}_i),l} = \rho^{(t)}_{i,\mathcal{I}(\mathbf{c}_i),l} +  \frac{\lambda\eta}{mn}(1-\text{logit}^{(t)}_{\mathcal{I}(\mathbf{c}_i)}(\Xi^{(t)}([\mathbf{b}_i\;\;\mathbf{r}_2]))) = \rho^{(t)}_{i,\mathcal{I}(\mathbf{c}_i),l} + \Theta(\frac{\lambda\eta}{mn}).
	\end{align}
	\item for all $i,j\in[N]$ and $j\neq i$,
	\begin{align}
		&\gamma^{(t+1)}_{i,\mathcal{I}(\mathbf{b}_j),l} = \gamma^{(t)}_{i,,\mathcal{I}(\mathbf{b}_j),l} -\frac{\lambda\eta}{mn}\text{logit}_{\mathcal{I}(\mathbf{b}_i)}^{(t)}(\Xi^{(t)}([\mathbf{a}_j\;\;\mathbf{r}_1])),\\
        &\rho^{(t+1)}_{i,\mathcal{I}(\mathbf{c}_j),l} = \rho^{(t)}_{i,\mathcal{I}(\mathbf{c}_j),l} -\frac{\lambda\eta}{mn}\text{logit}_{\mathcal{I}(\mathbf{c}_i)}^{(t)}(\Xi^{(t)}(\Xi^{(t)}([\mathbf{b}_j\;\;\mathbf{r}_2]))).
	\end{align} 
	\item for all $i\in[N]$ and $l\in[m]$,
	\begin{align}
		&\zeta_{1,i,l}^{(t+1)} = \zeta_{1,i,l}^{(t)} + \frac{\lambda\eta}{mn}(1-\text{logit}^{(t)}_{\mathcal{I}(\mathbf{b}_i)}(\Xi^{(t)}([\mathbf{a}_i\;\;\mathbf{r}_1]))) - \frac{\lambda\eta}{mn}\sum_{k\neq i} \text{logit}^{(t)}_{\mathcal{I}(\mathbf{b}_i)}(\Xi^{(t)}([\mathbf{a}_k\;\;\mathbf{r}_1])),\\
        &\zeta_{2,i,l}^{(t+1)} = \zeta_{2,i,l}^{(t)} + \frac{\lambda\eta}{mn}(1-\text{logit}^{(t)}_{\mathcal{I}(\mathbf{c}_i)}(\Xi^{(t)}([\mathbf{b}_i\;\;\mathbf{r}_2]))) - \frac{\lambda\eta}{mn}\sum_{k\neq i}\text{logit}^{(t)}_{\mathcal{I}(\mathbf{c}_i)}(\Xi^{(t)}([\mathbf{b}_k\;\;\mathbf{r}_2])).
	\end{align}
\end{itemize}
As $\text{logit}_{k}(\mathbf{a}_i\mathbf{r}_1) <\frac{2}{\sqrt{d}}$ for all $k\in[d]\backslash\{\mathcal{I}(\mathbf{b}_i)\}$ due to the fact that $\text{logit}_{k}(\Xi^{(t)}([\mathbf{a}_i\;\;\mathbf{r}_1]))\le \text{logit}_{\mathcal{I}(\mathbf{b}_i)}(\Xi^{(t)}([\mathbf{a}_i\;\;\mathbf{r}_1]))/2$, we have
\begin{itemize}
	\item for all $i\in[N]$ and $l\in[m]$,
	\begin{align}
		&\gamma^{(T_1)}_{i,\mathcal{I}(\mathbf{b}_i),l} =  \Theta\left(\frac{\lambda\eta}{mn}T_1\right) = \Theta\left(\frac{\log(d)}{\lambda\sqrt{d}\sigma_0}\right),\\
        &\rho^{(T_1)}_{i,\mathcal{I}(\mathbf{c}_i),l} =  \Theta\left(\frac{\lambda\eta}{mn}T_1\right) = \Theta\left(\frac{\log(d)}{\lambda\sqrt{d}\sigma_0}\right),
	\end{align}
	\item for all $i\in[N]$ and $j_1\notin [d]\backslash\{\mathcal{I}(\mathbf{b}_i)\}_{i=1}^N, j_2 \notin [d]\backslash\{\mathcal{I}(\mathbf{c}_i)\}_{i=1}^N$,
	\begin{align}
		&-\mathcal{O}\left(\frac{\lambda\eta}{mn\sqrt{d}}T_1\right) = -\mathcal{O}\left(\frac{\log(d)}{\lambda d\sigma_0}\right) \le\gamma^{(T_1)}_{i,j_1,l} \le \gamma^{(0)}_{i,j_1,l} = 0,\\
        &-\mathcal{O}\left(\frac{\lambda\eta}{mn\sqrt{d}}T_1\right) = -\mathcal{O}\left(\frac{\log(d)}{\lambda d\sigma_0}\right) \le\rho^{(T_1)}_{i,j_2,l} \le \rho^{(0)}_{i,j_2,l} = 0,
	\end{align} 
	\item for all $i\in[N]$ and $l\in[m]$,
	\begin{align}
		&\zeta_{1,\mathcal{I}(\mathbf{b}_i),l}^{(T_1)} = \Theta\left(\frac{\eta\lambda}{mn}T_1\right) = \Theta\left(\frac{\log(d)}{\lambda\sqrt{d}\sigma_0}\right),\\
        &\zeta_{2,\mathcal{I}(\mathbf{c}_i),l}^{(T_1)} = \Theta\left(\frac{\eta\lambda}{mn}T_1\right) = \Theta\left(\frac{\log(d)}{\lambda\sqrt{d}\sigma_0}\right).
	\end{align}
    \item for all $j_1\notin [d]\backslash\{\mathcal{I}(\mathbf{b}_i)\}_{i=1}^N$ and $j_2 \notin [d]\backslash\{\mathcal{I}(\mathbf{c}_i)\}_{i=1}^N$, 
    \begin{align}
        &-\mathcal{O}\left(\frac{\lambda\eta}{mn\sqrt{d}}T_1\right) = -\mathcal{O}\left(\frac{\log(d)}{\lambda d\sigma_0}\right) \le\zeta_{1,j_1,l}^{(T_1)} \le \zeta^{(0)}_{1,j_1,l} = 0,\\
        &-\mathcal{O}\left(\frac{\lambda\eta}{mn\sqrt{d}}T_1\right) = -\mathcal{O}\left(\frac{\log(d)}{\lambda d\sigma_0}\right) \le\zeta_{2,j_2,l}^{(T_1)}\le \zeta^{(0)}_{2,j_2,l} = 0.
    \end{align}
\end{itemize}

The sixth statement holds as the feature layer is not updated.

This completes the proof.
\end{proof}

\subsection{Convergence of Stage-2 training (two-hop training w/o identity bridge)}
\begin{lemma}
	There exists an iteration $t \in(T_1, T_1 + T_2]$, where $T_2 = \Theta(\oldfrac{mN^2}{\lambda^2d\sigma_0^2\eta })$ such that
    \begin{align}\mathcal{L}_{\mathcal{H}_1\cup\mathcal{H}_2}(\mathbf{Z}^{(t)},\mathbf{V}^{(t)},\mathbf{W}^{(t)}) \le \frac{0.01}{N}.
    \end{align}
\end{lemma}
\begin{proof}
	Let 
	\begin{equation}
		\mathbf{w}^*_{k,l} = 
        \mathbf{w}^{(0)}_{k,l} + 
    \begin{cases}
\oldfrac{C}{\lambda}\cdot\log(1/\epsilon)\oldfrac{\mathbf{V}^{(0)}\mathbf{a}_i}{\|\mathbf{V}^{(0)}\mathbf{a}_i\|_2^2}, & \text{ if $k=\mathcal{I}(\mathbf{b}_i)$},\\
\oldfrac{C}{\lambda}\cdot\log(1/\epsilon)\oldfrac{\mathbf{V}^{(0)}\mathbf{b}_i}{\|\mathbf{V}^{(0)}\mathbf{b}_i\|_2^2}, & \text{ if $k=\mathcal{I}(\mathbf{c}_i)$},\\
0, & \text{ if $k\notin\{\mathcal{I}(\mathbf{b}_i),\mathcal{I}(\mathbf{c}_i)\}_{i=1}^n$}.
\end{cases}
	\end{equation}
	 for all $k\in[d], i\in[N]$ and $l\in[m]$.
	 For simplicity, we denote $\mathcal{S} = \mathcal{H}_1\cup \mathcal{H}_2$ and $n = |\mathcal{S}|$.
	Then, we have
	\begin{equation}\label{equ: increment5}
		\begin{aligned}
			&\left\|\mathbf{W}^{(t)} - \mathbf{W}^*\right\|_F^2 - \left\|\mathbf{W}^{(t+1)} - \mathbf{W}^*\right\|_F^2\\
			=&  2\eta\langle\nabla \mathcal{L}_{\mathcal{S}}(\mathbf{Z}^{(t)},\mathbf{V}^{(t)},\mathbf{W}^{(t)}), \mathbf{W}^{(t)} - \mathbf{W}^*\rangle - \eta^2 \left\|\nabla \mathcal{L}_{\mathcal{S}}(\mathbf{Z}^{(t)},\mathbf{V}^{(t)},\mathbf{W}^{(t)})\right\|_F^2\\
			=& \frac{2\eta}{n} \sum_{(\mathbf{X},y)\in\mathcal{S}}\sum_{k=1}^d \frac{\partial \mathcal{L}}{\partial f_k} \langle\nabla f_k(\mathbf{Z}^{(t)},\mathbf{V}^{(t)},\mathbf{W}^{(t)},\mathbf{X}),\mathbf{W}^{(t)}\rangle -
			\frac{2\eta}{n} \sum_{(\mathbf{X},y)\in\mathcal{S}}\sum_{k=1}^d \frac{\partial \mathcal{L}}{\partial f_k} \underbrace{\langle \nabla f_k(\mathbf{Z}^{(t)},\mathbf{V}^{(t)},\mathbf{W}^{(t)},\mathbf{X}), \mathbf{W}^*\rangle}_{A}\\
			&- \underbrace{\eta^2 \left\|\nabla \mathcal{L}_\mathcal{S}(\mathbf{Z}^{(t)},\mathbf{V}^{(t)},\mathbf{W}^{(t)})\right\|_F^2}_{B}.
		\end{aligned}
	\end{equation}
	Next, we bound the terms $A$ and $B$.
	For the term $A$, we have
\begin{align}\label{equ: A5}
		A 
		\begin{cases}\ge \frac{C\log(1/\epsilon)}{\lambda} - \mathcal{O}(\oldfrac{\sqrt{N}\log(d)\log(1/\epsilon)}{\lambda^2\sqrt{md}\sqrt{d\sigma_0^2}})& \text{if } k = y, \\
			\le \mathcal{O}(\oldfrac{\sqrt{N}\log(d)\log(1/\epsilon)}{\lambda^2\sqrt{md}\sqrt{d\sigma_0^2}}) & \text{if } k \neq y.
		\end{cases}
	\end{align}
	For the term $B$, we have
	\begin{equation}\label{equ: B5}
	\begin{aligned}
		B \le& \eta^2 \left[ \frac{1}{n} \sum_{(\mathbf{X},y)\in\mathcal{S}}\sum_{k\in[d]}\left|\frac{\partial \mathcal{L}(\mathbf{Z}^{(t)},\mathbf{V}^{(t)},\mathbf{W}^{(t)},\mathbf{X},y)}{\partial f_k}\right| \left\| \nabla f_k(\mathbf{Z}^{(t)},\mathbf{V}^{(t)},\mathbf{W}^{(t)},\mathbf{X})\right\|_F\right]^2\\
		\le& 2\eta^2\max_{(\mathbf{X},y)\in\mathcal{S}}\|\mathbf{V}^{(t)}\Xi^{(t)}(\mathbf{X})\|_2^2\cdot\mathcal{L}_\mathcal{S}(\mathbf{Z},\mathbf{V},\mathbf{W}).
	\end{aligned}
	\end{equation}
	Next, we need to bound $\max_{(\mathbf{X},y)\in\mathcal{S}}\|\mathbf{V}^{(t)}\Xi^{(t)}(\mathbf{X})\|_2^2$. 
	By Lemma \ref{lemma: h1}, We have
	\begin{align}
		\max_{(\mathbf{X},y)\in\mathcal{S}}\|\mathbf{V}^{(t)}\Xi^{(t)}(\mathbf{X})\|_2^2 = \max_{(\mathbf{X},y)\in\mathcal{S}}\|\mathbf{V}^{(T_1)}\Xi^{(T_1)}(\mathbf{X})\|_2^2 = \mathcal{O}(Nmd\sigma_0^2\log^2(d)/\lambda^2).
	\end{align}
	
	Combining (\ref{equ: A5}), (\ref{equ: B5}), and (\ref{equ: increment5}), with $\eta = \Theta(\oldfrac{\lambda^2}{Nmd\sigma_0^2\log^2(d)})$, we have
	\begin{equation}\label{equ: descent6}
		\begin{aligned}
			&\left\|\mathbf{W}^{(t)} - \mathbf{W}^*\right\|_F^2 - \left\|\mathbf{W}^{(t+1)} - \mathbf{W}^*\right\|_F^2\\
			=&  \frac{2\eta}{n} \sum_{(\mathbf{X},y)\in\mathcal{S}}\left(\sum_{k=1}^d \frac{\partial \mathcal{L}}{\partial f_k} \langle\nabla f_k(\mathbf{Z}^{(t)},\mathbf{V}^{(t)},\mathbf{W}^{(t)},\mathbf{X}),\mathbf{W}^{(t)}\rangle - \sum_{k=1}^d \frac{\partial \mathcal{L}}{\partial f_k} \langle \nabla f_k(\mathbf{Z}^{(t)},\mathbf{V}^{(t)},\mathbf{W}^{(t)},\mathbf{X}), \mathbf{W}^*\rangle\right)\\
			&- \eta^2 \left\|\nabla \mathcal{L}_\mathcal{S}(\mathbf{Z}^{(t)},\mathbf{V}^{(t)},\mathbf{W}^{(t)})\right\|_F^2\\
			\overset{(a)}{\ge}& \frac{2\eta}{n}\sum_{(\mathbf{X},y)\in\mathcal{S}}\left(\mathcal{L}(\mathbf{Z}^{(t)},\mathbf{V}^{(t)},\mathbf{W}^{(t)},\mathbf{X}_i,y_i)-\log(1/\epsilon)\right) - \eta\mathcal{L}_\mathcal{S}(\mathbf{Z}^{(t)},\mathbf{V}^{(t)},\mathbf{W}^{(t)})\\
			=&\eta \mathcal{L}_\mathcal{S}(\mathbf{Z}^{(t)},\mathbf{V}^{(t)},\mathbf{W}^{(t)}) - 2\eta \epsilon,
		\end{aligned}
	\end{equation}
	where $(a)$ is by the homogeneity of the feature layer and convexity of the cross-entropy function.
	
	Rearranging (\ref{equ: descent6}), we have
	\begin{align}
		\frac{1}{T_2+1}\sum_{t'=0}^{T_2}\mathcal{L}_{\mathcal{H}_1\cup\mathcal{H}_2}(\mathbf{Z}^{(T_1+t')},\mathbf{V}^{(T_1+t')},\mathbf{W}^{(T_1+t')}) \le \frac{\left\|\mathbf{W}^{(0)} - \mathbf{W}^*\right\|_F^2}{\eta (T_2+1)} + 2\epsilon.
	\end{align}
	Here, we have
	\begin{align}
		\left\|\mathbf{W}^{(0)} - \mathbf{W}^*\right\|_F^2 = \mathcal{O}(\frac{mN\log^2(1/\epsilon)}{\lambda^2d\sigma_0^2}).
	\end{align}
	As a result, we have
	\begin{align}
		\frac{1}{T_2+1}\sum_{t=0}^{T_2}\mathcal{L}_{\mathcal{H}_1\cup\mathcal{H}_2}(\mathbf{Z}^{(T_1+t')},\mathbf{V}^{(T_1+t')},\mathbf{W}^{(T_1+t')}) \le 3\epsilon.
	\end{align}
    Letting $\epsilon = 0.003/N$ finishes the proof.
\end{proof}
\subsection{Feature similarity (two-hop training w/o identity bridge)}
\begin{proposition}
    After the Stage-1 training, for all $i\in[N]$, we have
    $$\frac{\langle\mathbf{V}^{(T_1)}\mathbf{a}_i, \mathbf{V}^{(T_1)}\mathbf{b}_i\rangle}{\|\mathbf{V}^{(T_1)}\mathbf{a}_i\|_2\|\mathbf{V}^{(T_1)}\mathbf{b}_i\|_2}= o(1).$$
\end{proposition}
\begin{proof}
    Based on Lemma \ref{lemma: h1}, $\mathbf{V}^{(T_1)}\mathbf{a}_i$ and $\mathbf{V}^{(T_1)}\mathbf{b}_i$ have large coefficients on two parts of nearly orthogonal components $\mathbf{w}^{(0)}_{\mathcal{I}(\mathbf{b}_i),l}$ and $\mathbf{w}^{(0)}_{\mathcal{I}(\mathbf{c}_i),l}$ for all $l\in[m]$.
    As a result, we have
    \begin{align}
        |\langle\mathbf{V}^{(T_1)}\mathbf{a}_i, \mathbf{V}^{(T_1)}\mathbf{b}_i\rangle| = \tilde{\mathcal{O}}(m\frac{\log^2(d)}{\lambda^2\sqrt{d}}+\frac{1}{\sqrt{d}})
    \end{align}
    Therefore, we have 
    \begin{align}
        \frac{|\langle\mathbf{V}^{(T_1)}\mathbf{a}_i, \mathbf{V}^{(T_1)}\mathbf{b}_i\rangle|}{\|\mathbf{V}^{(T_1)}\mathbf{a}_i\|_2\|\mathbf{V}^{(T_1)}\mathbf{b}_i\|_2}  =\frac{\tilde{\mathcal{O}}(m\frac{\log^2(d)}{\lambda^2\sqrt{d}}+\frac{1}{\sqrt{d}})}{\Theta(m\log^2(d)/\lambda^2+d\sigma^2_0)} = o(1). 
    \end{align}
    This finishes the proof.
\end{proof}

\subsection{Success rate bound (two-hop training without identity bridge)}
As at initialization and during training, all $i,j \in [N]$ are symmetric.
Then, for all $(\mathbf{X},y)\in\mathcal{R}$ and $t\in[T_1+T_2]$, we have
\begin{align}
    \mathbb{P}[f_{\mathcal{I}(\mathbf{c}_i)}(\mathbf{Z}^{(t)},\mathbf{V}^{(t)},\mathbf{W}^{(t)},\mathbf{X})>\max_{j\neq \mathcal{I}(\mathbf{c}_i)} f_j(\mathbf{Z}^{(t)},\mathbf{V}^{(t)},\mathbf{W}^{(t)},\mathbf{X}) ] = \frac{1}{N}.
\end{align}
Therefore,
\begin{align}
    \mathcal{L}^{0-1}_\mathcal{R}(\mathbf{Z}^{(t)},\mathbf{V}^{(t)},\mathbf{W}^{(t)}) = 1-\frac{1}{N}.
\end{align}
This finishes the proof.

\section{Probability bookkeeping}
In the above proofs for theorems, we assume some properties, such as Lemmas 3,4,5 hold with high probability. 
By union bound, letting $\delta = C' \delta'$ with a large constant $C'$, we can claim that the theorems hold with probability at least $1-\delta$.

\section{Proof of deep linear neural networks}\label{app: linear_nn}
For simplicity, we denote 
\begin{align}
& T = \sum_{i=1}^LT_i,\\
& s_i = \sum_{k=1}^iT_i, s_0:=0,\\
    &\textbf{logit}^{(t)}(\mathbf{x}) = \text{softmax}(\mathbf{W}_L^{(t)}\cdots\mathbf{W}_{1}^{(t)}\mathbf{x}),\\
    &\mathbf{A}^{(t)}_j = \mathbf{W}^{(t)}_{j} \dots \mathbf{W}^{(t)}_1.
\end{align}
For the data, we have the following lemma,
\begin{lemma}
    Suppose that $r\in\mathbb{R}$ is a constant.
    For any $i\in[n], j\in[d]$, with probability $1-\delta$, we have
    \begin{align}
        x_{i,j} = \mathcal{O}(\frac{\log(2nd/\delta)}{\sqrt{d}}).
    \end{align}
\end{lemma}
This lemma is proved using the sub-Gaussian property of the variables sampled from unit sphere \cite{vershynin2009high} and union bound. 

For $i\in[L]$ and all $(\mathbf{x},y)\in\mathcal{S}_o$, the gradient of the loss with a linear neural network is given as 
    \begin{align}
        \frac{\partial \mathcal{L}}{\partial \mathbf{W}_i} = \nabla_{\mathbf{W}_i} \mathcal{L} = \frac{1}{n}(\mathbf{W}_L \dots \mathbf{W}_{i+1})^\top\sum_{(\mathbf{x},y)\in\mathcal{S}_o}(\text{softmax}(\mathbf{W}_L\cdots\mathbf{W}_{1}\mathbf{x})-\mathbf{e}_y)(\mathbf{W}_{i-1} \dots \mathbf{W}_1 \mathbf{x})^\top.
    \end{align}

When $T_1,\cdots,T_{L-1} = \mathcal{O}(\oldfrac{n}{L\eta})$, for all $t\in[T]$, $i\in[n]$ and $j\neq y_i$, we have
\begin{align}
    1-\text{logit}^{(t)}_{y_i}(\mathbf{x}_i) \ge& \frac{1}{2},\\
    \text{logit}^{(t)}_j(\mathbf{x}_i) \le& \frac{1}{d^{1/3}}.
\end{align}

During stage $i$, the model update for data sample $(\mathbf{x}_j,y_j)\in\mathcal{S}_o$ satisfies
\begin{align}\label{equ: update_multi}
    &\mathbf{W}^{(t+1)}_i\mathbf{A}_{i-1}^{(s_{i-1})}\mathbf{x}_j - \mathbf{W}_i^{(t)}\mathbf{A}_{i-1}^{(s_{i-1})}\mathbf{x}_j = \frac{\eta}{n} \sum_{k=1}^n(\mathbf{e}_{y_k}-\textbf{logit}^{(t)}(\mathbf{x}_k))\underbrace{\mathbf{x}_k^\top(\mathbf{A}^{(s_{i-1})}_{i-1})^\top\mathbf{A}^{(s_{i-1})}_{i-1}\mathbf{x}_j}_{D}.
\end{align}
For the term $D$, we have the following lemma.
\begin{lemma}\label{lemma: term_D}
Let $s_l = \sum_{q=1}^lT_q$. For all $l\in[1,L-1]$, at time $s_l$, we have
\begin{itemize}
    \item For any data,
    \begin{align}
    &1\le\mathbf{x}_i^\top(\mathbf{A}^{(s_l)}_l)^\top\mathbf{A}^{(s_l)}_l\mathbf{x}_i \le 2,
    \end{align}
    \item For any data with $y_j\neq y_i$,
    \begin{align}
    &|\mathbf{x}_j^\top(\mathbf{A}^{(s_l)}_l)^\top\mathbf{A}^{(s_l)}_l\mathbf{x}_i| \le \tilde{\mathcal{O}}(\frac{1}{d^{1/2}}),
    \end{align}
    \item For any data with $y_j = y_i$ and $j\neq i$,
    \begin{align}
    \mathbf{x}_j^\top(\mathbf{A}^{(s_l)}_l)^\top\mathbf{A}^{(s_l)}_l\mathbf{x}_i = \Theta(\frac{l^2}{L^2})\pm \tilde{\mathcal{O}}(\frac{1}{\sqrt{d}}).
    \end{align}
\end{itemize}
\end{lemma}
\begin{proof}
We prove it by induction.
We first prove that the conclusion holds for the first layer.
    For the first layer training $t\in[T_1]$, we have
    \begin{align}
        \frac{\eta}{2n}T_1\le (\mathbf{W}^{(T_1)}_1\mathbf{x}_i-\mathbf{W}^{(0)}_1\mathbf{x}_i)_{y_i} \le \frac{\eta}{n}T_1,
    \end{align}
    leading to $(\mathbf{W}^{(T_1)}_l\mathbf{x}_i-\mathbf{W}^{(0)}_l\mathbf{x}_i)_{y_i} = \Theta(\frac{1}{L})$.
    Then, we have
        \begin{align}
         \mathbf{x}_i^\top(\mathbf{A}^{(T_1)}_1)^\top\mathbf{A}^{(T_1)}_1\mathbf{x}_i = 1+ \underbrace{(\mathbf{W}^{(T_1)}_1\mathbf{x}_i-\mathbf{W}^{(0)}_1\mathbf{x}_i)^\top(\mathbf{W}^{(T_1)}_1\mathbf{x}_i-\mathbf{W}^{(0)}_1\mathbf{x}_i)}_{\Theta(\frac{1}{L^2})} +  \underbrace{2(\mathbf{W}^{(T_1)}_1\mathbf{x}_i - \mathbf{W}^{(0)}_1\mathbf{x}_i)^\top \mathbf{x}_i}_{\mathcal{O}(\frac{\log(2nd/\delta)}{\sqrt{d}})},
    \end{align}
    leading to 
    \begin{align}
        1\le\mathbf{x}_i^\top(\mathbf{A}^{(T_1)}_1)^\top\mathbf{A}^{(T_1)}_1\mathbf{x}_i\le 1.2.
    \end{align}
    For $y_j\neq y_i$, we have
    \begin{align}
        |\mathbf{x}_j^\top(\mathbf{A}^{(T_1)}_1)^\top\mathbf{A}^{(T_1)}_1\mathbf{x}_i| \le \tilde{\mathcal{O}}(\frac{1}{Ld^{1/2}}).
    \end{align}
For any data with $y_i=y_j$ and $i\neq j$, we have
\begin{equation}
\begin{aligned}
\mathbf{x}_j^\top(\mathbf{A}^{(T_1)}_{1})^\top\mathbf{A}^{(T_1)}_{1}\mathbf{x}_i  =& \underbrace{(\mathbf{A}_{1}^{(T_1)}\mathbf{x}_i-\mathbf{x}_i)^\top(\mathbf{A}_{1}^{(T_1)}\mathbf{x}_j-\mathbf{x}_j)}_{\Theta(\frac{1}{L^2})} +  \underbrace{(\mathbf{A}_{1}^{(T_1)}\mathbf{x}_i-\mathbf{x}_i)^\top \mathbf{x}_j}_{\mathcal{O}(\frac{\log(2nd/\delta)}{\sqrt{d}})})\\
&+ \underbrace{(\mathbf{A}_{1}^{(T_1)}\mathbf{x}_j-\mathbf{x}_j)^\top \mathbf{x}_i}_{\mathcal{O}(\frac{\log(2nd/\delta)}{\sqrt{d}})})\\
=& \Theta(\frac{1}{L^2}) \pm \tilde{\mathcal{O}}(\frac{1}{\sqrt{d}}).
\end{aligned}
\end{equation}    
    
    Assume that the conclusions hold for layer $l-1$. 
    For layer $l$, we have
    \begin{equation}
    \begin{aligned}
(\mathbf{W}^{(s_l)}_l\mathbf{A}_{l-1}^{(s_l)}\mathbf{x}_i-\mathbf{W}^{(s_{l-1})}_l\mathbf{A}_{l-1}^{(s_{l-1})}\mathbf{x}_i)_{y_i} =& \left(\frac{\eta}{n}\sum_{t'=s_{l-1}}^{s_l-1} \sum_{k=1}^n(\mathbf{e}_{y_k}-\textbf{logit}^{(t')}(\mathbf{x}_k))\mathbf{x}_k^\top(\mathbf{A}^{(s_{l-1})}_{l-1})^\top\mathbf{A}^{(s_{l-1})}_{l-1}\mathbf{x}_i\right)_{y_i}\\
=& \Theta(\frac{1}{L}).
    \end{aligned}
    \end{equation}
As a result, we have
\begin{align}
     (\mathbf{A}_{l}^{(s_l)}\mathbf{x}_i-\mathbf{x}_i)_{y_i} = \Theta(l/L).
\end{align}
For any data $(\mathbf{x}_i,y_i)\in\mathcal{S}_o$, we have
    \begin{align}
\mathbf{x}_i^\top(\mathbf{A}^{(s_{l})}_{l})^\top\mathbf{A}^{(s_{l})}_{l}\mathbf{x}_i  = 1+ \underbrace{(\mathbf{A}_{l}^{(s_l)}\mathbf{x}_i-\mathbf{x}_i)^\top(\mathbf{A}_{l}^{(s_l)}\mathbf{x}_i-\mathbf{x}_i)}_{\Theta(\frac{l^2}{L^2})} +  \underbrace{2(\mathbf{A}_{l}^{(s_l)}\mathbf{x}_i-\mathbf{x}_i)^\top \mathbf{x}_i}_{\mathcal{O}(\frac{\log(2nd/\delta)}{\sqrt{d}})}) \le 2.
    \end{align}
For any data with $y_i=y_j$ and $i\neq j$, we have
\begin{equation}
\begin{aligned}
\mathbf{x}_j^\top(\mathbf{A}^{(s_{l})}_{l})^\top\mathbf{A}^{(s_{l})}_{l}\mathbf{x}_i  =& \underbrace{(\mathbf{A}_{l}^{(s_l)}\mathbf{x}_i-\mathbf{x}_i)^\top(\mathbf{A}_{l}^{(s_l)}\mathbf{x}_j-\mathbf{x}_j)}_{\Theta(\frac{l^2}{L^2})} +  \underbrace{(\mathbf{A}_{l}^{(s_l)}\mathbf{x}_i-\mathbf{x}_i)^\top \mathbf{x}_j}_{\mathcal{O}(\frac{\log(2nd/\delta)}{\sqrt{d}})})\\
&+ \underbrace{(\mathbf{A}_{l}^{(s_l)}\mathbf{x}_j-\mathbf{x}_j)^\top \mathbf{x}_i}_{\mathcal{O}(\frac{\log(2nd/\delta)}{\sqrt{d}})})\\
=& \Theta(l^2/L^2) \pm \tilde{\mathcal{O}}(\frac{1}{\sqrt{d}}).
\end{aligned}
\end{equation}    
In addition, for $y_j\neq y_i$, we have
\begin{equation}
    \begin{aligned}
        |\mathbf{x}_j^\top(\mathbf{A}^{(s_l)}_l)^\top\mathbf{A}^{(s_l)}_l\mathbf{x}_i| \le& \underbrace{|(\mathbf{A}^{(s_l)}_l\mathbf{x}_j-\mathbf{x}_j)^\top(\mathbf{A}^{(s_l)}_l\mathbf{x}_i-\mathbf{x}_i)|}_{\tilde{\mathcal{O}}(\oldfrac{l^2}{L^2d^{1/2}})} +  \underbrace{|(\mathbf{A}^{(s_l)}_l\mathbf{x}_i - \mathbf{x}_i)^\top \mathbf{x}_j|}_{\mathcal{O}(\frac{\log(2nd/\delta)}{\sqrt{d}})}\\
        &+\underbrace{|(\mathbf{A}^{(s_l)}_l\mathbf{x}_j - \mathbf{x}_j)^\top \mathbf{x}_i|}_{\mathcal{O}(\frac{\log(2nd/\delta)}{\sqrt{d}})}\\
        =& \mathcal{O}\left(\frac{\log(2nd/\delta)}{d^{1/2}}\right).
    \end{aligned}
\end{equation}
\end{proof}

    By (\ref{equ: update_multi}), for $(\mathbf{x}_i,y_i)\in\mathcal{S}_o,l\in[1,L-1]$ and $k\neq y_i$, we have
    \begin{align}  \left(\mathbf{A}_{l+1}^{(s_{l+1})}\mathbf{x}_i-\mathbf{A}_{l}^{(s_l)}\mathbf{x}_i\right)_{y_i} = \Theta\left(\frac{1}{L}\right), \left|\left(\mathbf{A}^{(s_{l+1})}_{l+1}\mathbf{x}_i - \mathbf{A}^{(s_l)}_{l+1}\mathbf{x}_i\right)_{k}\right| = \mathcal{O}\left(\frac{\log(2nd/\delta)}{d^{1/2}}\right).
    \end{align}
As $\mathbf{B}^{(0)}_1 = \cdots = \mathbf{B}^{(0)}_L = \mathbf{I}$, based on Lemma \ref{lemma: term_D}, for $l\in [1,L-1]$, we have
\begin{align}
    \mathbf{W}_l^{(s_l)}\cdots\mathbf{W}_1^{(s_l)}\mathbf{x}_i = \mathbf{x}_i+\Theta\left(\frac{l}{L}\right)\mathbf{e}_{y_i} + \sum_{j\neq y_i}r_{l,j,i}\mathbf{e}_{j}.
\end{align}
where $|r_{l,j,i}| = \mathcal{O}\left(l\frac{\log(2nd/\delta)}{d^{1/2}}\right)$ for all $j\neq y_i$.
By the condition that $d^{0.9} > C'L^2\log(2nd/\delta)$ for a large constant $C'$, we have 
 \begin{equation}
    \begin{aligned}
        &\frac{\langle \mathbf{W}^{(T)}_{k+1}\cdots\mathbf{W}^{(T)}_1\mathbf{x}_i,\mathbf{W}^{(T)}_{k+1}\cdots\mathbf{W}^{(T)}_1\mathbf{x}_j \rangle}{\| \mathbf{W}^{(T)}_{k+1}\cdots\mathbf{W}^{(T)}_1\mathbf{x}_i\|_2\| \mathbf{W}^{(T)}_{k+1}\cdots\mathbf{W}^{(T)}_1\mathbf{x}_j\|_2} \ge \frac{\langle \mathbf{W}^{(T)}_{k}\cdots\mathbf{W}^{(T)}_1\mathbf{x}_i,\mathbf{W}^{(T)}_{k}\cdots\mathbf{W}^{(T)}_1\mathbf{x}_j \rangle}{\| \mathbf{W}^{(T)}_{k}\cdots\mathbf{W}^{(T)}_1\mathbf{x}_i\|_2\| \mathbf{W}^{(T)}_{k}\cdots\mathbf{W}^{(T)}_1\mathbf{x}_j\|_2} - \tilde{\mathcal{O}}(d^{-0.1}), 
    \end{aligned}
    \end{equation}
    and 
    \begin{align}
        \frac{\langle \mathbf{W}^{(T)}_{L-1}\cdots\mathbf{W}^{(T)}_1\mathbf{x}_i,\mathbf{W}^{(T)}_{L-1}\cdots\mathbf{W}^{(T)}_1\mathbf{x}_j \rangle}{\| \mathbf{W}^{(T)}_{L-1}\cdots\mathbf{W}^{(T)}_1\mathbf{x}_i\|_2\| \mathbf{W}^{(T)}_{L-1}\cdots\mathbf{W}^{(T)}_1\mathbf{x}_j\|_2} =\Omega(1),
    \end{align}
This finishes the proof.
The term $\tilde{\mathcal{O}}(d^{-0.1})$ is due to the randomness of random orthogonal embeddings.
\section{Experimental details}\label{app: exp}
\subsection{Synthetic data experiments}
For one-layer transformers, we set
\begin{itemize}
    \item Number of entity tuples $N = 100$.
    \item Embedding dimension $d = 427$.
    \item Scaling factor $\lambda = 20$.
    \item Training iterations $T_1 = 500$, $T_2 = 2000$ and $T_3 = 2000$ (for sequential training).
    \item Width of the feature layer $m=50$.
    \item Initialization magnitude $\sigma_0 = 0.03$.
\end{itemize}

For GPT-2, we set
\begin{itemize}
    \item Number of entity tuples $N = 100$.
    \item Embedding dimension $d = 200$.
    \item Training iterations: $T = 500$ for joint training and $T_1 = T_2 = 600$ for sequential training.
    \item Number of heads is 12 and number of layers is 12.
    \item Initialization magnitude: $\sigma_0 = 0.003$.
\end{itemize}

\subsection{Natural language data experiments}\label{app: exp_nlp}
We fine-tune for 10 epochs with per-device batch size 16, learning rate
\(10^{-5}\), zero weight decay, and mixed-precision FP16 training.

We show the representative examples of the generated natural language dataset in Table \ref{tab:extended_data}.
\begin{table}[h]
    \centering
    \caption{Dataset Sample Examples}
    \label{tab:extended_data}
    \small
    \setlength{\tabcolsep}{4pt}
    \begin{tabular}{l l l l l}
        \toprule
        \textbf{Subject 1} & \textbf{Subject 2} & \textbf{Aspect} & \textbf{Relation} & \textbf{Category} \\
        \midrule
        
        Apple & Pear & Tree & AtLocation & Fruit \\
        Ant & Bee & Small & HasProperty & Insect \\
        Eagle & Hawk & Hunting & CapableOf & Bird of Prey \\
        Bed & Sofa & Resting & UsedFor & Furniture \\
        Boots & Sneakers & Outfit & PartOf & Footwear \\
        
        \bottomrule
    \end{tabular}
\end{table}

The generated data has six possible relations in total, such as ``AtLocation'', ``MadeOf'', ``UsedFor'', ``CapableOf'', ``HasProperty'', and ``PartOf''. 
Each relation has a unique sentence template, as shown in Table \ref{tab:relations}.
\begin{table}[!t]
\vspace*{-0.8\textheight}
\centering
\caption{Sentence Templates of Relations}
\label{tab:relations}
\begin{tabular}{ll}
\toprule
\textbf{Relation} & \textbf{Natural Language Template} \\ \midrule
AtLocation        & [subject] is found in [Aspect]       \\
MadeOf            & [subject] is made of [Aspect]         \\
UsedFor           & [subject] is used for [Aspect]         \\
CapableOf         & [subject] is capable of [Aspect]      \\
HasProperty       & [subject] is [Aspect]                 \\
PartOf            & [subject] is part of [Aspect]         \\ \bottomrule
\end{tabular}
\end{table}

\end{document}